\documentclass[lettersize,journal]{IEEEtran}
\usepackage{amsmath,amsfonts}
\usepackage{algorithmic}
\usepackage{algorithm}
\usepackage{array}
\usepackage[caption=false,font=normalsize,labelfont=sf,textfont=sf]{subfig}
\usepackage{textcomp}
\usepackage{stfloats}
\usepackage{url}
\usepackage{verbatim}
\usepackage{graphicx}
\usepackage{xcolor}
\usepackage{siunitx}
\newenvironment{restatetheorem}[2]
  {\par\medskip\noindent\textbf{Theorem~\ref{#1} (#2).}\ \itshape\ignorespaces}
  {\par\medskip}
\usepackage{hyperref}
\usepackage{multirow}
\usepackage{amsthm}
\usepackage{cite}
\usepackage{booktabs}
\usepackage{makecell}
\newtheorem{theorem}{Theorem}[section]

\begin{document}
\bstctlcite{BSTcontrol}

\title{SAIL: A Learnable Generalisation of Spatial Autocorrelation over Graph-Structured Data}

\author{Piotr~Keller and Fayyaz~Minhas%
\thanks{This work has been submitted to the IEEE for possible publication.
Copyright may be transferred without notice, after which this version may
no longer be accessible.}%
\thanks{P.~Keller and F.~Minhas are with the Predictive Systems in Biomedicine (PRISM) Lab, Department of Computer Science, University of Warwick, Coventry CV4 7AL, United Kingdom (e-mail: piotr.keller@warwick.ac.uk)}}

\markboth{}%
{Keller and Minhas: SAIL: A Learnable Generalisation of Spatial Autocorrelation}

\maketitle

\begin{abstract}
In vision and scientific domains involving spatial data, a key question is whether node feature values vary independently across space or form coherent spatial structures. Classical spatial autocorrelation statistics such as Moran's $I$ and Geary's $C$ quantify this phenomenon but operate on a single scalar variable with fixed neighbourhood weights, limiting their use for high-dimensional graph-like data. We introduce SAIL (Spatial Autocorrelation Index Learner), a differentiable framework that generalises spatial autocorrelation into a learnable operator by jointly learning feature projections defining candidate scalar fields and inter-node neighbourhood affinities. SAIL produces interpretable local and global spatial coherence descriptors, enabling supervised prediction and unsupervised discovery of spatial autocorrelation modes and hotspots. We show that SAIL exactly recovers classical spatial statistics under appropriate parametrisations and, for tasks driven by latent spatial autocorrelation, can approximate the Bayes-optimal predictor with an inductive bias matched to such tasks. Experiments on synthetic data and large-scale biomedical graphs demonstrate statistical correctness and discovery of spatial coherence with applications across vision, geospatial analysis, and biomedicine. The code repository is available at: \href{https://github.com/pkeller00/SAIL}{https://github.com/pkeller00/SAIL}.
\end{abstract}


\section{Introduction}
\IEEEPARstart{A}{re} tumour and immune cells intermixed within a tissue sample, or do they form distinct spatial niches associated with patient outcomes? Do neighbouring regions of a tissue section express similar gene programmes, or do these programmes change sharply across space? Are economic indicators geospatially coherent across neighbouring countries? 

Such questions arise across many domains involving spatially structured data, including pathobiological imaging, geospatial analysis, and computer vision, and are fundamentally about spatial self-similarity: whether observations resemble, or differ from, their neighbours. Quantifying this dependence allows us to ask not only if spatial structure is present, but also how strongly it is expressed, where it occurs, and if similar patterns recur across samples. Spatial autocorrelation (SAC) provides the formal framework for measuring such spatial self-similarity at local and global scales~\cite{manap2021identification,koenig2002global,castro2021spatial}.

Existing formulations for quantifying SAC, however, are poorly suited to modern high-dimensional, learned representations. Classical statistics such as Moran's $I$ and Geary's $C$~\cite{moran1950notes,geary1954contiguity,getis2009spatial,ospina2025spatialge,al2013spatial} operate on a single scalar variable under fixed, hand-crafted spatial weights, and their multivariate extensions~\cite{lin2023comparison,yamada2024moran,yamada2024geary} largely aggregate these univariate statistics rather than jointly learning which feature combinations carry spatial signal. Graph Neural Networks (GNNs) can implicitly encode spatial dependence via message passing~\cite{mackenzie2022neural,pizurica2023whole,ali2025graph}, but they do not provide an explicit, quantitative measure of spatial clustering or dispersion that can be optimised, discovered in an unsupervised manner, or subjected to statistical significance testing. This reveals a methodological gap: there is currently no scalable formulation of SAC that \emph{learns} the feature projection and neighbourhood structure for irregular, high-dimensional graph data and can be automatically optimised.

To address this, we introduce Spatial Autocorrelation Index Learner (SAIL), a neural framework that learns spatial autocorrelation descriptors directly from high-dimensional node features in spatial graphs or point sets. Unlike fixed metrics or post-hoc statistics, SAIL formulates SAC as a learnable, high-dimensional operator that jointly learns feature projections defining candidate scalar fields and inter-node neighbourhood affinities, adapting to multi-feature representations. This provides an explicit, task-aware measure of SAC, bridging classical spatial statistics and deep representation learning.

\begin{figure*}[!t]
  \centering
  \includegraphics[width=\textwidth]{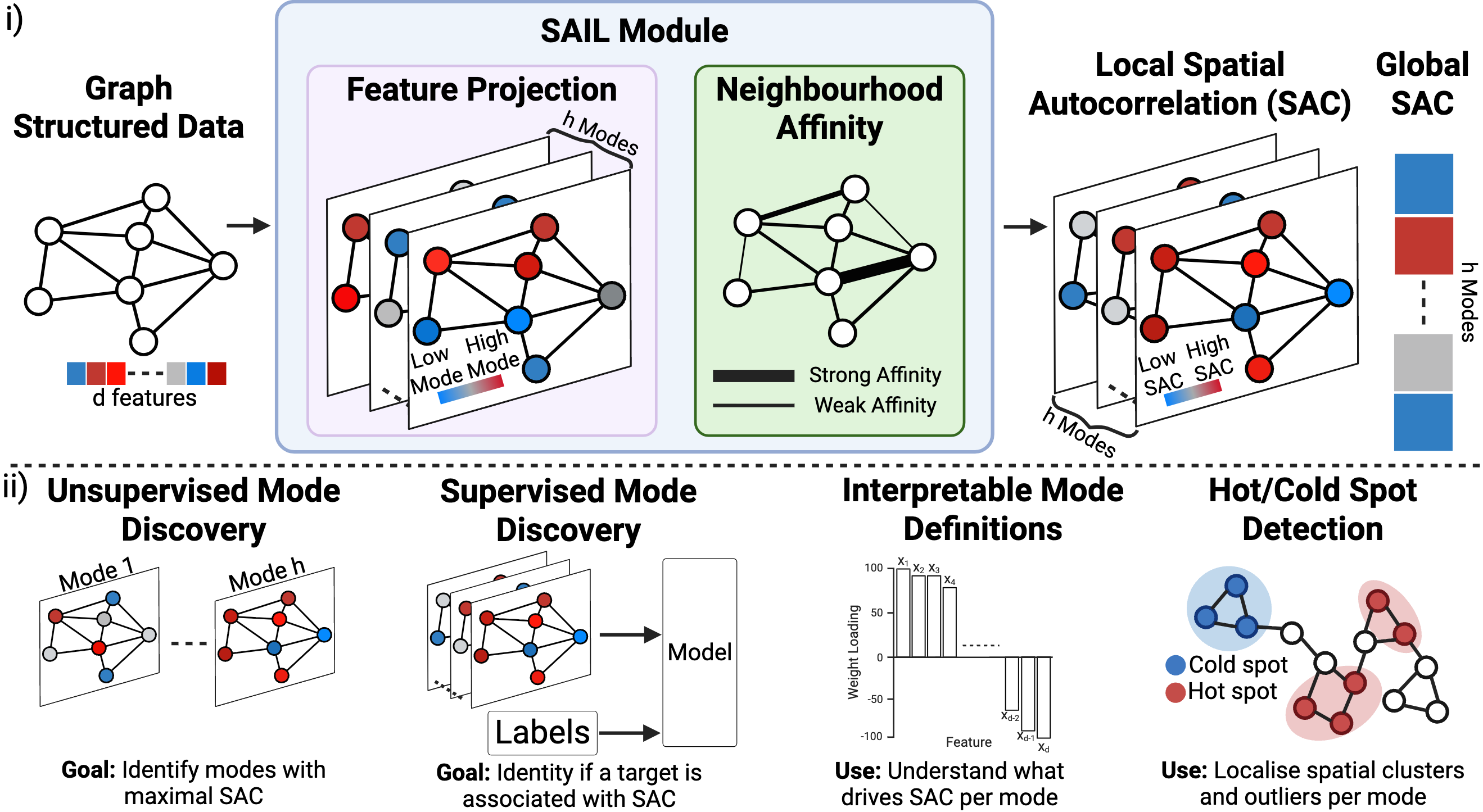}
\caption{Overview of the SAIL framework. (i) SAIL takes graph-structured data whose nodes carry $d$-dimensional features. Inside the SAIL module, a learnable feature projection maps node features to $h$ scalar fields (modes), and a learnable neighbourhood affinity assigns strong or weak weights to node pairs. Each projected field is combined with these affinities to give node-wise local spatial autocorrelation (SAC) scores for every mode, which can also be aggregated into an $h$-dimensional global SAC descriptor per graph. (ii) SAIL modes can be learned through supervised or unsupervised optimisation. In the unsupervised setting, SAIL identifies feature projections that maximise spatial autocorrelation (SAC), revealing which combinations of features exhibit the strongest spatial coherence across the dataset. Conversely, in supervised learning, global SAC descriptors are used to train a downstream predictor, enabling SAIL to identify spatial autocorrelation patterns associated with the target. In both settings, the learned feature projections provide interpretable combinations of features that define each mode. Finally, a formal statistical framework enables the
detection of significant hot and cold spots on the learned modes.}  \label{pipeline}
\end{figure*}

Our key contributions are as follows:
\begin{enumerate}
    \item{We introduce SAIL, a differentiable, multivariate generalisation of spatial autocorrelation that jointly learns feature projections defining scalar fields and spatial affinity weights governing neighbourhood interactions, yielding interpretable local and graph-level SAC maps.}
    \item{We demonstrate that SAIL supports supervised learning of predictive spatial patterns and unsupervised discovery of feature combinations exhibiting maximal spatial coherence or contrast, termed SAC modes and anti-modes, respectively, while enabling statistically significant detection of spatial hotspots, coldspots, and outliers.}
    \item{We establish theoretical guarantees showing that SAIL (i) exactly recovers classical spatial statistics (Moran's $I$, Geary's $C$, Getis--Ord $G^\star$, and Morisita--Horn) under appropriate configurations and (ii) provides an inductive bias suited to SAC-driven prediction by approximating Bayes-optimal predictors under stated assumptions, representing quadratic spatial interactions that linear $T$-hop graph filters cannot capture, and having no greater Rademacher complexity than broader hypothesis classes.}
    \item{We demonstrate SAIL's broad applicability across histology, spatial omics, and geospatial data, and its scalability to multi-graph datasets containing millions of nodes. SAIL uncovers meaningful and prognostically relevant spatial patterns without supervision, establishing it as a scalable, interpretable framework for learning spatial autocorrelation in high-dimensional graph-structured data.}
\end{enumerate}

\section{Related Work}
\label{sec:related_work}

\subsection{Classical Spatial Autocorrelation Statistics}
\label{sec:related_classical}

Consider a spatial graph $\mathcal{G}_i$ containing $n_i$ nodes. For a scalar feature $k$, let $x_{ij}^{(k)}$ denote the value of that feature at node $j$, and let 
$
\bar{x}_i^{(k)} = \frac{1}{n_i}\sum_{j=1}^{n_i} x_{ij}^{(k)}
$
denote its mean across the graph. Classical SAC statistics quantify whether similar or dissimilar values of $x_{ij}^{(k)}$ tend to occur at spatially neighbouring nodes. The spatial relationship between nodes $j$ and $j'$ is encoded by a predefined non-negative weight $w_{jj'} \geq 0$, typically derived from distance or adjacency.

\subsubsection{Moran's $I$}
Moran's $I$~\cite{moran1950notes,getis2009spatial} measures how similar the value of feature $k$, after centring by its global mean $\bar{x}_i^{(k)}$, is to the mean-centred values at its spatially weighted neighbours, relative to the global variance. Positive values indicate spatial clustering, whereas negative values indicate spatial contrast:
\begin{equation}
I_{\mathrm{Moran}}^{(k)} =
\frac{n_i}{\sum_{j \ne j'} w_{jj'}}
\cdot
\frac{\sum_{j \ne j'} w_{jj'} \left(x_{ij}^{(k)} - \bar{x}_i^{(k)}\right)\left(x_{ij'}^{(k)} - \bar{x}_i^{(k)}\right)}
{\sum_{j=1}^{n_i} \left(x_{ij}^{(k)} - \bar{x}_i^{(k)}\right)^2}.
\label{eq:moran}
\end{equation}

\subsubsection{Geary's $C$}
Geary's $C$~\cite{geary1954contiguity,ospina2025spatialge} instead penalises local squared differences directly. Values of $C<1$ suggest spatial clustering of similar values (positive spatial autocorrelation), whereas $C>1$ indicates spatial dispersion (negative spatial autocorrelation):
\begin{equation}
C_{\mathrm{Geary}}^{(k)} = \frac{(n_i - 1)}{2 \sum_{j \ne j'} w_{jj'}} \cdot \frac{\sum_{j \ne j'} w_{jj'} \left(x_{ij}^{(k)} - x_{ij'}^{(k)}\right)^2}{\sum_{j=1}^{n_i} \left(x_{ij}^{(k)} - \bar{x}_i^{(k)}\right)^2}.
\label{eq:geary}
\end{equation}
Other widely used classical SAC statistics include the Getis--Ord $G^\star$ statistic, which measures local hotspot concentration, and the Morisita--Horn index, which quantifies compositional overlap between spatial neighbourhoods. Their formulations are provided in Supplementary Section~\ref{sec:supp_classical_indices}. 

Despite their widespread use, the above classical SAC statistics assume a predefined scalar feature and spatial weighting scheme, limiting their use with high-dimensional data where informative feature combinations or neighbourhood relationships may need to be learned. Moreover, different SAC statistics are typically applied as separate fixed formulations rather than within a unified learnable framework. These limitations motivate the more general formulation introduced in SAIL.

\subsection{Multivariate Extensions and Their Limitations}
To move beyond a single scalar feature, several multivariate generalisations have been proposed~\cite{lin2023comparison,yamada2024moran,yamada2024geary}. The approach by Yamada et al.~\cite{yamada2024moran} standardises each feature channel $k=1,\dots,d$ to $z_{ij}^{(k)}=(x_{ij}^{(k)}-\bar{x}_i^{(k)})/s_i^{(k)}$, where $s_i^{(k)}$ is the standard deviation of feature $k$ in graph $\mathcal{G}_i$, and defines the multivariate Moran's $I$ as an average of pairwise inner products between standardised feature vectors $z_j=[z_j^{(1)},\dots,z_j^{(d)}]^\top$ over the graph:
\begin{equation}
I_{\mathrm{multi}} = \frac{1}{d\sum_{j\ne j'} w_{jj'}} \sum_{j\ne j'} w_{jj'}\, z_j^\top z_{j'}.
\label{eq:multi_moran}
\end{equation}
Because $z_j^\top z_{j'} = \sum_{k=1}^{d} z_j^{(k)} z_{j'}^{(k)}$, it can be shown that this multivariate index is exactly the mean of the per-feature Moran's $I$ values, $I_{\mathrm{multi}} = \frac{1}{d}\sum_{k=1}^{d} I^{(k)}_{\mathrm{Moran}}$~\cite{yamada2024moran}.

An analogous construction extends Geary's $C$ to the multivariate case~\cite{yamada2024geary} by replacing the scalar squared difference with the squared Euclidean norm between standardised feature vectors at neighbouring nodes:
\begin{equation}
C_{\mathrm{multi}} = \frac{1}{2d\sum_{j\ne j'} w_{jj'}} \sum_{j\ne j'} w_{jj'}\, \lVert z_j - z_{j'} \rVert^2.
\label{eq:multi_geary}
\end{equation}
Similarly to Moran's $I$, $C_{\mathrm{multi}}$ is exactly the mean of the per-feature Geary's $C$ values, $C_{\mathrm{multi}} = \frac{1}{d}\sum_{k=1}^{d} C^{(k)}_{\mathrm{Geary}}$~\cite{yamada2024geary}.

A second strategy avoids the aggregation problem via dimensionality reduction, computing classical univariate statistics on the resulting components. The simplest variant applies Principal Component Analysis (PCA), which lets a single component absorb variance across many correlated channels, but chooses the projection to maximise explained variance in the feature space alone, with no reference to the spatial coordinates or the neighbourhood structure. The component that best summarises feature variability therefore need not be the one that best expresses spatial coherence. MULTISPATI-PCA~\cite{dray2008spatial} addresses this by optimising the product of the variance and Moran's $I$. Its axes are thus spatially informed, but the objective still rewards variance as well as coherence, so a weakly varying yet highly coherent direction can be ranked below a high-variance, moderately coherent one.

However, these strategies remain fundamentally hand-crafted and linear: the feature projection is either fixed as an unweighted average across all channels or fixed by an objective (variance) unrelated to spatial structure, rather than learned to maximise spatial coherence directly. They require a predefined spatial weight matrix, operate on a fixed feature transformation, and are non-differentiable with respect to any learned representation. Because they rely on linear summaries of high-dimensional data, they cannot capture non-linear interactions between features or automatically identify feature subsets that jointly define a spatial niche. This limits their applicability in domains such as spatial omics, where the relevant spatial structure is often carried by a non-linear combination of many correlated channels rather than any single one, or by a projection chosen for an unrelated criterion.

\subsection{Graph Neural Networks}
\label{sec:related_gnn}

Graph Neural Networks (GNNs) partly address the need for modelling high-dimensional spatial dependence by aggregating information across spatial neighbourhoods through message passing, and have shown strong performance across a range of spatially structured prediction tasks~\cite{mackenzie2022neural,pizurica2023whole,ali2025graph}. However, GNNs encode spatial dependence only \emph{implicitly}, through the aggregation and update rules, without ever computing an explicit SAC statistic. As a result, standard GNNs do not directly provide an explicit or interpretable measure of how strongly a feature or learned representation is clustered or dispersed in space, nor do they support statistically testing the importance of spatial structure for an observed target without an associated supervised objective.

\subsection{Integration of Spatial Autocorrelation into Deep Learning}
\label{sec:related_sata}

Unlike GNNs, a separate line of work explicitly incorporates spatial autocorrelation into deep learning pipelines, for example as an auxiliary prediction target for geographic GNNs~\cite{klemmer2023positional} or, in Spatial Autocorrelation Token Analysis (SATA)~\cite{nikzad2025sata}, as a robustness mechanism for Vision Transformers. While these demonstrate that an explicit SAC statistic can be useful for deep learning, the underlying computation remains a fixed, predefined index evaluated on a predefined variable under a hand-crafted, non-adaptive spatial weight matrix, rather than a jointly learned feature projection and affinity structure. Consequently, these methods cannot adapt which feature combination is evaluated, nor how neighbourhood influence is defined, offering no mechanism for unsupervised discovery of task-relevant spatial modes.

\section{Methods}
\label{sec:methods}
The core idea of SAIL is to overcome the limitations of existing approaches outlined in the previous section by making the feature representation and spatial weighting scheme learnable, differentiable components of a modular spatial autocorrelation operator over graphs, providing a common formulation for diverse spatial data modalities that can be represented as graphs, including images, spatial omics, point sets, and geospatial data.

Formally, let
$
\mathcal{D}=\{\mathcal{G}_i\}_{i=1}^{N}
$
denote a dataset of $N$ spatial graphs. Each graph is represented as
$
\mathcal{G}_i=(\mathcal{V}_i,\mathcal{E}_i),
$
with the node set
\begin{equation}
\mathcal{V}_i
=
\left\{
v_{ij}
=
\left(x_{ij},c_{ij}\right)
\,\middle|\,
j=1,\ldots,n_i
\right\}.
\end{equation}
Here, $n_i$ denotes the number of nodes in graph $\mathcal{G}_i$, $x_{ij}\in\mathbb{R}^{d}$ is the $d$-dimensional feature vector associated with node $j$, and $c_{ij}\in\mathbb{R}^{c}$ denotes its coordinate in a physical or abstract spatial domain. The scalar feature value $x_{ij}^{(k)}$ used in the classical formulation corresponds here to the $k$th component of the feature vector $x_{ij}$ for a particular graph $i$.

The edge set
$
\mathcal{E}_i\subseteq\mathcal{V}_i\times\mathcal{V}_i
$
specifies any available topological relationships between nodes. Self-connections are excluded so that spatial autocorrelation is evaluated only between distinct nodes. For point sets without an explicit topology, $\mathcal{E}_i$ may be empty and neighbourhood relationships may instead be inferred directly from the spatial coordinates.

Spatial proximity between nodes $j$ and $j'$ is quantified using a user-specified distance function
$
f_{\mathrm{dist}}:
\mathbb{R}^{c}\times\mathbb{R}^{c}
\rightarrow
\mathbb{R}_{\geq 0}\cup\{\infty\},
$
such that
$
D_{ijj'}
=
f_{\mathrm{dist}}
\left(
c_{ij},
c_{ij'}
\right).
$
For example, $f_{\mathrm{dist}}$ may be the Euclidean, graph-geodesic or a masked distance that assigns infinity to node pairs that cannot interact.

Below, we describe the various components of the SAIL framework, how it can be used for supervised learning, and how it can discover spatial autocorrelation modes in unsupervised settings and statistically significant spatial hotspots.

\subsection{Learning Multivariate Spatial Autocorrelation with SAIL}
SAIL generalises classical SAC from a fixed scalar variable and predefined spatial weighting scheme to a learnable graph operator with modular components for feature transformation, neighbourhood construction, and spatial autocorrelation measurement. Specifically, for a spatial graph $\mathcal{G}$ with node features $\mathbf{X}$ and spatial coordinates $\mathbf{C}$, SAIL produces a transformed graph $\mathcal{G}'$ whose node attributes describe the local spatial autocorrelation of learned feature signals. The framework comprises four key components (Fig.~\ref{pipeline}):

\begin{enumerate}

\item A learnable feature mapping $f_{\mathrm{map}}(\cdot;\theta_{\mathrm{map}}):\mathbb{R}^{d}\rightarrow\mathbb{R}^{h}$ that transforms the node features into $h$ scalar fields on which SAC is evaluated. 

\item An inter-node affinity matrix $\mathbf{A}$ that defines the effective neighbourhood of each node as a linear combination of spatial distance $f_{\mathrm{dist}}$ and a feature-based inter-node affinity function $f_{\mathrm{affinity}}$, with $\tau$ controlling their relative contribution. The feature-based term may be omitted, and $\tau$ may be fixed or learned.

\item A local SAC operator $f_{\mathrm{local}}$ that computes a node-level SAC score for each projected scalar field by comparing each node with its affinity-weighted neighbourhood, thereby quantifying local spatial coherence or contrast across the graph.

\item An order-invariant reduction operator $f_{\mathrm{reduce}}$ that aggregates node-level SAC scores into graph-level descriptors.

\end{enumerate}

Together, these components define the transformation

\begin{equation}
\mathcal{G}' = \mathrm{SAIL}\big(\mathcal{G}; 
f_{\mathrm{map}}, f_{\mathrm{affinity}},f_{\mathrm{dist}}, \tau,f_{\mathrm{local}},f_{\mathrm{act}},f_{\mathrm{reduce}}\big),
\label{eq:sail-main}
\end{equation}

where $\mathcal{G}'$ retains the structure of $\mathcal{G}$, but its node features are replaced by, or augmented with, the learned local SAC descriptors $I^{\mathrm{local}}\in\mathbb{R}^{n\times h}$. Thus, after the SAIL transformation, each node $j$ is represented by the vector $I^{\mathrm{local}}_{j,:}$, which contains its spatial autocorrelation scores across the $h$ learned scalar fields.

The decomposition above highlights the modular nature of SAIL. By varying the feature mapping, affinity construction, local SAC statistic, activation, and reduction, the same framework can instantiate different flavours or notions of spatial autocorrelation within a common computational structure. The learnable parameters are contained primarily within $f_{\mathrm{map}}$ and $f_{\mathrm{affinity}}$, together with the weighting parameter $\tau$ when it is optimised, and are learned jointly from the dataset. In a learning problem involving multiple graphs, the full SAIL configuration is shared across all graphs, including the chosen forms of $f_{\mathrm{dist}}$, $f_{\mathrm{local}}$, $f_{\mathrm{act}}$, and $f_{\mathrm{reduce}}$ and any associated hyperparameters. SAC itself, however, is evaluated independently within each graph. Algorithm~\ref{alg:sail_forward} makes this computation explicit, showing the forward pass through feature projection and standardisation, affinity construction, local SAC evaluation, and graph-level reduction for an individual graph.

We next describe each component in detail.

\begin{algorithm}[t]
\caption{SAIL Forward Pass}
\label{alg:sail_forward}
\begin{algorithmic}

\STATE \textbf{Input:}
Graph $\mathcal{G}=(\mathbf{X},\mathbf{C})$

\STATE \textbf{Parameters:}
$\theta_{\mathrm{map}}$,
$\theta_{\mathrm{affinity}}$, 
$\tau>0$

\STATE \textbf{Operators:}
$f_{\mathrm{standardise}}$,
$f_{\mathrm{dist}}$,
$f_{\mathrm{local}}$,
$f_{\mathrm{act}}$, and
$f_{\mathrm{reduce}}$

\STATE \textbf{Return:}
$I^{\mathrm{local}}\in\mathbb{R}^{n\times h}$,
$I^{\mathrm{global}}\in\mathbb{R}^{h}$

\STATE
\STATE \textbf{function} \textsc{SAIL\_Forward}$(\mathcal{G};\theta,\tau)$

\STATE
$\mathbf{Z}
\gets
f_{\mathrm{map}}
\left(
\mathbf{X};
\theta_{\mathrm{map}}
\right)$
\COMMENT{$h$ projected scalar fields}

\STATE
$\tilde{\mathbf{Z}}
\gets
f_{\mathrm{standardise}}(\mathbf{Z}),
\qquad
\tilde{\mathbf{Z}}_{:,k}
=
\frac{\mathbf{Z}_{:,k}-\mu_k}{\sigma_k},
\quad
k=1,\ldots,h$

\FOR{$j=1,\ldots,n$}
    \FOR{$j'=1,\ldots,n$}
        \STATE
        $D_{jj'}
        \gets
        f_{\mathrm{dist}}
        \left(
        c_j,c_{j'}
        \right)$

        \STATE
        $F_{jj'}
        \gets
        f_{\mathrm{affinity}}
        \left(
        x_j,x_{j'};
        \theta_{\mathrm{affinity}}
        \right)$

        \STATE
        $\alpha_{jj'}
        \gets
        -\left(
        \tau D_{jj'}+F_{jj'}
        \right)$
    \ENDFOR

    \STATE
    $\alpha_{jj}\gets-\infty$

    \STATE
    $\mathbf{A}_{j,:}
    \gets
    \operatorname{softmax}
    \left(
    \alpha_{j,:}
    \right)$
\ENDFOR

\FOR{$k=1,\ldots,h$}
    \STATE
    $I^{\mathrm{local}}_{:,k}
    \gets
    f_{\mathrm{act}}
    \left(
    f_{\mathrm{local}}
    \left(
    \tilde{\mathbf{Z}}_{:,k},
    \mathbf{A}
    \right)
    \right)$
\ENDFOR

\STATE
$I^{\mathrm{global}}
\gets
f_{\mathrm{reduce}}
\left(
I^{\mathrm{local}}
\right)$

\STATE
\textbf{return}
$I^{\mathrm{local}},I^{\mathrm{global}}$

\STATE \textbf{end function}

\end{algorithmic}
\end{algorithm}

\subsubsection{Input Feature Space Mapping}

To evaluate SAC in high-dimensional data, SAIL first constructs candidate scalar fields from the node features. Let $\mathbf{X}\in\mathbb{R}^{n\times d}$ denote the node-feature matrix for a graph with $n$ nodes. A learnable mapping $f_{\mathrm{map}}(\cdot;\theta_{\mathrm{map}}):\mathbb{R}^{d}\rightarrow\mathbb{R}^{h}$ produces

$$
\mathbf{Z}
=
f_{\mathrm{map}}
\left(
\mathbf{X};
\theta_{\mathrm{map}}
\right)
\in
\mathbb{R}^{n\times h},
$$

where each column of $\mathbf{Z}$ defines a scalar field on which SAC can be evaluated. In its simplest form, $f_{\mathrm{map}}$ is linear,

$$
\mathbf{Z}
=
\mathbf{X}\mathbf{W},
\qquad
\mathbf{W}
\in
\mathbb{R}^{d\times h},
$$

although non-linear, sparse, or orthogonalised mappings may also be used (Supplementary Section~\ref{sec:supp_fmap}). Non-linear mappings allow SAIL to capture more complex combinations of the original features, sparse mappings can encourage compact and interpretable feature combinations, and orthogonalised mappings, for example through optimisation on the Stiefel manifold, can promote diverse and non-redundant projected scalar fields. Unlike classical SAC operators, which are typically applied to predefined scalar variables, SAIL is therefore not restricted to the original feature space and can learn task-relevant scalar fields on which spatial autocorrelation is evaluated.

Because the scale of the learned projections is arbitrary, each projected scalar field is standardised using $f_{\mathrm{standardise}}$. This centres and scales the projected signals so that the resulting SAC scores reflect spatial coherence rather than differences in magnitude or variance. The standardised scalar fields are collected in $\tilde{\mathbf{Z}}=f_{\mathrm{standardise}}(\mathbf{Z})\in\mathbb{R}^{n\times h}$.

\subsubsection{Learnable Spatial Affinities}

SAIL next defines how nodes influence one another through an inter-node affinity matrix. Classical SAC methods typically use a predefined spatial weighting scheme, such as adjacency or inverse distance. SAIL instead allows neighbourhood influence to depend on spatial proximity and, optionally, learned feature relationships.

For each pair of nodes $(j,j')$ in the graph, their spatial separation is given by $D_{jj'}=f_{\mathrm{dist}}(c_j,c_{j'})$, and an unnormalised affinity score is defined as

$$
\alpha_{jj'}
=
-
\left(
\tau D_{jj'}
+
f_{\mathrm{affinity}}
\left(
x_j,x_{j'};
\theta_{\mathrm{affinity}}
\right)
\right),
$$

where $\tau>0$ controls the contribution of spatial distance and $f_{\mathrm{affinity}}$ models feature-dependent dissimilarity. Larger dissimilarity thus decreases the affinity between two nodes.

The resulting scores are converted into non-negative neighbourhood weights through row-wise softmax normalisation:

$$
A_{jj'}
=
\frac{
\exp(\alpha_{jj'})
}{
\sum_{l=1}^{n}\exp(\alpha_{jl})
},
\qquad
\mathbf{A}
\in
\mathbb{R}^{n\times n}.
$$

By optimising $\tau$ and $\theta_{\mathrm{affinity}}$, SAIL can adapt the effective neighbourhood structure to the data and learning objective. Purely spatial affinities are obtained by omitting the feature-dependent term. Self-connections are excluded by default by setting $A_{jj}=0$, so that SAC is evaluated with respect to neighbouring nodes rather than the node itself, but they may be retained when required by the chosen SAC formulation.

\subsubsection{Local and Graph-Level SAIL Outputs}

Given the standardised projected fields $\tilde{\mathbf{Z}}\in\mathbb{R}^{n\times h}$ and affinity matrix $\mathbf{A}\in\mathbb{R}^{n\times n}$, the local SAC operator $f_{\mathrm{local}}$ computes a node-level spatial statistic for each projected scalar field. Here, $\tilde{\mathbf{Z}}=[\tilde{\mathbf z}^{(1)},\ldots,\tilde{\mathbf z}^{(h)}]$, where $\tilde{\mathbf z}^{(k)}\in\mathbb{R}^{n}$ denotes the $k$th standardised scalar field across all nodes and $\tilde z_j^{(k)}$ is its value at node $j$. For each field, $f_{\mathrm{local}}$ compares the value at a node with the values of other nodes according to the affinities in $\mathbf{A}$, thereby quantifying the local spatial relationship around that node. Different choices of $f_{\mathrm{local}}$ instantiate different notions of spatial structure within the same SAIL framework.

For example, a Moran-style operator measures local agreement between the value at node $j$ and the affinity-weighted values of its neighbourhood:

\begin{equation}
I_{jk}^{\mathrm{local}}
=
\tilde z_j^{(k)}
\left[
\mathbf A
\tilde{\mathbf z}^{(k)}
\right]_j,
\label{eq:moran_local}
\end{equation}

where $[\mathbf A\tilde{\mathbf z}^{(k)}]_j=\sum_{j'}A_{jj'}\tilde z_{j'}^{(k)}$ is the spatial lag of the $k$th scalar field at node $j$. Positive values indicate that the node and its neighbourhood have similar signed deviations from the field mean, whereas negative values indicate local spatial contrast.

In contrast, a Geary-style operator measures local dissimilarity directly through affinity-weighted differences between the node and its neighbours:

$$
I_{jk}^{\mathrm{local}}
=
\frac{1}{2}
\sum_{j'=1}^{n}
A_{jj'}
\left(
\tilde z_j^{(k)}
-
\tilde z_{j'}^{(k)}
\right)^2.
$$

Here, smaller values indicate greater local similarity, while larger values indicate stronger local heterogeneity.

SAIL also supports Getis--Ord $G^\star$-type hotspot measures and Morisita--Horn-type neighbourhood-overlap statistics via alternative choices of $f_{\mathrm{local}}$. The precise form of the local operator, together with any associated preprocessing or affinity conventions, depends on the statistic being instantiated.

An optional elementwise transformation $f_{\mathrm{act}}$, such as ReLU or ELU, may then be applied,

$$
I_{jk}^{\mathrm{local}}
\leftarrow
f_{\mathrm{act}}
\left(
I_{jk}^{\mathrm{local}}
\right),
$$

yielding the node-level SAC descriptor matrix $I^{\mathrm{local}}\in\mathbb{R}^{n\times h}$.

Graph-level descriptors are obtained by applying an order-invariant reduction operator $f_{\mathrm{reduce}}:\mathbb{R}^{n\times h}\rightarrow\mathbb{R}^{h}$ to the local scores:

$$
I^{\mathrm{global}}
=
f_{\mathrm{reduce}}
\left(
I^{\mathrm{local}}
\right).
$$

Thus, $I^{\mathrm{local}}$ describes where spatial structure occurs within the graph, while $I^{\mathrm{global}}$ summarises the overall degree of spatial coherence or heterogeneity for each learned scalar field. 

\subsection{Supervised Learning with SAIL} 



A key question in spatial learning is whether the spatial autocorrelation of node-level feature signals is predictive of a target of interest. Two spatial samples, such as graphs, may have similar overall feature distributions yet differ substantially in how those features are organised across neighbouring nodes. For example, two images may contain similar visual features but differ in how those features are distributed across adjacent regions, while two tissue samples may contain similar cell populations or molecular programmes but differ in whether they are spatially clustered, dispersed, or organised into coherent niches.

Because SAIL transforms node-level features into learnable SAC descriptors, it is naturally suited to supervised tasks in which the target depends on SAC. These descriptors can be used directly for classification, regression, survival modelling, or other predictive objectives, allowing the model to learn which patterns of spatial autocorrelation are most informative for the target.

Formally, given a dataset $\mathcal{D}=\{(\mathcal{G}_i, y_i)\}_{i=1}^N$, where each $\mathcal{G}_i$ is a spatial graph and $y_i$ is a graph-level target label, SAIL transforms node features into SAC descriptors, either at the node $I^{\mathrm{local}}_i$ or graph level $I^{\mathrm{global}}_i$. These descriptors are passed to a prediction head $f_{\mathrm{pred}}$, and the full parameter set is optimised end-to-end by minimising a task-specific loss:
\[
\min_{\theta} 
\sum_{i=1}^{N}
\mathcal{L}
\left(
f_{\mathrm{pred}}
\big(
\mathrm{SAIL}(\mathcal{G}_i; \theta_{\mathrm{SAIL}});
\theta_{\mathrm{pred}}
\big),
y_i
\right).
\]

Since the predictor uses only SAIL outputs, performance reflects how much the target depends on SAC. High accuracy indicates informative SAC; near-chance accuracy indicates little dependence. Pseudocode can be found in Supplementary Algorithm~\ref{alg:sail_supervised}.

\subsection{Unsupervised Learning of Spatial Autocorrelation Modes}
\label{sec:modes}

Classical SAC statistics quantify the spatial coherence of a fixed scalar variable. In high-dimensional settings, however, spatial structure may only emerge through combinations of features. SAIL addresses this by learning projections that maximise or minimise spatial coherence across a graph dataset. We refer to these as \emph{SAIL Modes} and \emph{SAIL Anti-Modes}, respectively.

Formally, let $\mathcal{D} = \{\mathcal{G}_i\}_{i=1}^N$ denote a graph dataset, each with node features $\mathbf{X}_i \in \mathbb{R}^{n_i \times d}$. A projection vector $w \in \mathbb{R}^d$ defines a scalar field $z_i = \mathbf{X}_i w$ on each graph. We quantify its expected spatial coherence as
$$
J_{\mathrm{coh}}(w)
=
\mathbb{E}_{\mathcal{G}_i \sim \mathcal{D}}
\left[
C\left(
I^{\mathrm{global}}
\left(
\mathcal{G}_i;
f_{\mathrm{map}}(x)=x^\top w
\right)
\right)
\right],
$$
where $C(\cdot)$ ensures that larger values indicate stronger spatial coherence, for example, $C(I)=I$ for Moran-style operators and $C(I)=-I$ for Geary-style operators. SAIL Modes and Anti-Modes are then defined as
$$
\begin{aligned}
w^*_{\mathrm{mode}}
&= \underset{\|w\|=1}{\arg\max}\;J_{\mathrm{coh}}(w),\\
w^*_{\mathrm{anti}}
&= \underset{\|w\|=1}{\arg\min}\;J_{\mathrm{coh}}(w).
\end{aligned}
$$
Thus, Modes identify feature combinations exhibiting maximal spatial clustering, whereas Anti-Modes capture maximal spatial contrast.

To discover multiple patterns, we jointly optimise $K$ projection directions collected in $\mathbf{W}=[w^{(1)},\ldots,w^{(K)}]\in\mathbb{R}^{d\times K}$. The multi-mode objective promotes spatial coherence while encouraging diversity and sparsity:
$$
\begin{aligned}
\min_{\mathbf{W}}\quad
&
-\frac{1}{K}\sum_{r=1}^{K}J_{\mathrm{coh}}(w^{(r)})
\\
&+\lambda_{\mathrm{corr}}
\sum_{r\neq s}\mathrm{Corr}(z^{(r)},z^{(s)})^2
+\lambda_{\mathrm{sparse}}\|\mathbf{W}\|_1.
\end{aligned}
$$
Anti-Modes are learned by reversing the sign of the first term. The correlation penalty encourages distinct spatial patterns, while sparsity favours interpretable feature combinations. Optionally, the orthogonality constraint $\mathbf{W}^{\top}\mathbf{W}=\mathbf{I}_K$ can further promote diversity. Pseudocode is provided in Supplementary Algorithm~\ref{alg:sail_modes}.

\subsection{Statistically Significant Spatial Hotspot Discovery}
\label{hotspot_discovery}

While supervised learning with SAIL evaluates whether SAC carries a predictive signal, and unsupervised SAIL Modes identify directions of maximal SAC, both operate at a global level. They quantify or discover \emph{whether} spatial structure exists, but do not localise \emph{where} it occurs within the graph.

To address this, we adapt SAIL's local scores to a Local Indicators of Spatial Association (LISA) style test~\cite{anselin1995local}. This categorisation requires a \emph{signed} local statistic, since the High--High/Low--Low/High--Low/Low--High classification depends on the sign of the node value and the sign of its spatial lag. We therefore use the Moran-style local score, Eq.~\eqref{eq:moran_local}, as its formulation naturally decomposes into these four regions.

For a scalar field $\tilde{z}^{(k)}$, we define the spatial lag at node $j$ as the affinity-weighted neighbourhood average,
\begin{equation}
\ell_j^{(k)} = \sum_{j'} A_{jj'}\, \tilde{z}_{j'}^{(k)},
\label{eq:spatial_lag}
\end{equation}
so that the Moran-style local score, $I_{jk}^{\mathrm{local}}$, serves as the test statistic for node-level significance. Under the null hypothesis of spatial independence, we generate a reference distribution by randomly permuting $\tilde{z}^{(k)}$ across nodes, recomputing the lag $\ell_j^{(k)}$ and the resulting statistic for each permutation, holding the affinity matrix $\mathbf{A}$ fixed. Empirical $p$-values for each node are obtained by comparing the observed statistic to this permutation null, with either false-discovery-rate (Benjamini--Hochberg) or max-T~\cite{westfall1993resampling} correction applied across nodes.

Nodes whose local statistic is deemed significant are then categorised according to the signs of $\tilde{z}_j^{(k)}$ and its spatial lag $\ell_j^{(k)}$ considered separately: High--High (a high-value node surrounded by high-value neighbours), Low--Low (coldspots), and the discordant categories High--Low and Low--High (spatial outliers). Non-significant nodes remain unclassified. This procedure transforms SAIL's learned local scores into statistically validated spatial structure maps, enabling rigorous localisation of coherent and heterogeneous regions.

\section{Theoretical Analysis} \label{Thereotical_Analysis}
This section establishes SAIL as a principled generalisation of classical spatial autocorrelation (SAC) to learnable graph representations. Theorem~\ref{theorem:classical_reduction} shows that SAIL exactly recovers classical SAC statistics under appropriate configurations, preserving their statistical interpretation. Theorem~\ref{theorem:inductive_bias} establishes that SAIL is well suited to SAC-driven prediction: it can approximate Bayes-optimal predictors under stated assumptions, explicitly represents quadratic spatial interactions beyond the capacity of linear $T$-hop graph filters, and has no greater Rademacher complexity than broader hypothesis classes containing SAIL. Together, these results show that SAIL combines the interpretability of classical spatial statistics with the representational capacity and inductive bias needed for learning from spatial structure.

\begin{theorem}[Classical spatial statistics as exact special cases of SAIL]
\label{theorem:classical_reduction}

Let $\mathbf{A}=[A_{jj'}]\in\mathbb{R}_{\geq 0}^{n\times n}$ be a fixed
row-stochastic spatial weight matrix, i.e., $\mathbf{A}\mathbf{1}=\mathbf{1}$,
and let $f_{\mathrm{act}}=\mathrm{Id}$. 
Then, under appropriate fixed choices of the SAIL mapping,
local-scoring, standardisation, and reduction operators, SAIL (Eq.~\eqref{eq:sail-main}) exactly
recovers:

\begin{enumerate}
    \item[(i)] Moran's $I$ for a non-constant scalar mapped field,
    using population standardisation, $A_{jj}=0$, and mean reduction;

    \item[(ii)] Geary's $C$ for a non-constant scalar mapped field,
    using sample standardisation, $A_{jj}=0$, and mean reduction;

    \item[(iii)] the standardised Getis--Ord $G_j^\star$ statistic
    for a non-constant scalar mapped field, using population
    standardisation and allowing self-weights as in the starred
    statistic, whenever its null variance is non-zero; and

    \item[(iv)] the Morisita--Horn similarity between a non-negative,
    unstandardised multichannel mapped representation at a node and its
    affinity-weighted neighbourhood composition, provided that
    $f_{\mathrm{local}}$ is allowed to act jointly on the corresponding
    mapped channels.
\end{enumerate}

Specifically, the resulting SAIL outputs are respectively
$I_{\mathrm{Moran}}(z;\mathbf{A})$, $C_{\mathrm{Geary}}(z;\mathbf{A})$,
$G_j^\star(z;\mathbf{A})$, and
$C_{\mathrm{MH}}(u_j,v_j)$, where
$u_j=f_{\mathrm{map}}(x_j)\in\mathbb{R}_{\geq 0}^{h}$
is the mapped representation at node $j$ and
$v_j=\sum_{j'}A_{jj'}u_{j'}$ is its
affinity-weighted neighbourhood composition.
\end{theorem}

\begin{theorem}[Inductive Bias of SAIL for SAC-Driven Learning]
\label{theorem:inductive_bias}

Let $P$ be a distribution over spatial graphs
$G=(V,E,\mathbf{X},\mathbf{C})$, with an optional
associated target variable $Y\in\mathcal{Y}$,
and population risk
\[
R_P(f)=\mathbb{E}_{(G,Y)\sim P}
\left[\mathcal{L}_P(Y,f(G))\right],
\]

and Bayes-optimal predictor $f_P^\star\in\arg\min_f R_P(f)$ with risk $R_P^\star=R_P(f_P^\star)$.

Let $\mathcal{P}_{\mathrm{SAC}}$ denote the class of problems whose Bayes-optimal predictor depends on one or more spatial-autocorrelation (SAC) descriptors of node signals. For clarity, consider a single global descriptor,
$$
f_P^\star(G)=\rho_P\!\left(S_P(z_P,\mathbf{A}_P)\right),
\qquad
z_P=g_P(\mathbf{X})\in\mathbb{R}^n,
$$
where $\mathbf{A}_P$ is a spatial-affinity operator, $S_P$ is a SAC functional formed from local spatial interactions followed by an order-invariant reduction, and $\rho_P$ maps the SAC descriptor to the prediction target. 

Let
$$
\mathcal{F}_{\mathrm{SAIL}}
=
\left\{
G\mapsto
f_{\mathrm{pred}}
\left(
I^{\mathrm{global}}(G;\theta_{\mathrm{SAIL}});
\theta_{\mathrm{pred}}
\right)
\right\}
$$
denote the supervised SAIL hypothesis class, where $I^{\mathrm{global}}$ is obtained from the feature mapping $f_{\mathrm{map}}$, learned affinity $\mathbf{A}$, local SAC operator $f_{\mathrm{local}}$, optional activation $f_{\mathrm{act}}$, and reduction $f_{\mathrm{reduce}}$ of Eq.~\eqref{eq:sail-main}.

Assume that the admissible SAIL component families are sufficiently expressive to approximate $g_P$, $A_P$, and $\rho_P$, and that the chosen local and reduction operators implement $S_P$. Assume also that $0\leq\mathcal{L}_P\leq M$ and that $\mathcal{L}_P$ is $L$-Lipschitz in its prediction argument. Define
$$
\Delta_P(\mathcal{F})
=
\inf_{f\in\mathcal{F}}R_P(f)-R_P^\star,
$$
and let $\mathfrak{R}_m(\mathcal{F})$ denote the Rademacher complexity of a hypothesis class $\mathcal{F}$ with respect to $m$ i.i.d.\ samples from $P$.

Then:

\textnormal{(i) Sufficiency.}
For every $P\in\mathcal{P}_{\mathrm{SAC}}$ and every $\varepsilon>0$, there exists $f_\varepsilon\in\mathcal{F}_{\mathrm{SAIL}}$ such that
$$
\mathbb{E}_G\!\left[
|f_\varepsilon(G)-f_P^\star(G)|
\right]\leq\varepsilon,
$$
and therefore
$$
\Delta_P(\mathcal{F}_{\mathrm{SAIL}})
\leq L\varepsilon.
$$
If $\widehat f_{\mathrm{SAIL}}$ is an $\xi_m$-approximate empirical-risk minimiser from $m$ i.i.d.\ examples, then with probability at least $1-\delta$,
$$
R_P(\widehat f_{\mathrm{SAIL}})-R_P^\star
\leq
L\varepsilon+\xi_m
+
4L\mathfrak{R}_m(\mathcal{F}_{\mathrm{SAIL}})
+
2M\sqrt{\frac{\log(2/\delta)}{2m}}.
$$

\textnormal{(ii) Necessity of higher-order interaction for quadratic SAC.}
Consider the quadratic SAC subclass
$$
S(z,\mathbf{A})=z^\top Q(\mathbf{A})z,
$$
with $\mathbf{A}$ fixed with respect to $z$ and
$$
Q_s(\mathbf{A})=\frac{Q(\mathbf{A})+Q(\mathbf{A})^\top}{2}\neq0.
$$
Then
$$
\nabla_z^2 S(z,\mathbf{A})=2Q_s(\mathbf{A})\neq0.
$$
Let
$$
\mathcal{F}_{\mathrm{lin},T}
=
\left\{
f(z,\mathbf{A})=w^\top H_T(\mathbf{A})z+c:
H_T(\mathbf{A})=\sum_{t=0}^{T}\beta_t\mathbf{A}^t
\right\}
$$
denote the class of finite $T$-hop linear graph filters followed by linear pooling and readout. This includes, for example, Simple Graph Convolution (SGC), linearised Graph Convolution Networks (GCNs), polynomial graph filters, and linear message-passing models with linear pooling/readout. Every $f\in\mathcal{F}_{\mathrm{lin},T}$ is affine in $z$ and hence satisfies $\nabla_z^2f=0$. Thus $\mathcal{F}_{\mathrm{lin},T}$ cannot in general represent the quadratic SAC subclass, whereas SAIL explicitly contains the interaction
$$
z_j(\mathbf{A}z)_j=\sum_{j'}A_{jj'}z_j z_{j'}.
$$
\textnormal{(iii) Comparison with broader classes.}
More general non-linear graph neural networks (GNNs) may also represent the higher-order
interactions required by SAC, but only as one possibility within a
broader function class. Let $\mathcal{F}_{\mathrm{gen}}$ denote such a
broader SAC-capable hypothesis class satisfying
$$
\mathcal{F}_{\mathrm{SAIL}}
\subseteq
\mathcal{F}_{\mathrm{gen}}.
$$
Then
$$
\mathfrak{R}_m(\mathcal{F}_{\mathrm{SAIL}})
\leq
\mathfrak{R}_m(\mathcal{F}_{\mathrm{gen}}).
$$
Hence, once the SAC structure is already representable by SAIL,
additional generic expressivity cannot, by itself, yield a smaller
Rademacher-complexity term in the corresponding standard
generalisation bound. If the two classes attain the same approximation
and optimisation errors on a SAC-driven problem, the resulting
excess-risk bound for SAIL is no larger.

Thus SAIL provides a matched inductive bias for $\mathcal{P}_{\mathrm{SAC}}$: it contains the interactions required by SAC-driven prediction, simpler linear-interaction classes may fail to represent them, while broader enclosing classes provide no smaller complexity-dependent learning guarantee solely through their additional expressivity.
\end{theorem}

\section{Experimental Results and Applications}
We evaluate SAIL through a series of experiments designed to assess its statistical correctness, predictive capabilities, scalability, and applicability across diverse spatial data. We first validate its ability to recover classical spatial autocorrelation statistics and examine its advantages over conventional graph neural networks on synthetic tasks. We then investigate its ability to discover meaningful spatial patterns in real-world datasets spanning histology, spatial transcriptomics, spatial proteomics, and geospatial analysis, including large-scale graphs with millions of nodes. Together, these experiments demonstrate how learnable spatial autocorrelation supports both predictive modelling and unsupervised discovery of spatial structure. For readers less familiar with the biomedical setting, Supplementary Section~\ref{sec:supp_bio_background} summarises the relevant terminology, including spatial omics modalities, whole-slide images, the Dice overlap score and the survival-analysis metrics used throughout.

\subsection{Recovering classical spatial autocorrelation with SAIL}
\label{sec:synthetic_experiments}

\begin{figure}[!t]
\centering
\includegraphics[width=3in]{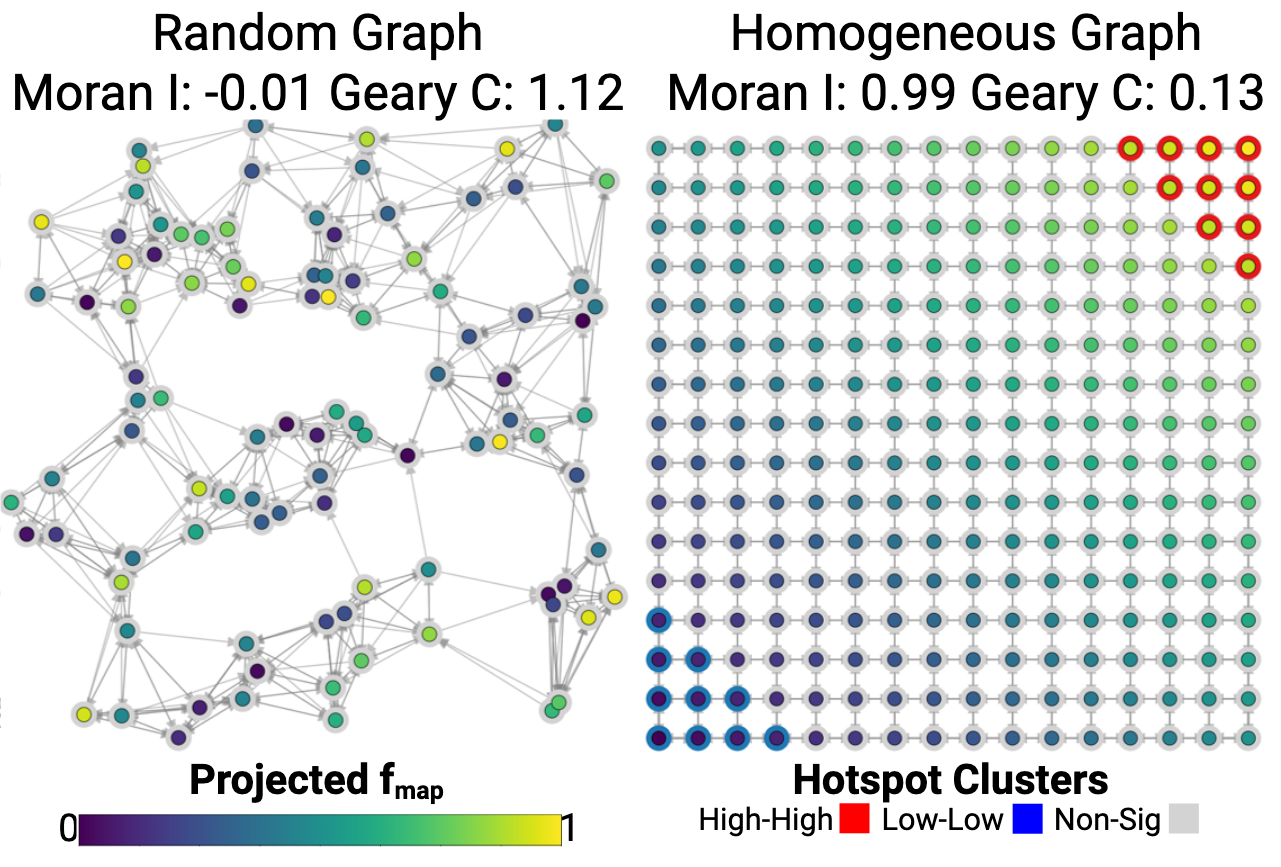}
\caption{SAIL recovers classical SAC on synthetic graphs with known structure. Left: a graph with unstructured, independently sampled node values (colour bar shows feature value), for which SAIL correctly reports near-zero global SAC (Moran's $I=-0.01$, Geary's $C=1.12$) and classifies most nodes as non-significant under hotspot analysis. Right: a spatially homogeneous graph with a smooth imposed gradient, for which SAIL correctly detects strong positive autocorrelation (Moran's $I=0.99$, Geary's $C$ $=0.13$) and hotspot analysis correctly localises High--High and Low--Low clusters matching the imposed gradient.}
\label{baseline_example_graphs}
\end{figure}

\begin{figure}[!t]
\centering
\includegraphics[width=3in]{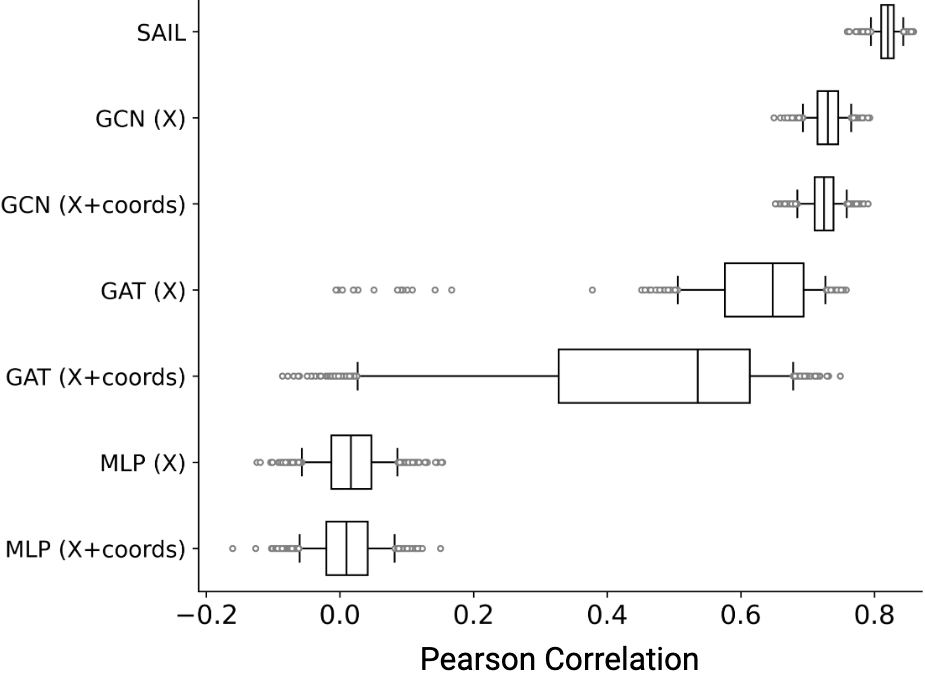}
\caption{Results demonstrating distinction of SAIL from traditional message-passing architectures. In a regression benchmark where models predict Moran's $I$ of heterogeneous test graphs, SAIL achieves significantly higher Pearson correlation between predicted and true values than classical message-passing networks with comparable or larger parameter counts ($p<0.0001$, box plots over 1000 train--validation--test runs).}
\label{baseline_regression}
\end{figure}

We begin by empirically validating that SAIL recovers classical SAC statistics on graphs with known spatial structure, confirming Theorem~\ref{theorem:classical_reduction}. We construct toy graphs with controlled spatial properties (Fig.~\ref{baseline_example_graphs}). First, we generate a random graph with unstructured scalar node values. As node values are sampled independently of spatial location, the expected global SAC is near zero. Second, we construct a spatially homogeneous graph with a smooth spatial gradient, inducing strong positive SAC. For both graphs, we compute Moran-style and Geary-style SAIL scores and visualise these local statistics alongside LISA-style maps.

For the random graph (Fig.~\ref{baseline_example_graphs}), SAIL detects negligible global SAC (Moran's $I \approx -0.01$, Geary's $C = 1.12$). The associated LISA map classifies most nodes as non-significant (NS), consistent with expectations. In contrast, for the spatially homogeneous graph (Fig.~\ref{baseline_example_graphs}), SAIL reports Moran's $I = 0.99$ and Geary's $C = 0.13$, reflecting strong global coherence. LISA decomposition reveals clear High--High clusters in high-value regions and Low--Low clusters in low-value regions, accurately recovering the imposed gradient structure. These results empirically confirm that SAIL reproduces classical SAC indices, confirming Theorem~\ref{theorem:classical_reduction}

\subsection{Learning spatial autocorrelation beyond message passing}
\label{sec:beyond_mp}

We next evaluate whether SAIL's spatial coherence operator provides an advantage in prediction tasks where the label is explicitly defined by a global SAC functional. This experiment tests the inductive bias characterised in Theorem~\ref{theorem:inductive_bias}: when labels depend on the SAC of a latent signal, SAIL contains the required quadratic interactions by construction (part (ii)), whereas message-passing architectures must learn them implicitly within a broader hypothesis class (part (iii)).

To this end, we generate graphs from two-dimensional point clouds with spatially correlated features and define the regression target as Moran's $I$ of a latent feature projection; full details are provided in Supplementary Section~\ref{sec:supp_synth_protocol}.

We train SAIL to predict this target and compare it against a diverse set of models, including Graph Convolution (GCN) and Graph Attention (GAT) Networks operating on complete graphs. Edges are either distance-weighted or unweighted. All models receive identical node feature channels, and when applicable, the 2D coordinates are provided as additional positional features. Models are trained under matched optimisation settings with early stopping, and experiments are repeated across multiple random seeds to assess statistical reliability. 

In Fig.~\ref{baseline_regression}, SAIL ($r=0.82 \pm 0.01$) achieves the highest Pearson correlation across runs. Among message-passing comparisons, the strongest performance is obtained by GCN operating on a complete graph ($r=0.73 \pm 0.02$), while augmenting the node features with spatial coordinates yields similar performance ($r=0.72 \pm 0.02$). A paired Wilcoxon signed-rank test confirms SAIL's improvement is statistically significant relative to all models (FDR-corrected $p < 10^{-4}$).

These results provide empirical support for Theorem~\ref{theorem:inductive_bias}(i): when labels depend on the SAC of a latent signal, SAIL can approximate the Bayes-optimal predictor. The performance gap relative to GCN and GAT is consistent with Theorem~\ref{theorem:inductive_bias}(iii): although non-linear message-passing networks may in principle represent the quadratic interactions underlying SAC, they must discover them within a broader hypothesis class, whereas SAIL encodes them explicitly and incurs no larger complexity term. Together, these experiments demonstrate that SAIL is not merely a flexible neural module, but a principled generalisation of classical spatial statistics with an inductive bias that is well matched to SAC-driven problems.

\subsection{Unsupervised discovery of Spatially Coherent Structures in Spatial Omics}
\label{sec:4i}

Having verified statistical correctness in synthetic settings, we next test if SAIL can discover biologically meaningful spatial niches in spatial omics data. Here, biologically relevant compartments arise from spatially coherent patterns of gene or protein expression rather than from individual feature values. Classical analyses typically operate on individual markers or clustering of feature vectors and therefore do not directly test if a feature combination forms a spatially coherent field. Likewise, standard GNNs are designed for supervised representation learning through message passing and do not provide an explicit unsupervised objective for discovering feature directions that maximise SAC. SAIL addresses this by explicitly learning feature projections that maximise SAC. If biologically meaningful niches correspond to coherent SAC of molecular programs, then unsupervised SAIL Modes should align with known compartments. We thus evaluate if SAIL Modes correspond to expert-annotated biological structures.
\subsubsection{Datasets and Setup}

We analyse the preprocessed Iterative Indirect Immunofluorescence Imaging (4i) spatial proteomics dataset~\cite{gut2018multiplexed}, which measures the expression levels of 43 proteins at $n=270{,}876$ spatial measurement locations for one sample. Each location is annotated with the subcellular compartment it belongs to (such as the nucleolus).

The task is to test whether protein combinations that SAIL identifies as spatially autocorrelated, without supervision to prevent overfitting with one sample, correspond to these annotated compartments. SAIL is applied in the unsupervised mode-learning configuration (Section~\ref{sec:modes}), learning $K=20$ SAIL Modes by minimising a Geary-style objective. For each mode, LISA-style hotspot analysis (Section~\ref{hotspot_discovery}) identifies locations belonging to significant High--High or Low--Low clusters, which form High--High or Low--Low masks respectively. Each mask is compared with each annotated compartment using the Dice score, which ranges from 0 (no overlap) to 1 (perfect agreement). 

We compare SAIL against four alternative ways of obtaining spatially autocorrelated fields, each followed by the same LISA-style hotspot analysis and Dice evaluation:
(i)~\emph{Univariate}: local Moran's $I$ on individual proteins, taking the best of the $20$ proteins with the highest global Moran's $I$;
(ii)~\emph{Mean}: the multivariate Moran's $I$ obtained by averaging univariate statistics across all proteins~\cite{yamada2024moran};
(iii)~\emph{MULTISPATI-PCA}~\cite{dray2008spatial}: local Moran's $I$ on axes that maximise the product of variance and Moran's $I$; and
(iv)~\emph{Hotspot}~\cite{detomaso2021hotspot}: local Moran's $I$ on the scores of spatially autocorrelated protein modules. Additional implementation details are given in Supplementary Section~\ref{sec:supp_4i}.

\subsubsection{Results}
In the 4i dataset (Table~\ref{tab:dice_performance}), SAIL Modes capture distinct compartments and outperform univariate ($p=0.003$), mean ($p=0.003$), MULTISPATI-PCA ($p=0.014$) and Hotspot ($p=0.010$) comparisons across compartments (one-sided paired Wilcoxon signed-rank tests; Supplementary Table~\ref{tab:supp_4i_full_comparison}). In Fig.~\ref{4i_figure} (a), SAIL Mode 8, whose Low--Low hotspots achieve a Dice score of $0.84$ with the Nucleolus, contrasts negative weights on FBL, PCNA with positive weights on GSK3A/B and CAV1. Because this is a Low--Low region, the mask comprises locations with low projected values, i.e.\ high expression of the negatively weighted proteins, that are surrounded by similarly scoring locations. These include fibrillarin (FBL), a canonical nucleolar marker, together with the nuclear proteins PCNA and CCNE1, whereas the positively weighted proteins mark cytoplasmic, endosomal (SARA), caveolar (CAV1) and mitochondrial (HSP60) compartments~\cite{gut2018multiplexed}. SAIL Mode 1, whose Low--Low hotspots achieve a Dice score of $0.59$ with Cell periphery 2, places positive weights on p-MEK and TSG101 and a negative weight on ISG15, corresponding to the cell periphery (Fig.~\ref{4i_figure} (b))~\cite{gut2018multiplexed}.

Overall, these results show that unsupervised SAIL Modes recover spatially coherent molecular programs that correspond closely to known subcellular compartments. This demonstrates that learning SAC directly from high-dimensional molecular measurements can reveal biologically meaningful spatial coherence without supervision.

\begin{table}[!t]
\centering
\caption{SAIL segmentation performance on the 4i dataset. We compare SAIL against classical Moran's $I$ computed on the best-performing individual protein among the $20$ proteins with the highest global Moran's $I$ (Univariate), the multivariate extension of Moran's $I$ obtained by averaging the univariate Moran's $I$ statistic across all proteins (Mean), MULTISPATI-PCA~\cite{dray2008spatial} and Hotspot~\cite{detomaso2021hotspot}. For each compartment, we report the Dice score between LISA-style hotspot mask and region annotation.}
\label{tab:dice_performance}
\setlength{\tabcolsep}{3pt}
\begin{tabular}{lccccc}
\toprule
Region/DICE Score & SAIL & Univariate & Mean & MULTISPATI & Hotspot \\
\midrule
Nucleolus & $\mathbf{0.84}$ & $0.75$ & $0.50$ & $0.79$ & $0.71$ \\ \midrule
Endosomes/Golgi 1 & $\mathbf{0.58}$ & $0.38$ & $0.34$ & $0.38$ & $0.56$ \\ \midrule
Cell periphery 1 & $\mathbf{0.53}$ & $0.51$ & $0.51$ & $0.50$ & $0.50$ \\ \midrule
Cell periphery 2 & $\mathbf{0.59}$ & $0.49$ & $0.50$ & $0.49$ & $0.50$ \\
\bottomrule
\end{tabular}
\end{table}

\begin{figure}[!t]
\centering
\includegraphics[width=3in]{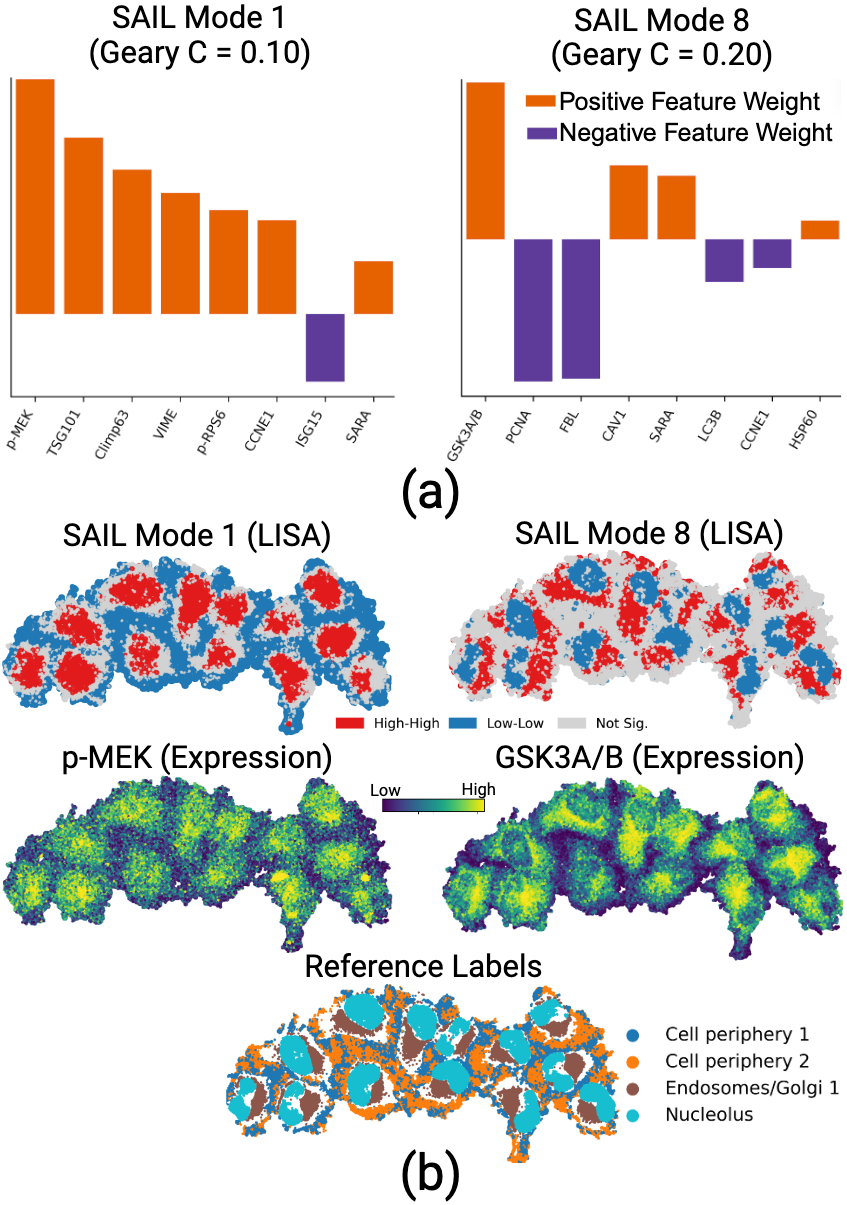}
\caption{Unsupervised SAIL Modes recover known subcellular compartments in the 4i spatial proteomics dataset. (a) Two data-driven SAIL Modes are visualised over the field of view. SAIL Mode 1 (showing the highest Dice score with Cell periphery 2) is defined by positive weights on p-MEK, TSG101 and Climp63 and a negative weight on ISG15. SAIL Mode 8 (showing the highest Dice score with the Nucleolus) is defined by positive weights on GSK3A/B, CAV1 and SARA and negative weights on PCNA, FBL and LC3B. (b) Visualisations of LISA High--High and Low--Low scoring locations for both SAIL Modes, which were used for defining spatial regions. We also show the expression of the most important proteins for each SAIL Mode and the reference annotation mask.}
\label{4i_figure}
\end{figure}

\subsection{Discovering spatial coherence biomarkers of treatment response}
\begin{table}[!t]
\centering
\caption{Prediction of response to PD-1/PD-L1 blockade in the cSCC cohort. We report discriminative performance as the area under the ROC curve (AUC) with 95\% confidence interval, together with the two-sided Mann--Whitney U p-value for separating responders from non-responders. SAIL is compared against the Favourable Niches classifier and pre-treatment PD-L1 status~\cite{lee2026single}.}
\label{tab:pdl1_response}
\begin{tabular}{lcc}
\toprule
Model & \makecell{AUC \\ (95\% CI)} & \makecell{Mann--Whitney $U$ \\ ($p$-value)}\\
\midrule
SAIL (Mode 18) & \textbf{0.86 (0.64--1.00)} & \textbf{0.005} \\
Favourable Niches & 0.78 (0.55--0.95) & 0.021 \\
PD-L1+ (untreated slide) & 0.78 (0.49--1.00) & 0.072 \\
PD-L1+ (all slides) & 0.69 (0.42--0.92) & 0.19 \\

\bottomrule
\end{tabular}
\end{table}

\begin{figure}[!t]
\centering
\includegraphics[width=3in]{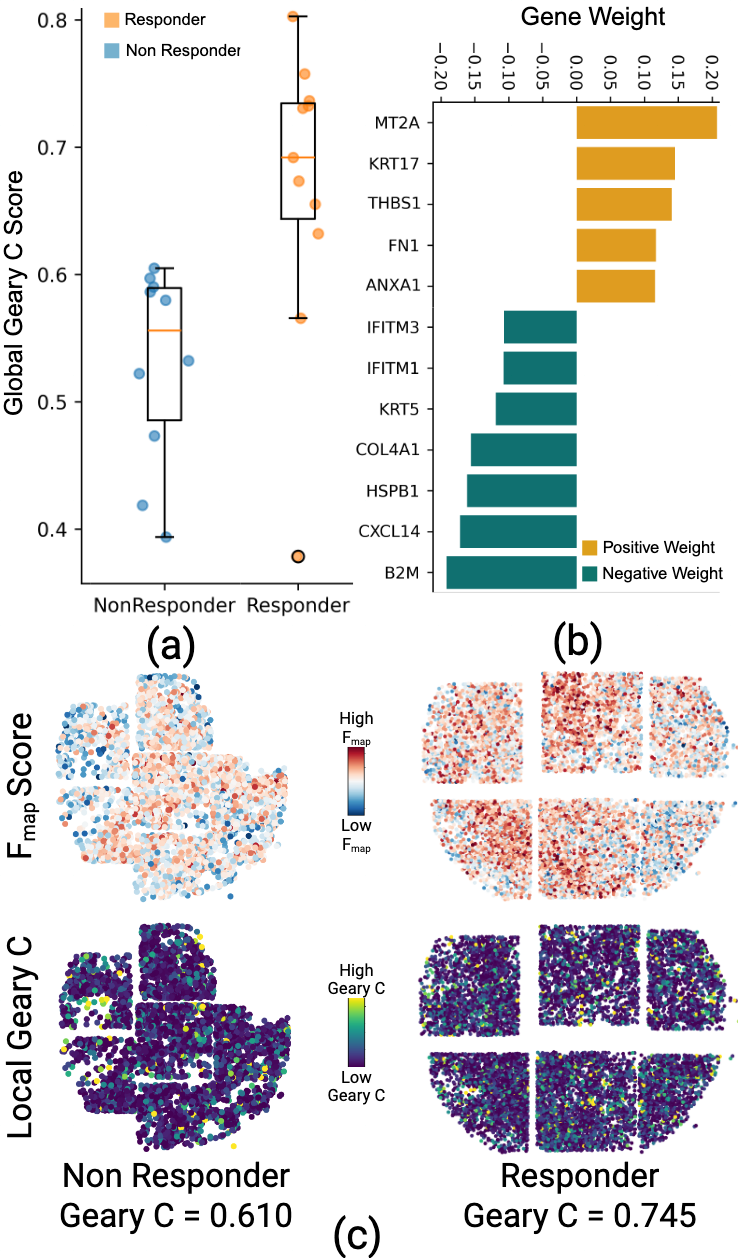}
\caption{Association of an unsupervised SAIL Mode with response to PD-1/PD-L1 blockade. (a) Per-patient global Geary-style dissimilarity SAIL score for SAIL Mode 18, compared between responders and non-responders. Each dot represents a patient; responders show a higher score ($p=0.005$), suggesting greater local cell heterogeneity. (b) Gene weights defining SAIL Mode 18. (c) Visualisation of responder and non-responder patients. We show the local $f_{\mathrm{map}}$ and Geary-style dissimilarity score per cell.}
\label{pd1_response}
\end{figure}

Response to PD-1/PD-L1 immune checkpoint blockade in cutaneous squamous cell carcinoma (cSCC) is highly heterogeneous, and reliable pre-treatment predictors of which patients will respond remain limited. PD-L1 expression, the standard clinical biomarker, is only weakly associated with response. Recent work has instead shown that the spatial organisation of the tumour microenvironment, such as how tumour, stromal and immune cells are arranged relative to one another, predicts pathologic response more strongly than PD-L1 status~\cite{lee2026single}. This suggests that response may depend not only on which cell states are present, but on how coherently or heterogeneously they are spatially arranged. This is precisely the property that SAC quantifies. We therefore test whether SAIL can be used to determine if a SAC mode of gene expression is predictive of response to PD-1/PD-L1 blockade.

\subsubsection{Dataset and Setup}

We analyse the publicly available dataset of Lee \textit{et al.}~\cite{lee2026single}, comprising single-cell spatial transcriptomic profiles of cSCC tumour biopsies from patients treated with neoadjuvant PD-1/PD-L1 blockade (cemiplimab or atezolizumab), profiled using the NanoString CosMx spatial molecular imaging platform with a 1000-plex gene panel. Tissue microarrays were generated, with each core containing on average $4{,}500$ cells, resulting in a total of $140{,}633$ cells from $11$ responder and $10$ non-responder patients. Following the original study, patients are classified as responders (pathologic complete response or major pathologic response, i.e.\ $<10\%$ residual tumour cells) or non-responders (all other patients).

We construct one spatial graph per patient, in which nodes are segmented cells, coordinates are cell centroids, and node features are the 1000-gene expression profiles. Each of a patient's tissue cores forms a separate $k$-nearest-neighbour component, with no edges between cores. 

SAIL is applied in the unsupervised mode-learning configuration (Section~\ref{sec:modes}), minimising a Geary-style objective over $28$ pre-treatment biopsies from $N=21$ patients. Response labels are not used during training. We use unsupervised rather than supervised learning as with only $21$ patients and $1000$ genes, a supervised model would be highly prone to overfitting. For each of the $20$ learned modes, we compute a per-patient global SAIL score and compare it between responders and non-responders using a two-sided Mann--Whitney $U$ test, with discriminative performance summarised by the area under the receiver operating characteristic curve (AUC). 

We compare SAIL against the two predictors reported in the original study~\cite{lee2026single}. The first is the niche-based classifier of Lee \textit{et al.} (Favourable Niches), which predicts response from the composition of six spatial niches identified by clustering cells according to their neighbourhood cell-type composition. These niches therefore describe which cell types co-occur, but are not optimised with respect to SAC. The second is pre-treatment PD-L1 positivity, the standard clinical biomarker, evaluated both on the untreated slide alone and on all slides.  

\subsubsection{Results}

As shown in Table~\ref{tab:pdl1_response}, the most discriminative SAIL Mode (Mode 18) separates responders from non-responders with AUC$=0.86$ (95\% CI 0.64--1.00; $p=0.005$). This exceeds both the Favourable Niches classifier (AUC$=0.78$, $p=0.021$) and pre-treatment PD-L1 status (untreated slide: AUC$=0.78$, $p=0.072$; all slides: AUC$=0.69$, $p=0.19$). Responders exhibit higher Mode 18 scores than non-responders (Fig.~\ref{pd1_response}(a)). Because the global score is a Geary-style dissimilarity, a higher score corresponds to greater local contrast between neighbouring cells. Responders' tumours are therefore more spatially heterogeneous along this mode.

Mode 18 (Fig.~\ref{pd1_response}(b)) contrasts genes associated with squamous tumour and cancer-associated-fibroblast/extracellular-matrix states (KRT17, FN1, THBS1, MT2A)~\cite{garcia2025cancer} with genes associated with interferon signalling and MHC class I antigen presentation (B2M, IFITM1, IFITM3)~\cite{jeong2023transcriptome} and with basal epidermal or basement-membrane programs (KRT5, COL4A1, CXCL14)~\cite{szabo2015statistically}. The higher scores in responders thus indicate sharper spatial juxtaposition of interferon-associated and tumour--stromal programs (Fig.~\ref{pd1_response}(c)).

These findings align with, and extend, existing work. They support the conclusion of Lee \textit{et al.}~\cite{lee2026single} that spatial organisation is more informative of response than PD-L1 status, and are compatible with the established association between an inflamed, spatially intermixed immune microenvironment and response to immunotherapy~\cite{galon2019approaches}. Unlike the niche-based approach, however, SAIL does not require a predefined clustering pipeline. A single, interpretable gene projection discovered without labels, whose SAC is directly quantified, provides stronger discrimination than both existing predictors.

\subsection{Unsupervised SAC modes for prediction of patient survival}
\label{sec:prognosis_sail}
A further question to study is whether SAC carries prognostic information about patient outcome. In cancer tissue, disease progression is often governed not only by the presence of specific cellular or molecular programs but also by how these programs are spatially organised in the tissue. For example, spatially coherent immune infiltration or tumour--stroma interactions are known to influence patient outcome~\cite{valdeolivas2024profiling}. SAIL provides a direct way to analyse such structure by learning feature projections that extremise SAC. Unlike the prevailing use of supervised approaches for survival prediction, SAIL Modes are learned without access to survival outcomes~\cite{keller2025histokernel,keller2025insight}. If SAC contains a prognostic signal, then SAIL Modes should correlate with patient survival despite being unsupervised.

\subsubsection{Datasets and Graph Construction}

We evaluate SAIL on two spatial biomedical modalities: cellular spatial proteomics and histological whole-slide images (WSIs). Each sample is represented as a spatial graph where nodes correspond to spatial measurement units, and edges capture neighbourhood structure. We analyse the ORION-CRC cohort~\cite{lin2023high}, consisting of $N=36$ patient slides with approximately $1.8$ million cells per slide. Each slide is represented as a cellular spatial graph where nodes are segmented cells with 2D spatial coordinates. Node features consist of 17 multiplex immunofluorescence (mIF) marker intensities. Graph edges are constructed using a $k$-nearest-neighbour graph in coordinate space ($k_{\mathrm{NN}}=32$). Across the dataset this yields $n_{\mathrm{total}}=65{,}327{,}071$ nodes. We further analyse WSIs from TCGA breast cancer (BRCA; $N=498$ patients) and colorectal cancer (CRC; $N=558$ patients) cohorts~\cite{weinstein2013cancer}. Each WSI is converted into a spatial graph by extracting $512\times512$ image patches at $0.5$ microns per pixel and representing each patch as a node with spatial coordinates given by its centroid. Node features correspond to deep patch embeddings of dimensionality $d=1024$ (ShuffleNet and UNI features for BRCA and CRC, respectively)~\cite{mackenzie2022neural,chen2024towards}. Each WSI contains on average $n\approx1281$ nodes, yielding $n_{\mathrm{total}}=1{,}382{,}718$ nodes across both datasets. For each cohort, SAIL is applied in an unsupervised configuration by minimising a Geary-style objective, without access to survival outcomes or clinical labels. For each graph, we compute a global SAIL score that summarises each mode's strength across the tissue. Patients are then stratified into high- and low-SAC groups using a median split, and differences are assessed via Kaplan--Meier analysis, concordance index (C-Index), hazard ratio (HR), and the difference in restricted mean survival time ($\Delta$RMST). Additional implementation details are given in Supplementary Section~\ref{sec:supp_prognosis}

\begin{table*}[!t]
\centering
\caption{Prognostic performance of SAIL on the ORION cohort. We compare SAIL against classical Geary's $C$ computed on individual protein markers (Univariate) and the multivariate extension of Geary's $C$ obtained by averaging the univariate Geary's $C$ (Mean), MULTISPATI-PCA~\cite{dray2008spatial}, and Hotspot~\cite{detomaso2021hotspot}. Patients are stratified into groups via a median split on each score; C-Index, hazard ratio (HR) with 95\% CI, restricted mean survival time difference ($\Delta$RMST), and log-rank $p$-value are reported.}
\label{tab:prognosis_performance}
\begin{tabular*}{\textwidth}{@{\extracolsep{\fill}}lccccc}
\toprule
Metric & SAIL & Univariate & Mean & MULTISPATI & Hotspot \\
\midrule
C-Index & $\mathbf{0.74}$ & $0.69$ & $0.68$ & $0.70$ & $0.60$ \\
$\Delta$RMST (days) & $\mathbf{782}$ & $385$ & $362$ & $617$ & $777$ \\
HR (95\% CI) & $\mathbf{5.23}$ \textbf{(1.45--18.89)}& $0.44$ (0.15--1.31) & $3.15$ (0.98--10.08) & $3.30$ (1.03--10.56) & $4.89$ (1.36--17.59) \\
Log-rank $p$-value & $\mathbf{0.005}$ & $0.13$ & $0.04$ & $0.03$ & $0.007$ \\
\bottomrule
\end{tabular*}
\end{table*}

\begin{figure*}[!t]
  \centering
  \includegraphics[width=\textwidth]{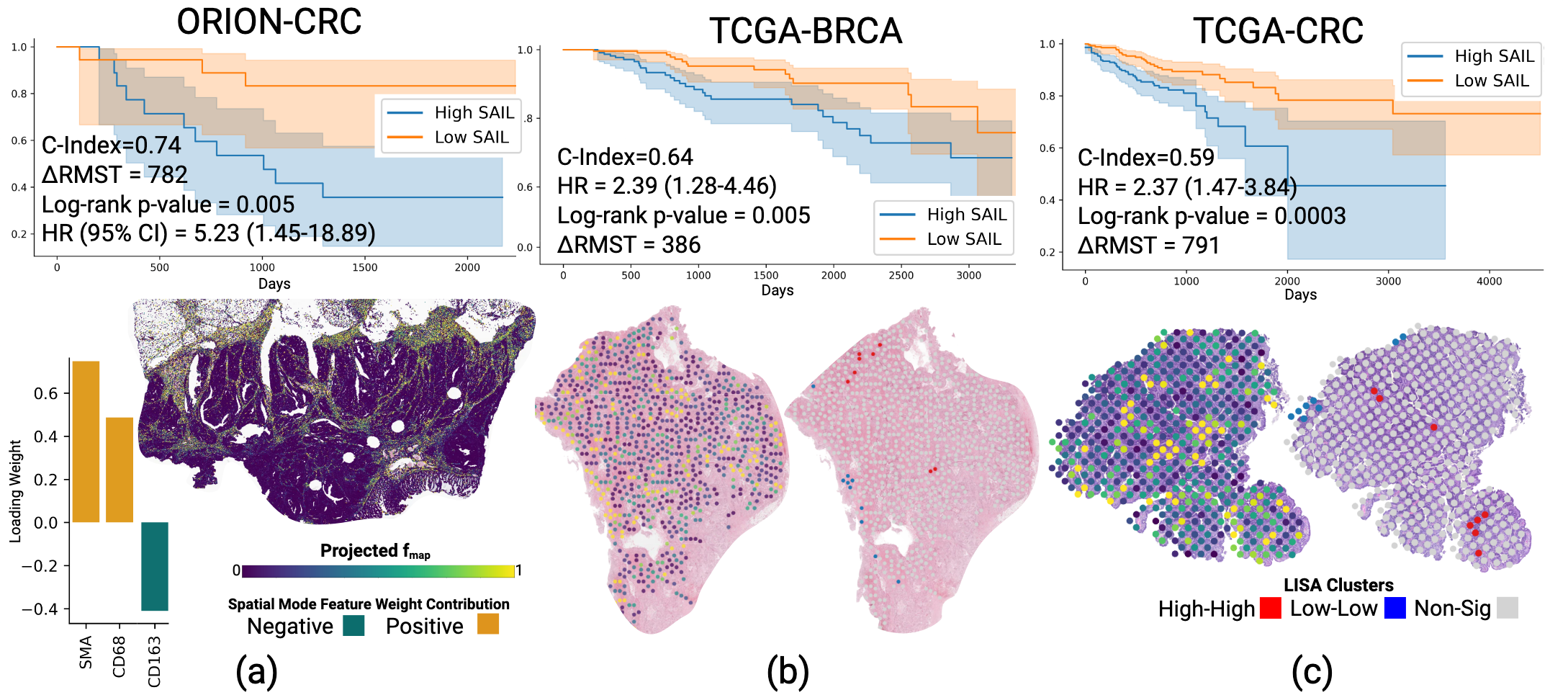}
  \caption{Unsupervised SAIL Modes link spatial coherence to outcome. (a) A SAIL Mode discovered in ORION-CRC (mIF). Bar chart shows $f_{\mathrm{map}}$ weights. The corresponding projected scalar field $f_{\mathrm{map}}(x)$ is overlaid on a tissue section to visualise its spatial pattern, and the associated global SAIL scores stratify patients into distinct survival groups. (b) SAIL learns a scalar field $f_{\mathrm{map}}(x)$ on TCGA-BRCA WSI deep-feature patch graphs (Mode 1); slides are visualised by $f_{\mathrm{map}}(x)$ alongside hotspot clusters. The corresponding global SAIL score (median split) defines risk groups with significant Kaplan--Meier separation in disease-specific survival. (c) Similarly to (b), we analyse the TCGA-CRC cohort.}
  \label{survival_analysis}
\end{figure*}

\subsubsection{Results}
Across cohorts, SAC captured by SAIL Modes is associated with outcome (Fig.~\ref{survival_analysis}; all modes in Supplementary Table~\ref{tab:supp_tcga_modes}). In the ORION cohort (Table~\ref{tab:prognosis_performance}; all modes in Supplementary Table~\ref{tab:supp_orion_modes}), the best-performing SAIL Mode (Mode 17) stratifies patients by overall survival with the highest C-Index (C-Index $=0.74$, log-rank $p=0.005$, HR (95\% CI) $5.23$\,(1.45--18.89), $\Delta$RMST $=782$ days), exceeding univariate Geary's $C$ on the best individual protein (C-Index $=0.69$, $p=0.13$), the mean-aggregation baseline (C-Index $=0.68$, $p=0.04$), MULTISPATI-PCA (C-Index $=0.70$, $p=0.03$) and Hotspot (C-Index $=0.60$, $p=0.007$)~\cite{detomaso2021hotspot}. The gap to the univariate and mean baselines suggests that the prognostic signal is not carried by the SAC of any single marker, nor by an unweighted combination of all markers, but by a learned combination of markers whose joint SAC SAIL can isolate. Interpreting the learned projection (Fig.~\ref{survival_analysis}(a)) reveals that this mode contrasts SMA and CD68 (positive weights) with CD163 (negative weight). It thus separates signal from SMA$^{+}$ stroma (myofibroblasts and cancer-associated fibroblasts) and CD68$^{+}$ macrophages from signal associated with CD163-expressing, M2-like macrophages. Because the SAIL score is a Geary-style dissimilarity, high-risk patients are those in whom these stromal and macrophage-associated signals are spatially intermixed rather than organised into coherent regions. This is consistent with reports that cancer-associated fibroblasts promote the recruitment and M2 polarisation of tumour-associated macrophages in colorectal cancer~\cite{zhang2019cancer}, and that infiltration by CD163$^{+}$ tumour-associated macrophages is an independent adverse prognostic factor~\cite{xue2021prognostic}.

A similar prognostic association holds in histological imaging. In TCGA-BRCA, the best-performing mode (Mode 1) separates patients by disease-specific survival (C-Index $=0.64$, log-rank $p=0.005$, HR (95\% CI) $2.39$\,(1.28--4.46), $\Delta$RMST $=386$ days), with high SAIL scores associated with worse prognosis. Since the SAIL score is a Geary-style dissimilarity, high scores correspond to patients with spatially heterogeneous tissue. In TCGA-CRC, Mode 17 gives the strongest separation in overall survival (C-Index $=0.59$, log-rank $p=2.9\times10^{-4}$, HR (95\% CI) $2.37$\,(1.47--3.84), $\Delta$RMST $=791$ days).

Together, these results indicate that patient outcome is linked not only to the presence of morphological or molecular patterns captured by imaging or omics measurements, but also to how coherently these patterns are spatially organised across the tissue~\cite{valdeolivas2024profiling}. In ORION, this signal is considerably weaker when examining individual features or their aggregate.

\subsection{Supervised classification of mesothelioma subtypes with SAIL}
\label{sec:mesograph}
We have shown that unsupervised SAIL Modes recover biologically meaningful signals. Beyond unsupervised discovery, SAIL also lets us ask a supervised question: for a given label, how much of the discriminative information lies in the SAC of features rather than in the features themselves?

A natural setting for this question is histological subtyping, where tissue classes are associated with cellular spatial coherence. One example is malignant mesothelioma which is classified into epithelioid (E) and non-epithelioid (NE) subtypes. These subtypes are clinically important as epithelioid tumours carry a substantially better prognosis and influence eligibility for surgery and immunotherapy benefit~\cite{neilly2024contemporary}. Yet subtyping from small biopsies is difficult and shows only moderate inter-observer agreement~\cite{brvcic2014reproducibility}.

There are strong morphological reasons to expect SAC to be informative for this task. Epithelioid mesothelioma typically grows as cohesive sheets of morphologically similar cells, so neighbouring cells tend to resemble one another, whereas non-epithelioid tumours show less organised growth, with greater variation between neighbouring cells. The subtypes may therefore differ not only in which cell states are present, but in how coherently they are spatially organised. We hypothesise that this second-order property, which SAIL quantifies explicitly, carries a substantial part of the subtype signal.

\subsubsection{Dataset and Setup}

To test this hypothesis, we use the publicly available St.~George's Hospital (SGH) mesothelioma cohort~\cite{eastwood2023mesograph}. It comprises $234$ tissue microarray (TMA) samples from $90$ patients ($148$ epithelioid, $86$ non-epithelioid). Each sample is represented as a cell graph whose nodes are segmented cells, with node features comprising $23$ morphological, $43$ haematoxylin and eosin (H\&E) intensity, $39$ texture and $512$ ResNet-derived descriptors. We evaluate two feature sets: all features ($617$ descriptors) and interpretable features only, which excludes the ResNet descriptors and keeps the $105$ hand-crafted morphological, intensity and texture features.

SAIL is trained in the supervised configuration. It learns $h$ feature projections and affinities, and the resulting global SAC descriptors $I^{\mathrm{global}}$ are passed to prediction head $f_{\mathrm{pred}}$ trained end-to-end against subtype labels. Because $f_{\mathrm{pred}}$ only sees SAC descriptors, predictive performance measures how much subtype information is carried by SAC. Implementation details are given in Supplementary Section~\ref{sec:supp_mesograph}.

To quantify this, we compare SAIL against predictors that do not model SAC explicitly but operate on the same cell graphs and, in each setting, the same features. These are MesoGraph~\cite{eastwood2023mesograph}, a graph neural network, and two multiple-instance learning baselines (PINS and Max-MIL). If SAC is a major driver of subtype, an SAC-only predictor should match or exceed these models.

We use leave-one-slide-out cross-validation, in which all samples from one slide are held out as the test set in each fold, and split the remaining samples $75\%$--$25\%$ into training and validation sets. AUC and average precision (AP) are computed with non-epithelioid samples as the positive class.

\subsubsection{Results}
Table~\ref{tab:meso_subtype} shows that SAIL achieves the strongest discriminative performance with both feature sets. With all features, SAIL (AUC$=0.902\pm0.035$, AP$=0.849\pm0.055$) exceeds MesoGraph (AUC$=0.863\pm0.065$), PINS (AUC$=0.844\pm0.017$) and Max-MIL (AUC$=0.782\pm0.050$), and a paired bootstrap test on the difference in AUC ($10{,}000$ resamples) confirms that these improvements are significant (MesoGraph and PINS: $p=0.011$; Max-MIL: $p=0.0002$). With interpretable features only, SAIL (AUC$=0.866\pm0.042$, AP$=0.814\pm0.040$) again ranks first, ahead of MesoGraph (AUC$=0.846\pm0.024$), PINS (AUC$=0.825\pm0.027$) and Max-MIL (AUC$=0.721\pm0.038$). In this setting the improvement over Max-MIL is significant ($p=0.0002$), whereas the differences from MesoGraph and PINS are not. 

Fig.~\ref{meso_results_figure}(a) shows the two SAIL Modes with the largest weights in the prediction head for the model trained on interpretable features: Mode~11 (weight $+0.57$) and Mode~12 (weight $+0.46$). Mode~11 loads positively on texture entropy and eosin heterogeneity and negatively on nuclear and cellular solidity (convexity). Mode~12 contrasts solid cells with small nuclei and homogeneous haematoxylin texture against cells with larger nuclei, rounder neighbouring nuclei and more variable eosin texture. Because the local operator is Geary-style, a high score indicates that neighbouring cells take markedly different values along the mode. Since both weights are positive, greater local dissimilarity along either mode is evidence for the non-epithelioid subtype. Consistent with this, Fig.~\ref{meso_results_figure}(b) shows that the non-epithelioid sample contains many scattered cells with high local dissimilarity, where cells differing in shape and texture are intermixed. Conversely, the epithelioid sample is largely uniform, matching the cohesive sheets of similar cells characteristic of epithelioid growth.

These results support our hypothesis. A predictor restricted to SAC descriptors matches or outperforms models with direct access to the same cell-level features, with or without deep image features. This indicates that a substantial part of the signal separating epithelioid from non-epithelioid mesothelioma is second-order: it depends on how similar neighbouring cells are, not merely on which cell states are present.

\begin{table}[!t]
\centering
\caption{Performance of supervised SAIL for mesothelioma subtype prediction~\cite{eastwood2023mesograph}, using all $617$ cell features (top) or only the $105$ interpretable features (bottom). We report AUC and average precision (AP) with their standard deviation ($\sigma$) across cross-validation folds. Comparisons include graph-based state of the art (MesoGraph) and standard multiple-instance learning pipelines (PINS, Max-MIL).}
\label{tab:meso_subtype}
\begin{tabular}{lcc}
\toprule
Model & AUC $\pm$ $\sigma$ & AP $\pm$ $\sigma$\\
\midrule
\multicolumn{3}{c}{All features ($d=617$)} \\
\midrule
SAIL & $\mathbf{0.902 \pm 0.035}$ & $\mathbf{0.849 \pm 0.055}$\\
MesoGraph & $0.863 \pm 0.065$ & $0.832 \pm 0.051$ \\
PINS & $0.844 \pm 0.017$ & $0.821 \pm 0.036$\\
Max-MIL & $0.782 \pm 0.050$ & $0.698 \pm 0.055$ \\
\midrule
\multicolumn{3}{c}{Interpretable features only ($d=105$)} \\
\midrule
SAIL & $\mathbf{0.866 \pm 0.042}$ & $\mathbf{0.814 \pm 0.040}$\\
MesoGraph & $0.846 \pm 0.024$ & $0.806 \pm 0.054$ \\
PINS & $0.825 \pm 0.027$ & $0.775 \pm 0.104$\\
Max-MIL & $0.721 \pm 0.038$ & $0.627 \pm 0.090$ \\
\bottomrule
\end{tabular}
\end{table}

\begin{figure}[!t]
\centering
\includegraphics[width=3in]{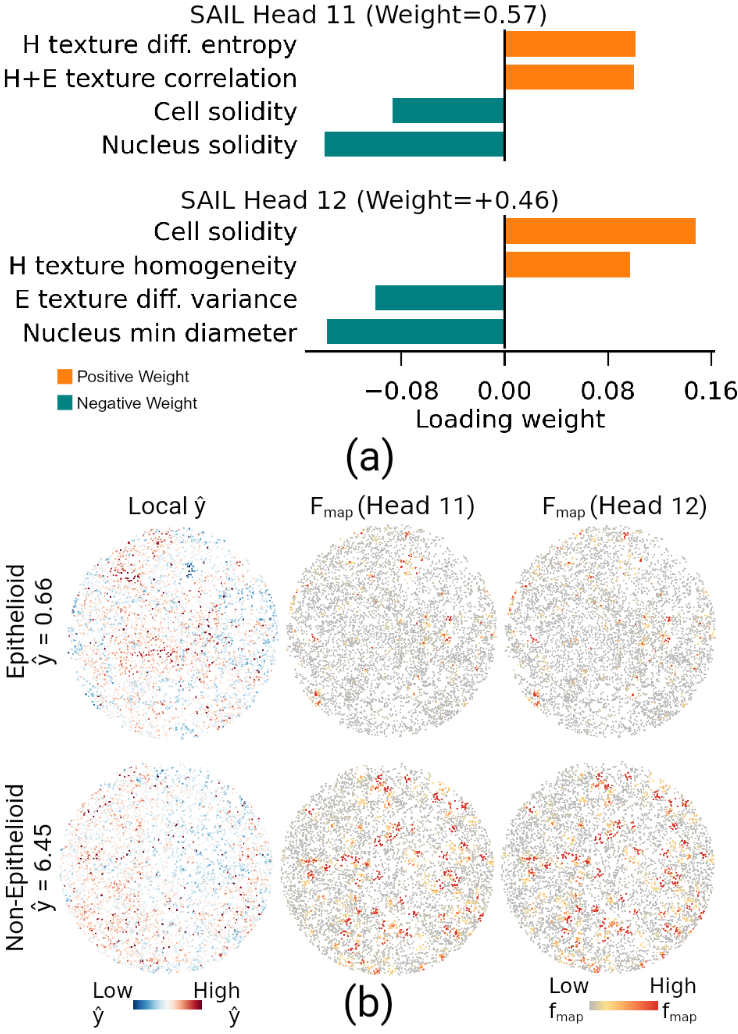}
\caption{Supervised SAIL discriminates mesothelioma subtype from the spatial coherence of interpretable cell features (model trained on the $105$ interpretable features only). (a) $f_{\mathrm{map}}$ loadings of the two SAIL Modes with the largest weights in the prediction head (Mode~11, weight $+0.57$; Mode~12, weight $+0.46$; both towards non-epithelioid), shown for the $12$ features with the largest absolute loading. (b) Local predicted value $\hat y$ and local Geary-style SAIL scores for these modes, shown for a representative epithelioid and non-epithelioid sample.}
\label{meso_results_figure}
\end{figure}

\subsection{Additional Results}
\label{sec:additional_results}
For brevity, proofs and further experiments are reported in the Supplementary Material:
\begin{itemize}
    \item \textbf{Proofs.} Full proofs of Theorems~\ref{theorem:classical_reduction} and~\ref{theorem:inductive_bias} are given in Supplementary Section~\ref{sec:supp_Theoretical_Analysis}.

    \item \textbf{Computational complexity.} In Supplementary Section~\ref{sec:supp_scaling}, we analyse the time and memory cost of SAIL. With a sparse $k_{\mathrm{NN}}$ graph, a training step costs $\mathcal{O}(ndh + nk_{\mathrm{NN}}h + nh^2)$ time and $\mathcal{O}(nd + nk_{\mathrm{NN}}h)$ memory, which is linear in the number of nodes $n$. The Stiefel constraint adds $\mathcal{O}(d^3)$ term.

    \item \textbf{Synthetic signals and 3D point clouds.} In Supplementary Section~\ref{sec:supp_toy_construction}, we apply SAIL to structured synthetic signals (smooth noise fields, checkerboards, step boundaries and concentric rings) and to a 3D point cloud. In each case, SAIL recovers the expected pattern of positive or negative spatial autocorrelation.

    \item \textbf{Ablation study.} In Supplementary Section~\ref{sec:ablation}, we ablate the components of SAIL ($f_{\mathrm{map}}$, $f_{\mathrm{local}}$, $f_{\mathrm{act}}$, $f_{\mathrm{reduce}}$, $f_{\mathrm{affinity}}$ and $\tau$) and of its loss (decorrelation and sparsity) on four datasets. SAIL is robust to most choices. Strong decorrelation degrades performance, Geary's $C$ and Moran's $I$ consistently outperform Getis--Ord $G^\star$, and a learnable spatial temperature helps when many graphs are analysed jointly.

    \item \textbf{BRAF-mutant CRC prognosis.} In Supplementary Section~\ref{sec:supp_prognosis}, we analyse an additional colorectal cancer cohort with cell-type labels~\cite{trahearn2025computational}. An unsupervised SAIL Mode contrasting stromal cells (myocytes, fibroblasts and endothelial cells) with immune cells (lymphocytes and neutrophils) stratifies patients by overall survival (log-rank $p=0.002$, HR $=1.83$). It outperforms all univariate cell-type baselines.

    \item \textbf{Macroeconomic geospatial data.} In Supplementary Section~\ref{sec:supp_macro}, we apply SAIL to the Global Macro Database~\cite{muller2025global}. It discovers modes of cross-country spatial coherence among economic indicators and tracks how they evolve from 2000 to 2023.

    \item \textbf{COVID-19.} In Supplementary Section~\ref{sec:supp_covid}, we use SAIL to track the spatial coherence of weekly COVID-19 case rates across countries. This reveals recurring periods of regional synchronisation, with hotspots concentrated in Western and Central Europe.

    \item \textbf{Mixing gases.} In Supplementary Section~\ref{sec:supp_diffusion}, we use SAIL to follow the mixing of two simulated gases by diffusion. Without supervision, it learns a mode based on species identity whose score decays as the gases mix, faster when the partition is removed than through a small hole, and stays constant when the partition is closed.

    \item \textbf{Scope and limitations.} In Supplementary Section~\ref{sec:supp_scope}, we vary how strongly the target depends on SAC. This empirically delineates the problem class $\mathcal{P}_{\mathrm{SAC}}$ of Theorem~\ref{theorem:inductive_bias}. SAIL outperforms message-passing architectures when the target is driven by spatial autocorrelation, but not when it is driven by feature content.
\end{itemize}

\section{Conclusion} 
We presented SAIL, a learnable formulation of SAC for high-dimensional spatial graphs. By jointly learning feature projections and spatial affinities, SAIL makes SAC an explicit quantity that can be analysed, optimised, and interpreted within modern learning pipelines. Across synthetic and real-world datasets, we observed that modelling SAC directly reveals structure that is difficult to capture through feature aggregation alone. In particular, unsupervised SAIL Modes recover spatially coherent biological niches, show a consistent association with response to PD-1/PD-L1 blockade, and expose prognostically relevant organisation in large biomedical datasets, while SAIL's supervised descriptors alone are sufficient to discriminate mesothelioma subtype, suggesting that SAC itself carries meaningful signal. More broadly, SAIL provides a practical bridge between classical spatial statistics and modern representation learning for cross-domain applications from geospatial analysis to biomedicine. Future work will explore generalisation to cross-channel coherence statistics.


\clearpage
\setcounter{algorithm}{1}
\begin{center}
{\Large\bfseries Supplementary Material for:\\
SAIL: A Learnable Generalisation of Spatial Autocorrelation\\over Graph-Structured Data}\\[1ex]
Piotr~Keller and Fayyaz~Minhas
\end{center}
\bigskip
\setcounter{tocdepth}{2}
\tableofcontents

\appendices

\section{Biological and Clinical Background}
\label{sec:supp_bio_background}
Since this work applies SAIL to biomedical data, we briefly summarise the main domain concepts and terminology.

\subsection{Spatial Omics Modalities}
\emph{Omics} refers to the large-scale measurement of molecular features within a biological sample. \emph{Spatial omics} technologies additionally record \emph{where} each measurement was taken, producing molecular readouts paired with coordinates.
\begin{itemize}
    \item \textbf{Spatial transcriptomics} measures gene expression (the abundance of RNA transcripts, i.e.\ how ``active'' each gene is) at many locations across a tissue.
    \item \textbf{Spatial proteomics} measures protein abundance rather than gene expression. We utilise fluorescent markers whose presence indicates a particular cell type or state (e.g., CD8a marks cytotoxic T cells, CD31 marks endothelial/blood-vessel cells).
\end{itemize}

\subsection{Tissue Structure and Spatial Niches}
Biological tissue is spatially organised: cells of similar type and function tend to cluster into coherent regions. A \emph{spatial niche} (or \emph{compartment}) is such a region, defined by a characteristic molecular or cellular composition, for example a subcellular compartment such as the nucleolus, nucleus, endoplasmic reticulum or cell periphery. Identifying niches is a core goal of spatial biology because tissue function depends not only on \emph{which} cells are present but on \emph{how} they are spatially arranged.

\subsection{Histology and Whole-Slide Images}
A whole-slide image (WSI) is a high-resolution digital scan of a stained tissue slide, often 100,000 px by 100,000 px. Because a WSI is too large to process directly, it is tiled into small patches; each patch is passed through a pretrained neural network or model to obtain a feature representation for each patch.

\subsection{Reference Annotations and Overlap (Dice)}
Biologically informed annotations provide reference masks. To measure how well a discovered SAIL Mode matches an annotated region, we use the Dice score, an overlap score between two regions equal to twice their intersection divided by the sum of their sizes, ranging from $0$ (no overlap) to $1$ (perfect overlap).

\subsection{Survival Analysis}
A central clinical question is whether a measured quantity predicts patient \emph{outcome}, most commonly \emph{survival time} (time from diagnosis or surgery until a particular event such as death or disease progression). Survival data are \emph{censored}: for patients still alive at the end of the study, the true survival time is unknown, only that it exceeds the observation window. Standard classification metrics cannot handle censoring, so specialised tools are used:
\begin{itemize}
    \item \textbf{Median split.} To test whether a continuous score (here, a global SAIL score) is prognostic, patients are divided into two groups by thresholding at the median value, giving a ``high'' and a ``low'' group. Outcomes are then compared between groups.
    \item \textbf{Kaplan--Meier analysis.} A method to estimate and visualise the fraction of patients surviving over time from censored data. Plotted as a step curve per group; well-separated curves indicate the groups have different prognoses.
    \item \textbf{Log-rank test.} A statistical hypothesis test for whether two Kaplan--Meier survival curves differ significantly.
    \item \textbf{Concordance index (C-Index).} The survival analogue of the AUC: the probability that, for a random pair of patients, the model correctly orders who survives longer. A value of $0.5$ is random; values $>0.5$ indicate better discrimination.
    \item \textbf{Hazard ratio (HR).} From a Cox proportional-hazards model, the relative instantaneous risk of the event in one group versus another. $\mathrm{HR}>1$ means higher risk; $\mathrm{HR}=1$ means no difference.
    \item \textbf{Restricted mean survival time (RMST).} The average event-free time over a fixed follow-up window; $\Delta$RMST is the difference between groups, giving an interpretable ``extra months/days survived'' quantity.
\end{itemize}

\section{Additional Classical Spatial Statistics}
\label{sec:supp_classical_indices}
Below we summarise two further indices used in this work: Getis--Ord $G^\star$, which measures local hotspot concentration, and the Morisita--Horn index, which measures neighbourhood compositional overlap.

\subsection{Getis--Ord $G^\star$} Getis--Ord $G^\star$ identifies local ``hotspots'' and ``coldspots''. For a scalar feature with values $x_1,\dots,x_n$ (dropping the graph index $i$ and feature index $k$ for readability), let $\bar x=n^{-1}\sum_{j'}x_{j'}$ denote its mean and $S=\sqrt{n^{-1}\sum_{j'}x_{j'}^2-\bar x^2}$ its population standard deviation. The standardised statistic at node $j$ is
\begin{equation}
G_j^{\star}
=
\frac{
\sum_{j'} w_{jj'} x_{j'} - \bar x \sum_{j'} w_{jj'}
}{
S\sqrt{
\dfrac{
n\sum_{j'} w_{jj'}^2 - \left(\sum_{j'} w_{jj'}\right)^2
}{n-1}
}
},
\label{eq:getisord}
\end{equation}
where the sums over $j'$ include $j'=j$, i.e.\ self-weights $w_{jj}>0$ are allowed, as in the starred statistic. $G_j^\star$ is a $z$-score: large positive values indicate that node $j$ and its neighbours collectively take high values (a potential hotspot), while large negative values indicate a coldspot.

\subsection{Morisita--Horn Index} The Morisita--Horn index measures compositional similarity between two distributions and is commonly used in ecology to quantify overlap between species abundance profiles. In spatial settings, it can be used to compare the feature composition of a node with that of its spatial neighbourhood, providing a measure of neighbourhood compositional overlap.

Let $x_{ij}\in\mathbb{R}_{\geq 0}^{d}$ denote the non-zero, non-negative feature vector of node $j$, and let
\begin{equation}
p_{ij}=\frac{x_{ij}}{\|x_{ij}\|_1}
\end{equation}
be its relative-abundance vector. We define the affinity-weighted neighbourhood composition as
\begin{equation}
m_{ij}^{(k)} = \sum_{j'} w_{jj'}\, p_{ij'}^{(k)},
\label{eq:mh_neighbour}
\end{equation}
which, for row-stochastic weights ($\sum_{j'}w_{jj'}=1$), is itself a relative-abundance vector describing the composition of the neighbourhood of node $j$. The Morisita--Horn similarity between node $j$ and its neighbourhood is then
\begin{equation}
\mathrm{MH}_{ij} =
\frac{
2 \sum_{k=1}^{d} p_{ij}^{(k)} m_{ij}^{(k)}
}{
\sum_{k=1}^{d} \left(p_{ij}^{(k)}\right)^2
+
\sum_{k=1}^{d} \left(m_{ij}^{(k)}\right)^2
}.
\label{eq:morisitahorn}
\end{equation}

This score lies in $[0,1]$, with larger values indicating greater compositional similarity between the node and its neighbourhood, and smaller values indicating compositional heterogeneity. In the context of spatial analysis, high Morisita--Horn values indicate regions where neighbouring nodes share similar feature compositions, whereas low values indicate spatial transitions between distinct compositions.

\section{Additional Input Feature Space Mapping}
\label{sec:supp_fmap}
\subsection{Linear Projection} The most basic choice is a dense linear layer,
\begin{equation}
\mathbf{Z} = \mathbf{X} \mathbf{W}, \quad \mathbf{W} \in \mathbb{R}^{d \times h},
\end{equation}
which learns spatial scalar fields as weighted combinations of input features.

\subsection{Orthogonal Projection} To encourage decorrelated and diverse spatial directions, we can constrain the projection matrix to lie on the Stiefel manifold:
\begin{equation}
\mathbf{Z} = \mathbf{X} \mathbf{W}, \quad \mathrm{subject to } \mathbf{W}^\top \mathbf{W} = \mathbf{I}_h.
\end{equation}
This orthogonality ensures that each column of \(\mathbf{Z}\) captures an independent spatial signal.

\subsection{Gumbel-Softmax Relaxation} To enable continuous relaxation of the discrete concept assignment while preserving differentiability, SAIL also supports the Gumbel-Softmax trick:
\begin{equation}
\mathbf{Z}_i^{(j)} = \frac{\exp\left((\log p_i^{(j)} + g_i^{(j)}) / \tau_{\mathrm{G}}\right)}{\sum_{k=1}^m \exp\left((\log p_i^{(k)} + g_i^{(k)}) / \tau_{\mathrm{G}}\right)},
\end{equation}
where \(p_i = \mathrm{softmax}(x_i \mathbf{V})\) and \(g_i^{(j)} \sim \mathrm{Gumbel}(0, 1)\). As the temperature \(\tau_{\mathrm{G}} \to 0\) (distinct from the spatial temperature $\tau$), this approximation converges to a true one-hot vector.

\subsection{Concept-Based (One-Hot) Projection} For discrete, interpretable representations, each feature vector is assigned to one of \(m\) learned ``concepts'' via a classification-style projection:
\begin{equation}
\mathbf{Z} = \mathrm{onehot}(\arg\max(\mathbf{X} \mathbf{V})) \in \{0,1\}^{n \times m}, \quad \mathbf{V} \in \mathbb{R}^{d \times m}.
\end{equation}
This hard assignment is implemented using a straight-through estimator:
\begin{equation}
\begin{aligned}
\mathbf{Z} &= \mathrm{onehot}(\arg\max\mathbf{H}) + (\mathbf{H} - \mathrm{stopgrad}(\mathbf{H})),\\
\mathbf{H} &= \mathrm{softmax}(\mathbf{X} \mathbf{V}),
\end{aligned}
\end{equation}
where gradients during backpropagation flow through \(\mathbf{H}\) while the forward pass uses the discrete one-hot encoding.

\section{Training Algorithms}
\label{sec:supp_algorithms}
Supplementary Algorithm~\ref{alg:sail_supervised} details supervised training of SAIL with a general task loss, and Supplementary Algorithm~\ref{alg:sail_modes} details unsupervised discovery of SAIL Modes; both call the SAIL forward pass of Algorithm~\ref{alg:sail_forward} of the main paper.

\begin{algorithm}[t]
\caption{Supervised Training of SAIL with a General Task Loss}
\label{alg:sail_supervised}
\begin{algorithmic}

\STATE \textbf{Input:}
Labelled dataset
$\mathcal{D}=\{(\mathcal{G}_i,y_i)\}_{i=1}^{N}$;
prediction head $f_{\mathrm{pred}}$;
task loss $\mathcal{L}_{\mathrm{task}}$;
learning rate $\eta$;
epochs $N_{\mathrm{epochs}}$;
and batch size $N_{\mathrm{batch}}$

\STATE \textbf{Parameters:}
SAIL parameters
$\theta_{\mathrm{map}}$,
$\theta_{\mathrm{affinity}}$,
$\tau>0$ and prediction parameters
$\theta_{\mathrm{pred}}$

\STATE \textbf{Return:}
optimised parameters
$\theta_{\mathrm{map}}^\star$,
$\theta_{\mathrm{affinity}}^\star$,
$\tau^\star$, and
$\theta_{\mathrm{pred}}^\star$

\STATE
Initialise
$\theta_{\mathrm{map}}$,
$\theta_{\mathrm{affinity}}$,
$\tau$, and
$\theta_{\mathrm{pred}}$

\FOR{$e=1,\ldots,N_{\mathrm{epochs}}$}

    \FOR{each mini-batch
    $\mathcal{B}\subset\mathcal{D}$,
    $|\mathcal{B}|=N_{\mathrm{batch}}$}

        \FOR{each $(\mathcal{G}_i,y_i)\in\mathcal{B}$}

            \STATE
            $I_i^{\mathrm{local}},I_i^{\mathrm{global}}
            \gets
            \textsc{SAIL\_Forward}
            \left(
            \mathcal{G}_i;
            \theta,
            \tau
            \right)$
            \COMMENT{Algorithm~\ref{alg:sail_forward} of the main paper}

            \STATE
            $\hat{y}_i
            \gets
            f_{\mathrm{pred}}
            \left(
            I_i^{\mathrm{global}};
            \theta_{\mathrm{pred}}
            \right)$

        \ENDFOR

        \STATE
        $\mathcal{L}_{\mathcal{B}}
        \gets
        \frac{1}{|\mathcal{B}|}
        \sum_{(\mathcal{G}_i,y_i)\in\mathcal{B}}
        \mathcal{L}_{\mathrm{task}}
        \left(
        \hat{y}_i,y_i
        \right)$

        \STATE

        $\theta
        \gets
        \theta
        -
        \eta\nabla_{\theta}\mathcal{L}_{\mathcal{B}},
        \qquad
        \theta
        =
        \left\{
        \theta_{\mathrm{map}},
        \theta_{\mathrm{affinity}},
        \tau,
        \theta_{\mathrm{pred}}
        \right\}$

    \ENDFOR

\ENDFOR

\STATE \textbf{return}
$\theta_{\mathrm{map}}^\star$,
$\theta_{\mathrm{affinity}}^\star$,
$\tau^\star$,
$\theta_{\mathrm{pred}}^\star$

\end{algorithmic}
\end{algorithm}

\begin{algorithm}[t]
\caption{Unsupervised Discovery of SAIL Modes}
\label{alg:sail_modes}
\begin{algorithmic}

\STATE \textbf{Input:}
Dataset $\mathcal{D}=\{\mathcal{G}_i\}_{i=1}^{N}$;
number of modes $K$;
learning rate $\eta$;
epochs $N_{\mathrm{epochs}}$;
batch size $N_{\mathrm{batch}}$;
regularisation weights
$\lambda_{\mathrm{corr}},\lambda_{\mathrm{sparse}}\geq 0$; and operator sign $s\in\{+1,-1\}$

\STATE \textbf{Parameters:}
$\theta_{\mathrm{map}}$,
$\theta_{\mathrm{affinity}}$,
$\tau>0$

\STATE \textbf{Return:}
optimised
$\theta_{\mathrm{map}}^\star$,
$\theta_{\mathrm{affinity}}^\star$ and $\tau^\star$ for $K$ modes

\STATE
Initialise
$\theta_{\mathrm{map}}$,
$\theta_{\mathrm{affinity}}$,
and $\tau$

\FOR{$e=1,\ldots,N_{\mathrm{epochs}}$}

    \FOR{each mini-batch
    $\mathcal{B}\subset\mathcal{D}$,
    $|\mathcal{B}|=N_{\mathrm{batch}}$}

        \FOR{each $\mathcal{G}_i\in\mathcal{B}$}

            \STATE
            $I_i^{\mathrm{local}},I_i^{\mathrm{global}}
            \gets
            \textsc{SAIL\_Forward}
            \left(
            \mathcal{G}_i;
            \theta,
            \tau
            \right)$
            \STATE
            $\mathbf{Z}_i
            \gets
            f_{\mathrm{map}}
            \left(
            \mathbf{X}_i;
            \theta_{\mathrm{map}}
            \right)$

        \ENDFOR

        \STATE
        $J_r
        \gets
        \frac{s}{|\mathcal{B}|}
        \sum_{\mathcal{G}_i\in\mathcal{B}}
        I_{i,r}^{\mathrm{global}},
        \quad r=1,\ldots,K$

        \STATE
        $z_{\mathcal{B}}^{(r)}
        \gets
        \operatorname{concat}
        \left(
        \left\{
        \mathbf{Z}_{i,:,r}
        \mid
        \mathcal{G}_i\in\mathcal{B}
        \right\}
        \right),
        \quad r=1,\ldots,K$

        \STATE
        $\mathcal{L}_{\mathrm{corr}}
        \gets
        \lambda_{\mathrm{corr}}
        \sum_{\substack{r,s=1\\r\neq s}}^{K}
        \operatorname{Corr}
        \left(
        z_{\mathcal{B}}^{(r)},
        z_{\mathcal{B}}^{(s)}
        \right)^2$

        \STATE
        $\mathcal{L}_{\mathrm{sparse}}
        \gets
        \lambda_{\mathrm{sparse}}
        \left\|
        \theta_{\mathrm{map}}
        \right\|_1$

        \STATE
        $\mathcal{L}_{\mathcal{B}}
        \gets
        -\frac{1}{K}
        \sum_{r=1}^{K}J_r
        +
        \mathcal{L}_{\mathrm{corr}}
        +
        \mathcal{L}_{\mathrm{sparse}}$

        \STATE
        $\left\{
        \theta_{\mathrm{map}},
        \theta_{\mathrm{affinity}},
        \tau
        \right\}
        \gets
        \left\{
        \theta_{\mathrm{map}},
        \theta_{\mathrm{affinity}},
        \tau
        \right\}
        -
        \eta\nabla\mathcal{L}_{\mathcal{B}}$

    \ENDFOR

\ENDFOR

\STATE \textbf{return}
$\theta_{\mathrm{map}}^\star$,
$\theta_{\mathrm{affinity}}^\star$,
$\tau^\star$

\end{algorithmic}
\end{algorithm}

\section{Additional Theoretical Analysis}
\label{sec:supp_Theoretical_Analysis}

We restate the theorems of the main text and provide their proofs.

\begin{restatetheorem}{theorem:classical_reduction}{Classical spatial statistics as exact special cases of SAIL}

Let $\mathbf{A}=[A_{jj'}]\in\mathbb{R}_{\geq 0}^{n\times n}$ be a fixed
row-stochastic spatial weight matrix, i.e., $\mathbf{A}\mathbf{1}=\mathbf{1}$,
and let $f_{\mathrm{act}}=\mathrm{Id}$. 
Then, under appropriate fixed choices of the SAIL mapping,
local-scoring, standardisation, and reduction operators, SAIL (Eq.~\eqref{eq:sail-main}) exactly
recovers:

\begin{enumerate}
    \item[(i)] Moran's $I$ for a non-constant scalar mapped field,
    using population standardisation, $A_{jj}=0$, and mean reduction;

    \item[(ii)] Geary's $C$ for a non-constant scalar mapped field,
    using sample standardisation, $A_{jj}=0$, and mean reduction;

    \item[(iii)] the standardised Getis--Ord $G_j^\star$ statistic
    for a non-constant scalar mapped field, using population
    standardisation and allowing self-weights as in the starred
    statistic, whenever its null variance is non-zero; and

    \item[(iv)] the Morisita--Horn similarity between a non-negative,
    unstandardised multichannel mapped representation at a node and its
    affinity-weighted neighbourhood composition, provided that
    $f_{\mathrm{local}}$ is allowed to act jointly on the corresponding
    mapped channels.
\end{enumerate}

Specifically, the resulting SAIL outputs are respectively
$I_{\mathrm{Moran}}(z;\mathbf{A})$, $C_{\mathrm{Geary}}(z;\mathbf{A})$,
$G_j^\star(z;\mathbf{A})$, and
$C_{\mathrm{MH}}(u_j,v_j)$, where
$u_j=f_{\mathrm{map}}(x_j)\in\mathbb{R}_{\geq 0}^{h}$
is the mapped representation at node $j$ and
$v_j=\sum_{j'}A_{jj'}u_{j'}$ is its
affinity-weighted neighbourhood composition.
\end{restatetheorem}

\begin{proof}

Since $\mathbf{A}$ is row-stochastic, $\sum_{j'}A_{jj'}=1$ for every $j$.
Consequently,
\[
S_0
:=
\sum_{j,j'}A_{jj'}
=
\sum_j 1
=
n.
\]

We consider each statistic in turn.

\medskip
\noindent
\emph{(i) Moran's $I$.}
Let $z_j=f_{\mathrm{map}}(x_j)$ and define the population
mean and variance by
$\bar z=n^{-1}\sum_j z_j$ and
$\sigma_z^2=n^{-1}\sum_j(z_j-\bar z)^2$.
Population standardisation gives
$\tilde z_j=(z_j-\bar z)/\sigma_z$, and therefore
$z_j-\bar z=\sigma_z\tilde z_j$ and
$\sum_j(z_j-\bar z)^2=n\sigma_z^2$.

The classical Moran statistic computed with weights $\mathbf{A}$ is
\[
I_{\mathrm{Moran}}(z;\mathbf{A})
=
\frac{n}{S_0}
\frac{
\sum_{j,j'}A_{jj'}
(z_j-\bar z)(z_{j'}-\bar z)
}{
\sum_j(z_j-\bar z)^2
}.
\]
Substituting $S_0=n$,
$z_j-\bar z=\sigma_z\tilde z_j$, and
$\sum_j(z_j-\bar z)^2=n\sigma_z^2$ yields
\[
\begin{aligned}
I_{\mathrm{Moran}}(z;\mathbf{A})
&=
\frac{n}{n}
\frac{
\sum_{j,j'}A_{jj'}
(\sigma_z\tilde z_j)
(\sigma_z\tilde z_{j'})
}{
n\sigma_z^2
}
\\
&=
\frac{\sigma_z^2}{n\sigma_z^2}
\sum_{j,j'}A_{jj'}\tilde z_j\tilde z_{j'}
\\
&=
\frac{1}{n}
\sum_{j,j'}A_{jj'}\tilde z_j\tilde z_{j'}.
\end{aligned}
\]

Now choose the SAIL local operator
\[
I_j^{\mathrm{local}}
=
\tilde z_j
\sum_{j'}A_{jj'}\tilde z_{j'}
=
\tilde z_j(\mathbf{A}\tilde z)_j
\]
and let $f_{\mathrm{reduce}}$ be the mean. Then
\[
\begin{aligned}
I^{\mathrm{global}}
&=
\frac{1}{n}\sum_j I_j^{\mathrm{local}}
\\
&=
\frac{1}{n}\sum_j
\tilde z_j
\sum_{j'}A_{jj'}\tilde z_{j'}
\\
&=
\frac{1}{n}
\sum_{j,j'}A_{jj'}\tilde z_j\tilde z_{j'}
\\
&=
I_{\mathrm{Moran}}(z;\mathbf{A}).
\end{aligned}
\]
Hence Moran's $I$ is an exact SAIL special case.

\medskip
\noindent
\emph{(ii) Geary's $C$.}
Again let $z_j=f_{\mathrm{map}}(x_j)$, but now define the
sample variance
$s_z^2=(n-1)^{-1}\sum_j(z_j-\bar z)^2$ and standardise as
$\tilde z_j=(z_j-\bar z)/s_z$.
It follows that
$z_j-z_{j'}=s_z(\tilde z_j-\tilde z_{j'})$ and
$\sum_j(z_j-\bar z)^2=(n-1)s_z^2$.

The classical Geary statistic under $\mathbf{A}$ is
\[
C_{\mathrm{Geary}}(z;\mathbf{A})
=
\frac{n-1}{2S_0}
\frac{
\sum_{j,j'}A_{jj'}(z_j-z_{j'})^2
}{
\sum_j(z_j-\bar z)^2
}.
\]
Substituting $S_0=n$,
$z_j-z_{j'}=s_z(\tilde z_j-\tilde z_{j'})$, and
$\sum_j(z_j-\bar z)^2=(n-1)s_z^2$ gives
\[
\begin{aligned}
C_{\mathrm{Geary}}(z;\mathbf{A})
&=
\frac{n-1}{2n}
\frac{
s_z^2
\sum_{j,j'}A_{jj'}
(\tilde z_j-\tilde z_{j'})^2
}{
(n-1)s_z^2
}
\\
&=
\frac{1}{2n}
\sum_{j,j'}A_{jj'}
(\tilde z_j-\tilde z_{j'})^2.
\end{aligned}
\]

Choose the SAIL local operator
\[
I_j^{\mathrm{local}}
=
\frac{1}{2}
\sum_{j'}A_{jj'}
(\tilde z_j-\tilde z_{j'})^2
\]
and again let $f_{\mathrm{reduce}}$ be the mean. Then
\[
\begin{aligned}
I^{\mathrm{global}}
&=
\frac{1}{n}\sum_j I_j^{\mathrm{local}}
\\
&=
\frac{1}{2n}
\sum_j\sum_{j'}
A_{jj'}
(\tilde z_j-\tilde z_{j'})^2
\\
&=
\frac{1}{2n}
\sum_{j,j'}A_{jj'}
(\tilde z_j-\tilde z_{j'})^2
\\
&=
C_{\mathrm{Geary}}(z;\mathbf{A}).
\end{aligned}
\]
Hence Geary's $C$ is also an exact SAIL special case.

\medskip
\noindent
\emph{(iii) Getis--Ord $G_j^\star$.}
Let $\bar z$ and $\sigma_z$ denote the population mean and standard
deviation defined above, and let
$\tilde z_j=(z_j-\bar z)/\sigma_z$.
The standardised Getis--Ord statistic under $\mathbf{A}$ is
\[
G_j^\star(z;\mathbf{A})
=
\frac{
\sum_{j'}A_{jj'}z_{j'}
-
\bar z\sum_{j'}A_{jj'}
}{
\sigma_z
\sqrt{
\frac{
n\sum_{j'}A_{jj'}^2
-
\left(\sum_{j'}A_{jj'}\right)^2
}{
n-1
}
}
}.
\]

Since $\mathbf{A}$ is row-stochastic, $\sum_{j'}A_{jj'}=1$. Substitution gives
\[
G_j^\star(z;\mathbf{A})
=
\frac{
\sum_{j'}A_{jj'}z_{j'}-\bar z
}{
\sigma_z
\sqrt{
\frac{
n\sum_{j'}A_{jj'}^2-1
}{
n-1
}
}
}.
\]
Moreover,
\[
\begin{aligned}
\sum_{j'}A_{jj'}z_{j'}-\bar z
&=
\sum_{j'}A_{jj'}z_{j'}
-
\bar z\sum_{j'}A_{jj'}
\\
&=
\sum_{j'}A_{jj'}(z_{j'}-\bar z)
\\
&=
\sigma_z
\sum_{j'}A_{jj'}\tilde z_{j'}.
\end{aligned}
\]
Substituting this expression and cancelling $\sigma_z$ yields
\[
G_j^\star(z;\mathbf{A})
=
\frac{
\sum_{j'}A_{jj'}\tilde z_{j'}
}{
\sqrt{
\frac{
n\sum_{j'}A_{jj'}^2-1
}{
n-1
}
}
}.
\]

Therefore, choosing
\[
I_j^{\mathrm{local}}
=
\frac{
\sum_{j'}A_{jj'}\tilde z_{j'}
}{
\sqrt{
\frac{
n\sum_{j'}A_{jj'}^2-1
}{
n-1
}
}
}
\]
gives $I_j^{\mathrm{local}}=G_j^\star(z;\mathbf{A})$ exactly. The denominator is
well-defined whenever $n\sum_{j'}A_{jj'}^2>1$. No reduction is required,
since $G_j^\star$ is itself a local statistic.

\medskip
\noindent
\emph{(iv) Morisita--Horn similarity.}
Let $x_j\in\mathbb{R}_{\geq0}^{d}$ be a non-zero
compositional feature vector and choose the (unstandardised) mapping
$
f_{\mathrm{map}}(x_j)
=
u_j
=
x_j/\|x_j\|_1,
$
so that $h=d$ and $\|u_j\|_1=1$.

Let
$
v_j
=
\sum_{j'}A_{jj'}u_{j'}
$
denote the affinity-weighted neighbourhood composition. Since $\mathbf{A}$ is
non-negative and row-stochastic,
\[
\|v_j\|_1
=
\sum_{j'}A_{jj'}\|u_{j'}\|_1
=
\sum_{j'}A_{jj'}
=
1.
\]
Thus both $u_j$ and $v_j$ are relative-abundance
vectors. The Morisita--Horn similarity between them is therefore
\[
C_{\mathrm{MH}}(u_j,v_j)
=
\frac{
2\,u_j^\top v_j
}{
\|u_j\|_2^2+\|v_j\|_2^2
}.
\]
Choosing
$
f_{\mathrm{local}}(u_j,v_j)
=
C_{\mathrm{MH}}(u_j,v_j)
$
gives
\[
I_j^{\mathrm{local}}
=
C_{\mathrm{MH}}
\left(
u_j,
\sum_{j'}A_{jj'}u_{j'}
\right),
\]
which is exactly the Morisita--Horn similarity between node $j$ and its
spatial neighbourhood.

Hence all four statistics are recovered as fixed instances of
the SAIL mapping, local-scoring, activation, and reduction operators.
\end{proof}

\begin{restatetheorem}{theorem:inductive_bias}{Inductive Bias of SAIL for SAC-Driven Learning}

Let $P$ be a distribution over spatial graphs
$G=(V,E,\mathbf{X},\mathbf{C})$, with an optional
associated target variable $Y\in\mathcal{Y}$,
and population risk
\[
R_P(f)=\mathbb{E}_{(G,Y)\sim P}
\left[\mathcal{L}_P(Y,f(G))\right],
\]

and Bayes-optimal predictor $f_P^\star\in\arg\min_f R_P(f)$ with risk $R_P^\star=R_P(f_P^\star)$.

Let $\mathcal{P}_{\mathrm{SAC}}$ denote the class of problems whose Bayes-optimal predictor depends on one or more spatial-autocorrelation (SAC) descriptors of node signals. For clarity, consider a single global descriptor,
$$
f_P^\star(G)=\rho_P\!\left(S_P(z_P,\mathbf{A}_P)\right),
\qquad
z_P=g_P(\mathbf{X})\in\mathbb{R}^n,
$$
where $\mathbf{A}_P$ is a spatial-affinity operator, $S_P$ is a SAC functional formed from local spatial interactions followed by an order-invariant reduction, and $\rho_P$ maps the SAC descriptor to the prediction target. 

Let
$$
\mathcal{F}_{\mathrm{SAIL}}
=
\left\{
G\mapsto
f_{\mathrm{pred}}
\left(
I^{\mathrm{global}}(G;\theta_{\mathrm{SAIL}});
\theta_{\mathrm{pred}}
\right)
\right\}
$$
denote the supervised SAIL hypothesis class, where $I^{\mathrm{global}}$ is obtained from the feature mapping $f_{\mathrm{map}}$, learned affinity $\mathbf{A}$, local SAC operator $f_{\mathrm{local}}$, optional activation $f_{\mathrm{act}}$, and reduction $f_{\mathrm{reduce}}$ of Eq.~\eqref{eq:sail-main}.

Assume that the admissible SAIL component families are sufficiently expressive to approximate $g_P$, $A_P$, and $\rho_P$, and that the chosen local and reduction operators implement $S_P$. Assume also that $0\leq\mathcal{L}_P\leq M$ and that $\mathcal{L}_P$ is $L$-Lipschitz in its prediction argument. Define
$$
\Delta_P(\mathcal{F})
=
\inf_{f\in\mathcal{F}}R_P(f)-R_P^\star,
$$
and let $\mathfrak{R}_m(\mathcal{F})$ denote the Rademacher complexity of a hypothesis class $\mathcal{F}$ with respect to $m$ i.i.d.\ samples from $P$.

Then:

\textnormal{(i) Sufficiency.}
For every $P\in\mathcal{P}_{\mathrm{SAC}}$ and every $\varepsilon>0$, there exists $f_\varepsilon\in\mathcal{F}_{\mathrm{SAIL}}$ such that
$$
\mathbb{E}_G\!\left[
|f_\varepsilon(G)-f_P^\star(G)|
\right]\leq\varepsilon,
$$
and therefore
$$
\Delta_P(\mathcal{F}_{\mathrm{SAIL}})
\leq L\varepsilon.
$$
If $\widehat f_{\mathrm{SAIL}}$ is an $\xi_m$-approximate empirical-risk minimiser from $m$ i.i.d.\ examples, then with probability at least $1-\delta$,
$$
R_P(\widehat f_{\mathrm{SAIL}})-R_P^\star
\leq
L\varepsilon+\xi_m
+
4L\mathfrak{R}_m(\mathcal{F}_{\mathrm{SAIL}})
+
2M\sqrt{\frac{\log(2/\delta)}{2m}}.
$$

\textnormal{(ii) Necessity of higher-order interaction for quadratic SAC.}
Consider the quadratic SAC subclass
$$
S(z,\mathbf{A})=z^\top Q(\mathbf{A})z,
$$
with $\mathbf{A}$ fixed with respect to $z$ and
$$
Q_s(\mathbf{A})=\frac{Q(\mathbf{A})+Q(\mathbf{A})^\top}{2}\neq0.
$$
Then
$$
\nabla_z^2 S(z,\mathbf{A})=2Q_s(\mathbf{A})\neq0.
$$
Let
$$
\mathcal{F}_{\mathrm{lin},T}
=
\left\{
f(z,\mathbf{A})=w^\top H_T(\mathbf{A})z+c:
H_T(\mathbf{A})=\sum_{t=0}^{T}\beta_t\mathbf{A}^t
\right\}
$$
denote the class of finite $T$-hop linear graph filters followed by linear pooling and readout. This includes, for example, Simple Graph Convolution (SGC), linearised Graph Convolution Networks (GCNs), polynomial graph filters, and linear message-passing models with linear pooling/readout. Every $f\in\mathcal{F}_{\mathrm{lin},T}$ is affine in $z$ and hence satisfies $\nabla_z^2f=0$. Thus $\mathcal{F}_{\mathrm{lin},T}$ cannot in general represent the quadratic SAC subclass, whereas SAIL explicitly contains the interaction
$$
z_j(\mathbf{A}z)_j=\sum_{j'}A_{jj'}z_j z_{j'}.
$$
\textnormal{(iii) Comparison with broader classes.}
More general non-linear graph neural networks (GNNs) may also represent the higher-order
interactions required by SAC, but only as one possibility within a
broader function class. Let $\mathcal{F}_{\mathrm{gen}}$ denote such a
broader SAC-capable hypothesis class satisfying
$$
\mathcal{F}_{\mathrm{SAIL}}
\subseteq
\mathcal{F}_{\mathrm{gen}}.
$$
Then
$$
\mathfrak{R}_m(\mathcal{F}_{\mathrm{SAIL}})
\leq
\mathfrak{R}_m(\mathcal{F}_{\mathrm{gen}}).
$$
Hence, once the SAC structure is already representable by SAIL,
additional generic expressivity cannot, by itself, yield a smaller
Rademacher-complexity term in the corresponding standard
generalisation bound. If the two classes attain the same approximation
and optimisation errors on a SAC-driven problem, the resulting
excess-risk bound for SAIL is no larger.

Thus SAIL provides a matched inductive bias for $\mathcal{P}_{\mathrm{SAC}}$: it contains the interactions required by SAC-driven prediction, simpler linear-interaction classes may fail to represent them, while broader enclosing classes provide no smaller complexity-dependent learning guarantee solely through their additional expressivity.

\end{restatetheorem}

\begin{proof}
\textit{Part (i).}
Consider a problem $P\in\mathcal{P}_{\mathrm{SAC}}$ with
$$
f_P^\star(G)=\rho_P\!\left(S_P(z_P,\mathbf{A}_P)\right).
$$
The argument extends componentwise to multiple SAC descriptors and directly to local SAC outputs.

The defining computation has the structure
$$
\mathbf{X}
\longrightarrow
z_P
\longrightarrow
S_P(z_P,\mathbf{A}_P)
\longrightarrow
\rho_P\!\left(S_P(z_P,\mathbf{A}_P)\right).
$$
SAIL has the same structure: $f_{\mathrm{map}}$ constructs scalar node fields, the affinity mechanism constructs $\mathbf{A}$, $f_{\mathrm{local}}$ and $f_{\mathrm{reduce}}$ compute SAC descriptors, and $f_{\mathrm{pred}}$ maps these descriptors to the task output. Thus
$$
g_P\leftrightarrow f_{\mathrm{map}},
\qquad
\mathbf{A}_P\leftrightarrow \mathbf{A},
\qquad
$$
$$
S_P\leftrightarrow
f_{\mathrm{reduce}}\circ f_{\mathrm{act}}\circ f_{\mathrm{local}},
\qquad
\rho_P\leftrightarrow f_{\mathrm{pred}}.
$$
Under exact component representation, an appropriate SAIL parametrisation therefore satisfies $f_\theta^{\mathrm{SAIL}}=f_P^\star$. More generally, by the assumed expressivity and continuity of the composition, for every $\varepsilon>0$ there exists $f_\varepsilon\in\mathcal{F}_{\mathrm{SAIL}}$ such that
$$
\mathbb{E}_G
\left[
|f_\varepsilon(G)-f_P^\star(G)|
\right]
\leq\varepsilon.
$$
Since $\mathcal{L}_P$ is $L$-Lipschitz,
$$
R_P(f_\varepsilon)-R_P^\star
\leq
L\,
\mathbb{E}_G
\left[
|f_\varepsilon(G)-f_P^\star(G)|
\right]
\leq L\varepsilon,
$$
so $\Delta_P(\mathcal{F}_{\mathrm{SAIL}})\leq L\varepsilon$.

Let
$$
\widehat R_m(f)
=
\frac{1}{m}
\sum_{r=1}^{m}\mathcal{L}_P(Y_r,f(G_r)),
$$
and suppose
$$
\widehat R_m(\widehat f_{\mathrm{SAIL}})
\leq
\inf_{f\in\mathcal{F}_{\mathrm{SAIL}}}\widehat R_m(f)
+\xi_m.
$$
Standard Rademacher uniform convergence for a bounded $L$-Lipschitz loss gives, with probability at least $1-\delta$,
$$
\begin{aligned}
R_P(\widehat f_{\mathrm{SAIL}})
-\inf_{f\in\mathcal F_{\mathrm{SAIL}}}R_P(f)
&\leq \xi_m
+4L\mathfrak R_m(\mathcal F_{\mathrm{SAIL}}) \\
&\quad
+2M\sqrt{\frac{\log(2/\delta)}{2m}}.
\end{aligned}
$$
Combining this with the approximation bound proves Part~(i).

\medskip

\textit{Part (ii).}
For
$$
S(z,\mathbf{A})=z^\top Q(\mathbf{A})z,
$$
the skew-symmetric part of $Q(\mathbf{A})$ contributes zero, so
$$
S(z,\mathbf{A})=z^\top Q_s(\mathbf{A})z,
\qquad
Q_s(\mathbf{A})=\frac{Q(\mathbf{A})+Q(\mathbf{A})^\top}{2}.
$$
Hence
$$
\nabla_z^2S(z,\mathbf{A})=2Q_s(\mathbf{A})\neq0,
$$
showing genuine second-order dependence on $z$.

For any $f\in\mathcal{F}_{\mathrm{lin},T}$,
$$
f(z,\mathbf{A})
=
w^\top
\left(
\sum_{t=0}^{T}\beta_t\mathbf{A}^t
\right)z+c
=
b(\mathbf{A})^\top z+c,
$$
and therefore
$$
\nabla_z^2f(z,\mathbf{A})=0.
$$
No such predictor can equal a non-degenerate quadratic SAC target on an open subset of the signal domain. Under squared-error loss, if the distribution of $z$ has finite fourth moment and support containing a non-empty open set, this implies
$$
\Delta_P(\mathcal{F}_{\mathrm{lin},T})>0
$$
for some quadratic SAC-driven problems.

Increasing $T$ changes the spatial range of the graph filter but not its order of dependence on $z$. SAIL, by contrast, directly evaluates
$$
z_j(\mathbf{A}z)_j
=
\sum_{j'}A_{jj'}z_j z_{j'},
$$
and global reduction yields the corresponding quadratic SAC statistic. Thus the required second-order interaction is explicit in SAIL. This proves Part~(ii).

\medskip

\textit{Part (iii).}
Let
$$
\mathcal{F}_{\mathrm{SAIL}}
\subseteq
\mathcal{F}_{\mathrm{gen}}.
$$
For a fixed sample $G_1,\ldots,G_m$ and Rademacher variables $\sigma_1,\ldots,\sigma_m$,
$$
\sup_{f\in\mathcal{F}_{\mathrm{SAIL}}}
\frac{1}{m}\sum_{r=1}^{m}\sigma_r f(G_r)
\leq
\sup_{f\in\mathcal{F}_{\mathrm{gen}}}
\frac{1}{m}\sum_{r=1}^{m}\sigma_r f(G_r).
$$
Taking expectations gives
$$
\mathfrak{R}_m(\mathcal{F}_{\mathrm{SAIL}})
\leq
\mathfrak{R}_m(\mathcal{F}_{\mathrm{gen}}).
$$
Therefore, for
$$
\Gamma_m(\mathcal{F},\delta)
=
4L\mathfrak{R}_m(\mathcal{F})
+
2M\sqrt{\frac{\log(2/\delta)}{2m}},
$$
we have
$$
\Gamma_m(\mathcal{F}_{\mathrm{SAIL}},\delta)
\leq
\Gamma_m(\mathcal{F}_{\mathrm{gen}},\delta).
$$

Because $\mathcal{F}_{\mathrm{gen}}$ is larger, it may attain smaller approximation error, so this does not imply worse realised test performance for the broader class. It establishes that once the SAC target is already adequately represented by SAIL, enlarging the hypothesis class alone cannot improve the Rademacher-complexity contribution to the standard generalisation bound. If approximation and optimisation errors are equal, the resulting excess-risk bound for SAIL is therefore no larger. This proves Part~(iii).
\end{proof}

\section{Additional implementation details on learning spatial autocorrelation beyond message passing}
\label{sec:supp_synth_protocol}

This section provides complete implementation details for the synthetic
Moran's $I$ regression experiment of Section~\ref{sec:beyond_mp}.
We adopt a two-stage approach where each architecture is assessed at its own best configuration under an identical optimisation budget. All models are trained with Adam, learning rate $10^{-2}$, mean-squared-error loss, and early stopping on validation Pearson correlation with patience $10$ and a maximum of $30$ epochs. The data splits are $600/250/600$ graphs for training/validation/test. Stage~1
selects hyperparameters on a single sweep seed; Stage~2 estimates the
performance distribution by retraining each selected model across $1000$
independent seeds.

\subsection{Dataset Construction}
\label{sec:supp_synth_dataset}

Each sample is an independently generated spatial graph. The generative process
is designed so that the regression target is, by construction, a global spatial
autocorrelation functional of a \emph{hidden} linear projection of the node features evaluated on a \emph{hidden} neighbourhood, and so that neither the projection nor the spatial scale can be memorised across graphs.

\paragraph{Coordinates} For each graph we sample the number of nodes uniformly,
$n \sim \mathcal{U}\{50,\dots,100\}$, and draw node coordinates $c_j \sim \mathcal{U}([-1,1]^2)$ independently.

\paragraph{Spatially correlated node features} Node features $\mathbf{X}\in\mathbb{R}^{n\times d}$ with $d=8$ are drawn from a Gaussian
process over the coordinates. We form an RBF kernel $K^{\mathrm{RBF}}_{jj'} = a^2 \exp\!\big(-\lVert c_j-c_{j'}\rVert_2^2 /(2\ell^2)\big)$
with amplitude $a=1$, add jitter $10^{-5}\mathbf{I}$ for numerical stability, and take its Cholesky factor $C_L$. Eight independent Gaussian process draws are generated as $\mathbf{X} = C_L\,\boldsymbol{\varepsilon}$ with
$\boldsymbol{\varepsilon}\in\mathbb{R}^{n\times 8}$, $\varepsilon_{jk}\sim\mathcal{N}(0,1)$.
The per-graph lengthscale $\ell$ is sampled log-uniformly on $[0.15,\,1.2]$, so that spatial smoothness varies across graphs. To introduce inter-feature correlation, the channels are mixed by a random low-rank matrix $\mathbf{M}=\mathbf{U}\mathbf{V}^\top\in\mathbb{R}^{d\times d}$ of rank $r=3$ (with $\mathbf{U},\mathbf{V}\in\mathbb{R}^{d\times r}$, standard normal entries), giving
$\mathbf{X}\leftarrow \mathbf{X}\mathbf{M}$. Finally each feature channel is standardised to zero mean and unit variance across nodes.

\paragraph{Hidden projection and regression target} 
For each graph a hidden projection $w^\star\in\mathbb{R}^d$ is drawn uniformly on the unit sphere ($w^\star=u/\lVert u\rVert$,
$u\sim\mathcal{N}(0,\mathbf{I}_d)$) and defines the latent scalar field $z^\star=\mathbf{X}w^\star$, standardised to zero mean and unit
variance. The regression label is:
\begin{equation}
\begin{aligned}
y \;=\; I_{\mathrm{Moran}}(z^\star)
&\;=\;
\frac{n}{S_0}\,
\frac{\sum_{j\neq j'} w_{jj'}\,\tilde z^\star_j \tilde z^\star_{j'}}
{\sum_j (\tilde z^\star_j)^2}, \\
S_0 &= \sum_{j\neq j'} w_{jj'},
\end{aligned}
\end{equation}
where the weights $w_{jj'}$ are defined on the \emph{true} neighbourhood: a
symmetric $k$-nearest-neighbour graph with $k_{\mathrm{true}}=8$, and edges
weighted by an RBF of the coordinate distance,
$w_{jj'}=\exp(-D_{jj'}^2/(2\ell_w^2))$ with $\ell_w=0.6$. Because
$w^\star$, $\ell$, and the induced neighbourhood vary per graph, the
mapping from $(\mathbf{X},c)$ to $y$ cannot be solved by memorising a
fixed direction or spatial scale; it requires modelling the spatial
autocorrelation functional itself. Crucially, the generating neighbourhood
$k_{\mathrm{true}}$ is never exposed to any model.

\subsection{Neighbourhood Condition}
\label{sec:supp_synth_neighbourhoods}
 
Because no model is handed the generating operator, all models must infer
spatial structure from a neighbourhood they construct themselves. We adopt an
access-matched, fully connected, distance-weighted comparison: every node pair
is connected, with edge weights $w=\exp(-D_{jj'}^2/(2\cdot 0.6^2))$, or raw pairwise
distances supplied as edge attributes for attention variants. The baselines see
all pairs and must learn the weighting, while SAIL likewise sees all pairwise
distances and must learn the spatial temperature $\tau$. 

\subsection{Models and Architectural Choices}
\label{sec:supp_synth_models}
 
All models share a two-stage template: a node-level encoder followed by global
mean pooling and a single linear head producing the scalar prediction. They
differ only in how the encoder uses graph structure. Each baseline uses a single
aggregation layer, matching SAIL's single quadratic operator. The input
dimension is $d_{\mathrm{in}}=8$ for feature-only variants and $d_{\mathrm{in}}=10$ when node coordinates are concatenated; the hidden width $d_{\mathrm{hid}}$ (and, for attention, the number of heads) is selected per model in Stage~1.
 
\begin{itemize}
\item \textbf{SAIL (dense).} A linear feature map
$f_{\mathrm{map}}:\mathbb{R}^{d_{\mathrm{in}}}\!\to\mathbb{R}^{h}$ followed by the
Moran-style local index under a dense affinity built from full pairwise
coordinate distances (no edge list, no self-loops, learnable temperature $\tau$,
mean global reduction), then a linear head $\mathrm{Linear}(h\!\to\!1)$.
\item \textbf{GCN.} A single $\mathrm{GCNConv}(d_{\mathrm{in}}\!\to\!d_{\mathrm{hid}})$ with ReLU,
global mean pooling, and $\mathrm{Linear}(d_{\mathrm{hid}}\!\to\!1)$, run on the fully connected
distance-weighted graph, evaluated with and without node coordinates.
\item \textbf{GAT.} A single attention layer $(d_{\mathrm{in}}\!\to\!d_{\mathrm{hid}})$ with ReLU,
global mean pooling, and $\mathrm{Linear}(d_{\mathrm{hid}}\!\to\!1)$. Run on the
fully connected graph, evaluated with and without coordinates.
\item \textbf{MLP (control)} Single-hidden-layer MLP
$(d_{\mathrm{in}}\!\to\!d_{\mathrm{hid}})$ with ReLU, global mean pooling, and
$\mathrm{Linear}(d_{\mathrm{hid}}\!\to\!1)$.
\end{itemize}
 
\subsection{Two-Stage Selection and Evaluation}
\label{sec:supp_synth_stages}
 
\paragraph{Stage 1: hyperparameter sweep} For each model we grid-search the
hidden width over $\{8,16,32,64\}$ and, for attention models, the number of
heads over $\{1,2,4\}$, on a single sweep seed. Each configuration is trained to
convergence and its validation Pearson correlation recorded. The configuration
with the highest validation Pearson is selected; ties within $10^{-2}$ are
broken toward the smaller model.
 
\paragraph{Stage 2: seeded evaluation} Each model is then retrained at its
selected configuration across $1000$ independent seeds, with datasets
regenerated per seed, and test metrics (Pearson, MSE) recorded. The
evaluation seeds are drawn disjoint from the sweep seed. The resulting distributions define the box plots of Figure~\ref{baseline_regression}, with per-model parameter counts and mean $\pm$ standard deviation of test Pearson and MSE reported in Supplementary Table~\ref{tab:supp_regression_results}; significance is assessed with a paired Wilcoxon signed-rank test with false-discovery-rate correction across models.

\section{Additional details on unsupervised discovery of Spatially Coherent Structures in spatial omics}
\label{sec:supp_4i}

We apply SAIL in a fully unsupervised setting to the 4i multiplexed imaging
dataset where subcellular compartment labels are used only for post-hoc evaluation and are never exposed to the model. A $k$-nearest-neighbour spatial graph is built directly from measurement-location coordinates with $k_{\mathrm{NN}}=32$ neighbours, and neighbour distances are
log-transformed before being passed to the affinity module. SAIL is trained to jointly discover $20$ modes, using a Geary-style local index, a Stiefel-constrained linear projection, no learned feature affinity, and $\lambda_{\mathrm{corr}}=\lambda_{\mathrm{sparse}}=1$. For each learned mode, LISA cluster maps are computed (Section~\ref{hotspot_discovery}), and the
High--High and Low--Low hotspot masks are compared against each of the $10$ expert-annotated compartments via Dice overlap. SAIL attains the highest Dice on $7$ of $10$ compartments; across compartments it outperforms Univariate ($p=0.003$), Mean ($p=0.003$), MULTISPATI-PCA ($p=0.014$) and Hotspot ($p=0.010$; one-sided paired Wilcoxon signed-rank tests). For every compartment, we report the single best-matching mode and cluster polarity
(Supplementary Table~\ref{tab:supp_4i_dice}), and we additionally benchmark against classical univariate and mean-aggregate multivariate Moran's $I$ baselines, MULTISPATI-PCA and Hotspot (Supplementary Table~\ref{tab:supp_4i_full_comparison}). The projection weights and LISA cluster maps of all $20$ learned modes are shown in Supplementary Figure~\ref{fig:all_sail_modes}, and the expression of all $43$ proteins together with the ground-truth compartment masks is shown in Supplementary Figure~\ref{fig:all_4i_expression}.

\begin{table}[!ht]
\centering
\caption{Best-matching unsupervised SAIL Mode per annotated subcellular compartment on the 4i dataset. For each compartment, we report the mode whose LISA hotspot mask (High--High or Low--Low) attains the highest Dice overlap with the expert annotation.}
\label{tab:supp_4i_dice}
\begin{tabular}{lccc}
\toprule
Compartment & Mode & LISA cluster & Dice \\
\midrule
Nucleolus            & $8$ & Low--Low & $0.839$ \\
Cell periphery 2     & $1$ & Low--Low & $0.587$ \\
Endosomes/Golgi 1    & $3$ & High--High & $0.577$ \\
Cell periphery 1     & $4$ & Low--Low & $0.533$ \\
Nucleus              & $5$ & High--High & $0.529$ \\
ER/mitochondria 1    & $14$ & High--High & $0.488$ \\
ER/mitochondria 2    & $9$ & Low--Low & $0.347$ \\
Cell periphery 3     & $17$ & High--High & $0.316$ \\
Nuclear envelope     & $16$ & High--High & $0.126$ \\
Endosomes/Golgi 2    & $11$ & High--High & $0.077$ \\
\bottomrule
\end{tabular}
\end{table}

\begin{table}[!ht]
\centering
\newcommand{\uv}[2]{\begin{tabular}[c]{@{}c@{}}#1\\[-2pt] {\footnotesize (#2)}\end{tabular}}
\caption{Full per-compartment Dice comparison on the 4i dataset, extending the
four-compartment summary of the main text to all $10$ annotated compartments.
\emph{SAIL} is the best-matching unsupervised mode (Supplementary Table~\ref{tab:supp_4i_dice});
\emph{Univariate} (Univ.) is classical Moran's~$I$ with LISA computed on the single
best-performing individual protein (named in parentheses) among the $20$ proteins with the highest global Moran's~$I$; \emph{Mean} is the traditional multivariate extension of Moran's~$I$, i.e.\ the mean of the univariate Moran's~$I$ statistics across all $43$ proteins; \emph{MULTISPATI} is
MULTISPATI-PCA~\cite{S:dray2008spatial} with local Moran's~$I$ LISA computed on the axes and \emph{Hotspot} is Hotspot~\cite{S:detomaso2021hotspot} with local Moran's~$I$ LISA computed on the scores of its protein modules, all using the same $k_{\mathrm{NN}}=32$ log-distance kNN graph. For each compartment we report the Dice overlap between the derived High--High/Low--Low hotspot mask and the expert annotation.}
\label{tab:supp_4i_full_comparison}
\setlength{\tabcolsep}{3pt}
\footnotesize
\begin{tabular}{lccccc}
\toprule
Compartment & SAIL & Univ. & Mean & MULTISPATI & Hotspot \\
\midrule
Nucleolus & $\mathbf{0.84}$ & \uv{$0.75$}{DAPI} & $0.50$ & $0.79$ & $0.71$ \\
Cell periphery 2 & $\mathbf{0.59}$ & \uv{$0.49$}{CRT} & $0.50$ & $0.49$ & $0.50$ \\
Endosomes/Golgi 1 & $\mathbf{0.58}$ & \uv{$0.38$}{SARA} & $0.34$ & $0.38$ & $0.56$ \\
Cell periphery 1 & $\mathbf{0.53}$ & \uv{$0.51$}{CRT} & $0.51$ & $0.50$ & $0.50$ \\
Nucleus & $\mathbf{0.53}$ & \uv{$0.47$}{PCNA} & $0.41$ & $0.40$ & $0.46$ \\
ER/mitochondria 1 & $0.49$ & \uv{$0.45$}{ABCD3} & $0.24$ & $\mathbf{0.50}$ & $0.47$ \\
ER/mitochondria 2 & $0.35$ & \uv{$0.37$}{Actin} & $\mathbf{0.38}$ & $0.38$ & $0.37$ \\
Cell periphery 3 & $0.32$ & \uv{$0.31$}{CTNNB1} & $0.10$ & $0.31$ & $\mathbf{0.33}$ \\
Nuclear envelope & $\mathbf{0.13}$ & \uv{$0.06$}{NUPS} & $0.01$ & $0.08$ & $0.07$ \\
Endosomes/Golgi 2 & $\mathbf{0.08}$ & \uv{$0.03$}{FBL} & $0.01$ & $0.04$ & $0.02$ \\
\bottomrule
\end{tabular}
\end{table}
\begin{figure*}[!t]
  \centering
   \includegraphics[width=5in]{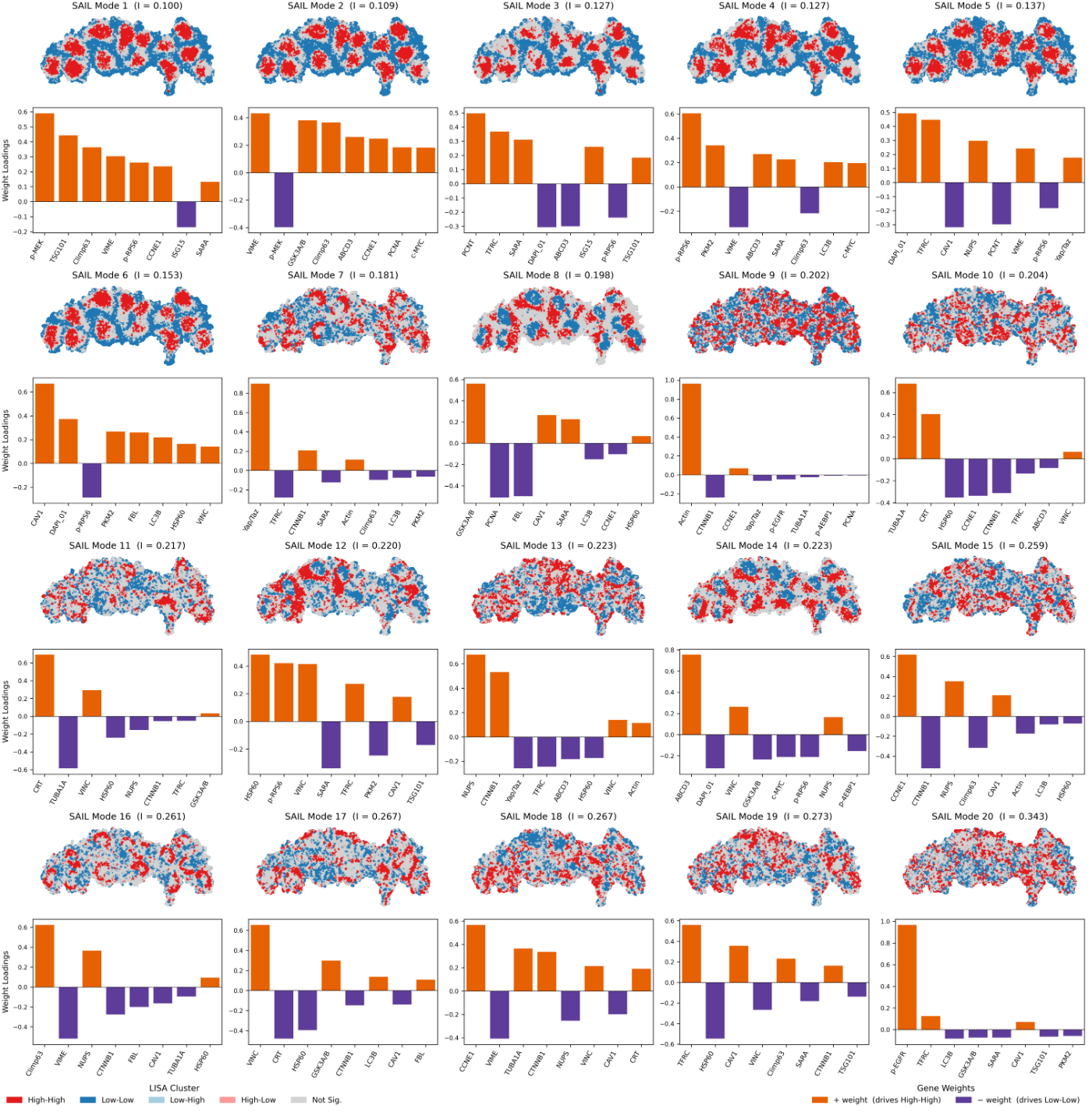}
   \caption{All $20$ unsupervised SAIL Modes learned on the 4i dataset. For each
   mode we show the learned $f_{\mathrm{map}}$ projection weights over the $43$
   protein markers (bar chart, as in Figure~\ref{4i_figure}), alongside the
   corresponding High--High and Low--Low LISA cluster maps overlaid on the
   tissue. Together these define, for every mode, both the molecular program it
   emphasises and the spatial regions where that program is coherently high
   (High--High) or low (Low--Low). The best-matching mode per annotated
   compartment is quantified in Supplementary Table~\ref{tab:supp_4i_dice}.}
   \label{fig:all_sail_modes}
\end{figure*}

\begin{figure*}[!t]
  \centering
   \includegraphics[width=6in]{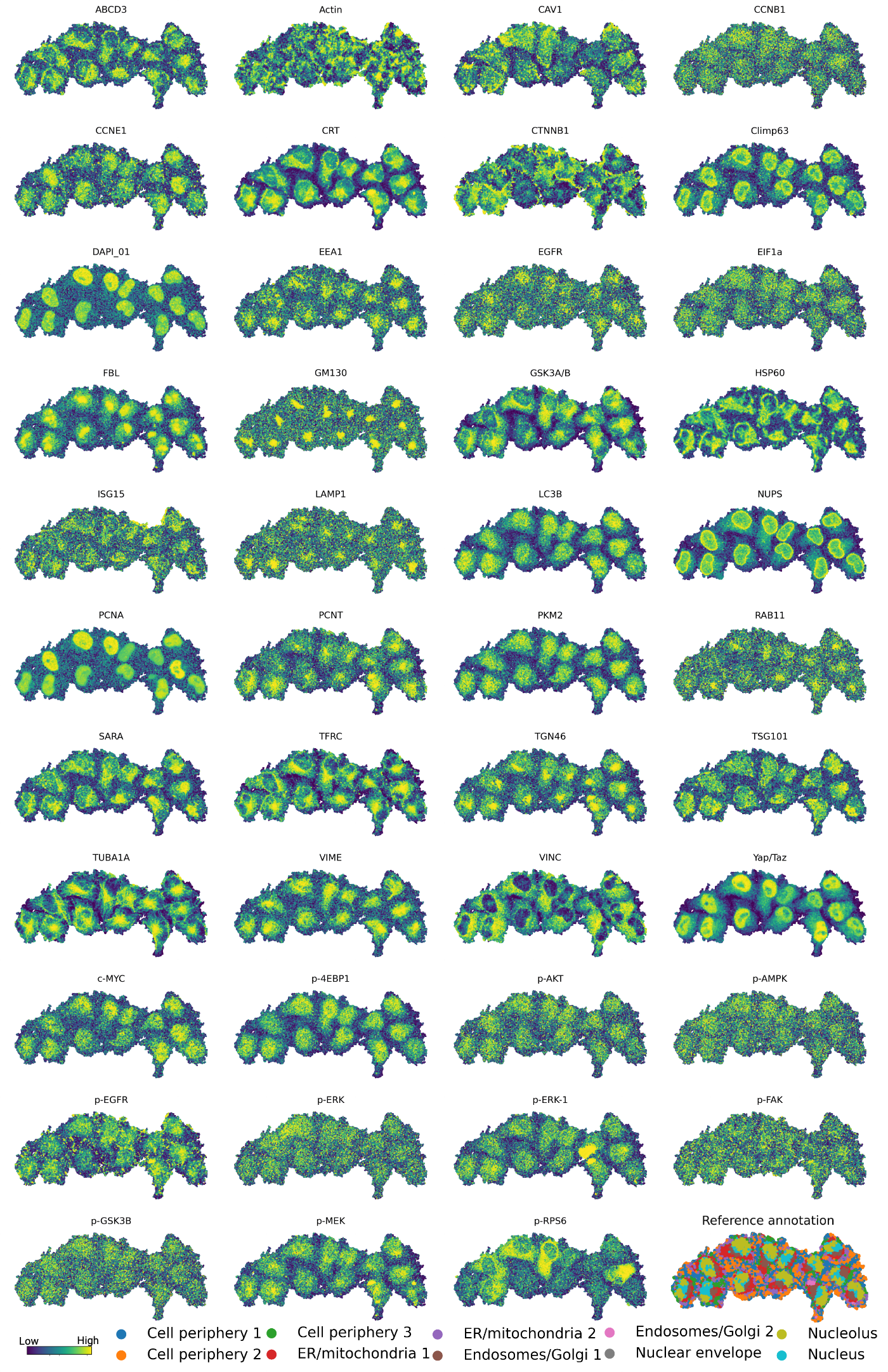}
      \caption{All $43$ proteins in the 4i dataset. For each protein we show its localised expression, to help understand its expected spatial distribution and how it contributes to SAIL Modes that overlap with known subcellular compartments. We also show the ground-truth annotation masks for all compartments.}
   \label{fig:all_4i_expression}
\end{figure*}

\section{Additional details on Discovering spatial coherence biomarkers of treatment response}
\label{sec:sail_niches}

We next test whether unsupervised SAIL niches are predictive of clinical
response to PD-1/PD-L1 blockade in a cutaneous squamous cell carcinoma (cSCC)
CosMx cohort, comprising baseline (pre-treatment) biopsies from $11$ responder and $10$ non-responder patients. Response is treated as an
unpaired, subject-level label, so association is assessed with a two-sided Mann--Whitney $U$ test and the corresponding ROC AUC per mode. Cells with fewer than $100$ total counts and genes detected in fewer than $10$ cells are removed prior to analysis, and any tissue core with fewer than $50$ surviving cells is dropped. Each patient is represented by a single graph in which every retained core forms an internally connected
$k$-nearest-neighbour block ($k_{\mathrm{NN}}=32$) with no edges between cores.

SAIL is trained fully unsupervised across all patients jointly to learn $20$ spatial modes with: a Geary-style local index, a Stiefel-constrained linear projection, $\lambda_{\mathrm{corr}}=\lambda_{\mathrm{sparse}}=1$, a learnable spatial temperature, and no learned feature affinity. The resulting $20$ modes are scored per subject (Supplementary Table~\ref{tab:sail_pd1_response}) and their defining gene weights are shown in Supplementary Figure~\ref{fig:all_sail_modes_pdl1}.

\begin{table}[htbp]
\centering
\caption{SAIL niche discrimination of PD-1/PD-L1 blockade response (see Supplementary Section~\ref{sec:sail_niches}). Each row is one SAIL Mode, scored at the subject level across 11 responder and 10 non-responder patients and ranked by discriminative AUC. \emph{Med.\ $R$} and \emph{Med.\ $NR$} are the median per-subject global Geary-style SAIL scores among responder ($R$) and non-responder ($NR$) patients, respectively. \emph{AUC} is the area under the ROC curve for separating the two groups by SAIL score; $U$ and $p$ are the Mann--Whitney--Wilcoxon rank-sum statistic and its two-sided $p$-value.}
\label{tab:sail_pd1_response}
\small
\setlength{\tabcolsep}{4pt}
\begin{tabular}{rrrrrr}
\toprule
Mode & Med.\ $R$ & Med.\ $NR$ & AUC & $U$ & $p$ \\
\midrule
18 & 0.692 & 0.556 & 0.864 & 95 & 0.005 \\
11 & 0.589 & 0.554 & 0.809 & 89 & 0.018 \\
4 & 0.588 & 0.513 & 0.791 & 87 & 0.027 \\
8 & 0.634 & 0.541 & 0.782 & 86 & 0.032 \\
14 & 0.638 & 0.549 & 0.773 & 85 & 0.038 \\
7 & 0.611 & 0.519 & 0.764 & 84 & 0.045 \\
20 & 0.683 & 0.614 & 0.764 & 84 & 0.045 \\
9 & 0.635 & 0.554 & 0.755 & 83 & 0.053 \\
16 & 0.673 & 0.563 & 0.745 & 82 & 0.062 \\
19 & 0.695 & 0.577 & 0.727 & 80 & 0.084 \\
3 & 0.582 & 0.505 & 0.700 & 77 & 0.130 \\
10 & 0.641 & 0.537 & 0.691 & 76 & 0.149 \\
13 & 0.608 & 0.542 & 0.682 & 75 & 0.170 \\
15 & 0.609 & 0.545 & 0.673 & 74 & 0.193 \\
6 & 0.596 & 0.534 & 0.655 & 72 & 0.245 \\
1 & 0.485 & 0.495 & 0.636 & 70 & 0.307 \\
17 & 0.625 & 0.592 & 0.627 & 69 & 0.342 \\
12 & 0.612 & 0.542 & 0.618 & 68 & 0.379 \\
2 & 0.563 & 0.502 & 0.600 & 66 & 0.460 \\
5 & 0.553 & 0.540 & 0.591 & 65 & 0.504 \\
\bottomrule
\end{tabular}
\end{table}

\begin{figure*}[!t]
  \centering
   \includegraphics[width=7in]{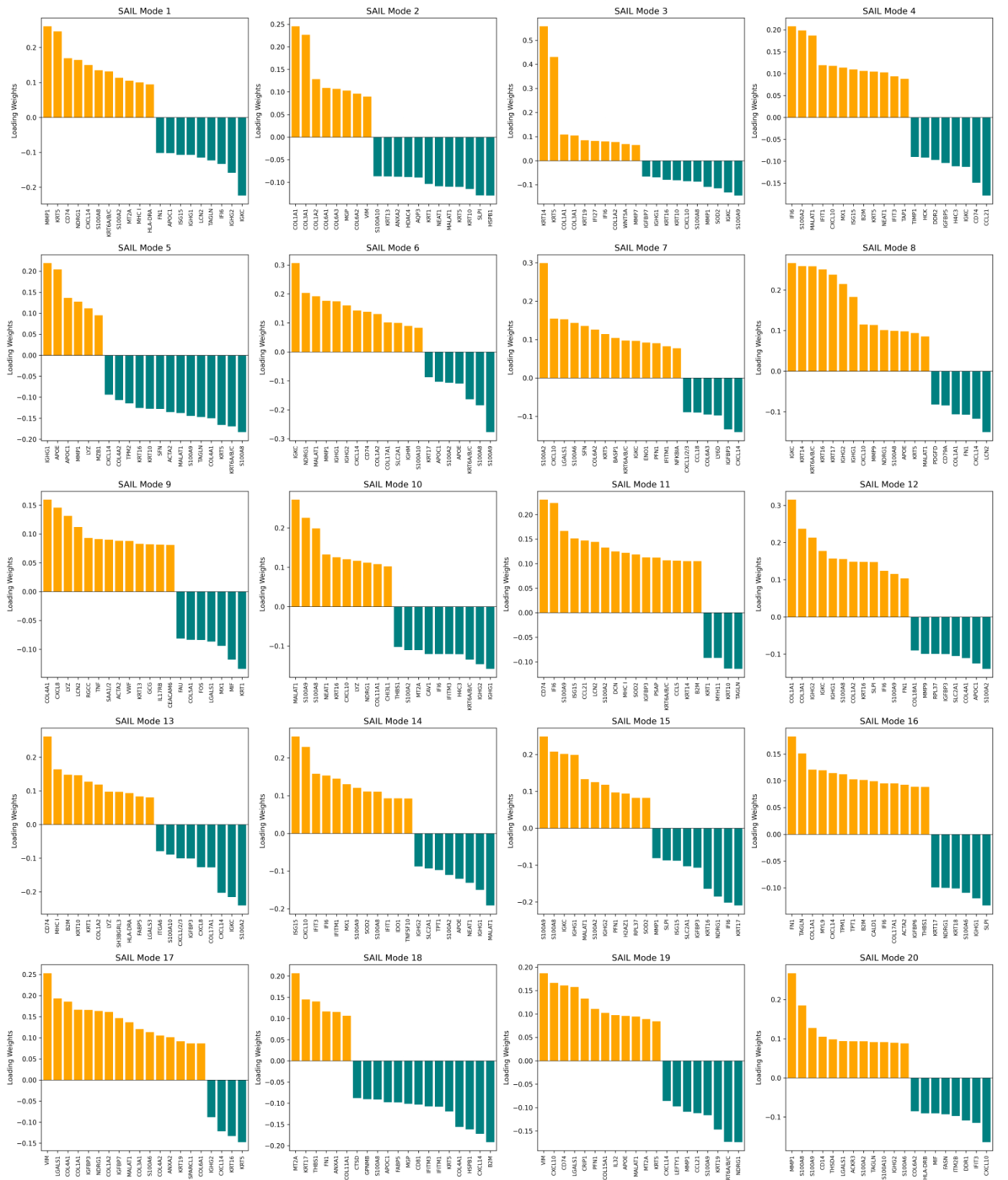}
   \caption{Gene weights defining all $20$ unsupervised SAIL Modes learned on the cSCC PD-1/PD-L1 CosMx cohort. For each mode, bars show the top genes by absolute loading in the learned $f_{\mathrm{map}}$ projection, with positively and negatively weighted genes coloured separately; these weights specify the gene combination whose spatial Geary-style dissimilarity each mode minimises.}
   \label{fig:all_sail_modes_pdl1}
\end{figure*}

\section{Additional details on Unsupervised SAC modes for prediction of patient survival}
\label{sec:supp_prognosis}
Implementation details for the prognostic analyses of Section~\ref{sec:prognosis_sail}. In
all four cohorts, SAIL is applied in the unsupervised mode-learning configuration. Spatial modes are learned by minimising a Geary-style objective ($K=20$ modes, $k_{\mathrm{NN}}=32$, $\lambda_{\mathrm{corr}}=\lambda_{\mathrm{sparse}}=1$, Stiefel-constrained linear projection, no learned feature affinity), unless stated otherwise. For each patient, the resulting global SAIL score(s) are used to evaluate performance.

\subsection{ORION-CRC}
In the ORION-CRC cohort, $k_{\mathrm{NN}}=32$-nearest-neighbour cell graphs are used, with node features given by $17$ marker intensities after excluding the autofluorescence (AF1) and ARG550 channels. SAIL learns $20$ modes jointly across all slides ($K=20$ modes, $k_{\mathrm{NN}}=32$, $\lambda_{\mathrm{corr}}=1$, $\lambda_{\mathrm{sparse}}=0.1$, Stiefel-constrained linear projection, no learned feature affinity). We reduce the sparsity penalty from the default $\lambda_{\mathrm{sparse}}=1$ because, on this cohort, the default produced modes that each loaded on a single marker, so that the global SAIL score reduced to a univariate SAC statistic. Supplementary Table~\ref{tab:supp_orion_modes} reports the survival association of every mode. Survival associations of univariate Geary's $C$ on each individual marker, the mean-aggregate multivariate baseline, tumour grade and the immune groups (IFM1--4) of the original ORION study are reported in Supplementary Table~\ref{tab:orion_baselines}.

\begin{table}[!t]
\centering
\caption{Survival association (overall survival) of all $20$ unsupervised SAIL Modes in ORION-CRC. Patients are split at the median of each mode's global SAIL score. C-Index, log-rank $p$-value and $\Delta$RMST (restricted mean survival time of the low- minus the high-score group, in days) is reported.}
\label{tab:supp_orion_modes}
\begin{tabular}{rccc}
\toprule
Mode & C-Index & $p$ & $\Delta$RMST \\
\midrule
1 & 0.68 & $0.403$ & 213 \\
2 & 0.66 & $0.403$ & 213 \\
3 & 0.52 & $0.300$ & 288 \\
4 & 0.72 & $0.040$ & 502 \\
5 & 0.68 & $0.048$ & 482 \\
6 & 0.51 & $0.502$ & -153 \\
7 & 0.57 & $0.527$ & 196 \\
8 & 0.64 & $0.360$ & 201 \\
9 & 0.55 & $0.442$ & 342 \\
10 & 0.57 & $0.193$ & 417 \\
11 & 0.60 & $0.244$ & 349 \\
12 & 0.61 & $0.046$ & 602 \\
13 & 0.63 & $0.261$ & 323 \\
14 & 0.57 & $0.447$ & 344 \\
15 & 0.65 & $0.007$ & 660 \\
16 & 0.58 & $0.275$ & 321 \\
17 & 0.74 & $0.005$ & 782 \\
18 & 0.59 & $0.331$ & 345 \\
19 & 0.60 & $0.138$ & 407 \\
20 & 0.53 & $0.458$ & -219 \\
\bottomrule
\end{tabular}
\end{table}

\subsection{TCGA-BRCA and TCGA-CRC}
In the TCGA analysis, patches form graph nodes connected via $k_{\mathrm{NN}}=32$ nearest-neighbour edges. For both cohorts, SAIL learns $20$ modes ($\lambda_{\mathrm{corr}}=\lambda_{\mathrm{sparse}}=1$, Stiefel-constrained linear projection, no learned feature affinity). Supplementary Table~\ref{tab:supp_tcga_modes} reports the survival association of every mode.

\begin{table*}[!t]
\centering
\small
\caption{Survival association of all $20$ unsupervised SAIL Modes in TCGA-BRCA (disease-specific survival) and TCGA-CRC (overall survival). Patients are split at the median of each mode's global SAIL score. C-Index is reported as $\max(c,1-c)$; $p$ is the log-rank $p$-value; HR (95\% CI) compares the high- with the low-score group; $\Delta$RMST is the restricted mean survival time of the low- minus the high-score group, in days.}
\label{tab:supp_tcga_modes}
\setlength{\tabcolsep}{3pt}
\begin{tabular}{r cccc cccc}
\toprule
& \multicolumn{4}{c}{TCGA-BRCA} & \multicolumn{4}{c}{TCGA-CRC} \\
\cmidrule(lr){2-5}\cmidrule(lr){6-9}
Mode & C-Index & $p$ & HR (95\% CI) & $\Delta$RMST & C-Index & $p$ & HR (95\% CI) & $\Delta$RMST \\
\midrule
1 & 0.64 & $0.005$ & 2.39 (1.28--4.46) & 386 & 0.59 & $0.117$ & 1.45 (0.91--2.30) & 102 \\
2 & 0.55 & $0.450$ & 1.26 (0.69--2.28) & 195 & 0.60 & $0.057$ & 1.56 (0.98--2.48) & 653 \\
3 & 0.58 & $0.288$ & 1.37 (0.76--2.47) & 159 & 0.54 & $0.576$ & 1.14 (0.72--1.80) & 107 \\
4 & 0.57 & $0.053$ & 1.80 (0.98--3.28) & 302 & 0.56 & $0.059$ & 1.55 (0.98--2.47) & 365 \\
5 & 0.54 & $0.690$ & 1.13 (0.62--2.07) & 104 & 0.55 & $0.020$ & 1.74 (1.09--2.79) & 689 \\
6 & 0.57 & $0.105$ & 1.63 (0.90--2.97) & 240 & 0.57 & $0.185$ & 1.37 (0.86--2.17) & 329 \\
7 & 0.56 & $0.051$ & 1.81 (0.99--3.30) & 270 & 0.53 & $0.510$ & 1.17 (0.74--1.85) & 118 \\
8 & 0.51 & $0.264$ & 1.40 (0.78--2.51) & 223 & 0.52 & $0.889$ & 1.03 (0.65--1.63) & 0 \\
9 & 0.53 & $0.271$ & 1.39 (0.77--2.53) & 254 & 0.55 & $0.081$ & 1.51 (0.95--2.42) & 548 \\
10 & 0.51 & $0.492$ & 0.81 (0.45--1.48) & -13 & 0.53 & $0.258$ & 1.30 (0.82--2.05) & 213 \\
11 & 0.50 & $0.752$ & 1.10 (0.61--1.97) & 88 & 0.54 & $0.206$ & 1.35 (0.85--2.13) & 432 \\
12 & 0.63 & $0.010$ & 2.18 (1.19--4.00) & 323 & 0.56 & $0.143$ & 1.41 (0.89--2.24) & 586 \\
13 & 0.52 & $0.641$ & 1.15 (0.64--2.08) & 67 & 0.52 & $0.646$ & 1.11 (0.71--1.75) & 350 \\
14 & 0.51 & $0.782$ & 1.09 (0.61--1.95) & 40 & 0.52 & $0.715$ & 1.09 (0.69--1.73) & 156 \\
15 & 0.55 & $0.580$ & 0.85 (0.47--1.53) & 8 & 0.56 & $0.288$ & 1.28 (0.81--2.03) & 201 \\
16 & 0.52 & $0.605$ & 1.17 (0.65--2.12) & 53 & 0.54 & $0.137$ & 1.42 (0.89--2.24) & 234 \\
17 & 0.59 & $0.486$ & 1.23 (0.69--2.21) & 99 & 0.59 & $2.9\times10^{-4}$ & 2.37 (1.47--3.84) & 791 \\
18 & 0.56 & $0.022$ & 2.00 (1.09--3.67) & 262 & 0.57 & $0.660$ & 1.11 (0.70--1.76) & -137 \\
19 & 0.52 & $0.420$ & 1.28 (0.70--2.31) & 126 & 0.50 & $0.265$ & 0.77 (0.48--1.22) & -234 \\
20 & 0.58 & $0.600$ & 1.17 (0.65--2.10) & 43 & 0.50 & $0.907$ & 0.97 (0.61--1.54) & 93 \\
\bottomrule
\end{tabular}
\end{table*}

\subsection{BRAF-Mutant Metastatic Colorectal Cancer (BRAF-CRC)}
We further evaluate SAIL on a cohort of patients with BRAF-mutant metastatic colorectal cancer~\cite{S:trahearn2025computational}. Multiplexed tissue images were segmented and each cell assigned one of eight coarse cell-type labels: cancer epithelial (cep), normal epithelial (nep), fibroblast (fib), lymphocyte (lym), neutrophil (neu), macrophage (mac), endothelial (end), and myocyte (myo) cells. Cell type is one-hot encoded as the node feature vector, with 2D cell centroids as node coordinates. Each tissue image is converted into a $k_{\mathrm{NN}}$-nearest-neighbour spatial graph ($k_{\mathrm{NN}}=32$); where a patient contributes multiple images, per-image SAIL scores are averaged to obtain a single per-patient score. Overall survival is measured from diagnosis of metastatic disease, with death as the event and patients alive at last follow-up treated as censored. SAIL learns $30$ modes jointly across all patient graphs (batch size $8$) with $\lambda_{\mathrm{corr}}=\lambda_{\mathrm{sparse}}=0$.The best-performing SAIL Mode contrasts stromal cells (myocytes, fibroblasts and endothelial cells) with immune cells (lymphocytes and neutrophils) and stratifies patients by overall survival (log-rank $p=0.002$, HR $=1.83$; Supplementary Figure~\ref{fig:braf_analysis}). This exceeds univariate Geary's $C$ on every individual cell type and the mean-aggregate baseline (Supplementary Table~\ref{tab:braf_baselines}).

Zoomed-in regions of the local SAIL scores $I^{\mathrm{local}}$ and LISA clusters are shown in Supplementary Figure~\ref{TCGA_ZOOMED_examples}, for TCGA-BRCA and TCGA-CRC patch graphs (Supplementary Figure~\ref{TCGA_ZOOMED}) and for BRAF-CRC and ORION-CRC cell graphs (Supplementary Figure~\ref{CELL_ZOOMED}).

\begin{table*}[!t]
\centering
\small
\caption{Baseline survival associations in the ORION cohort using inverse-distance spatial weighting to compute univariate global spatial autocorrelation using Geary's $C$. For the multivariate case, we report the mean of the univariate spatial autocorrelations across proteins (Mean aggregate). Results are also compared with immune groups (IFM1--4) defined in the original ORION study.}
\begin{tabular}{lcccc}
\toprule
Covariate & C-Index & Log-rank $p$-value & HR (high vs low) & $\Delta$RMST (low--high, days) \\
\midrule
Mean aggregate & 0.68 & 0.04 & 3.15 & 362 \\
CD68 & 0.61 & 0.05 & 3.00 & 643 \\
CD4 & 0.69 & 0.13 & 2.28 & 385 \\
Grade & 0.61 & 0.13 & 2.20 & 393 \\
E-cadherin & 0.59 & 0.14 & 2.22 & 440 \\
CD3e & 0.63 & 0.15 & 2.18 & 460 \\
SMA & 0.63 & 0.21 & 2.00 & 428 \\
IFM4 & 0.58 & 0.24 & 0.51 & -296 \\
CD45RO & 0.61 & 0.25 & 1.88 & 386 \\
CD45 & 0.61 & 0.25 & 1.88 & 339 \\
IFM3 & 0.57 & 0.29 & 1.76 & 265 \\
Hoechst & 0.60 & 0.32 & 1.70 & 245 \\
CD8a & 0.60 & 0.37 & 1.62 & 197 \\
CD20 & 0.69 & 0.38 & 1.60 & 225 \\
Pan-CK & 0.67 & 0.39 & 1.58 & 170 \\
IFM2 & 0.56 & 0.40 & 0.58 & -209 \\
FOXP3 & 0.64 & 0.44 & 1.52 & 190 \\
CD163 & 0.56 & 0.49 & 0.69 & -210 \\
IFM1 & 0.53 & 0.63 & 0.75 & -61 \\
PD-1 & 0.54 & 0.66 & 1.27 & 154 \\
Ki67 & 0.51 & 0.72 & 0.82 & -101 \\
CD31 & 0.58 & 0.84 & 1.12 & 82 \\
PD-L1 & 0.50 & 0.99 & 1.01 & 30 \\
\bottomrule
\end{tabular}
\label{tab:orion_baselines}
\end{table*}

\begin{figure*}[!t]
  \centering
   \includegraphics[width=6in]{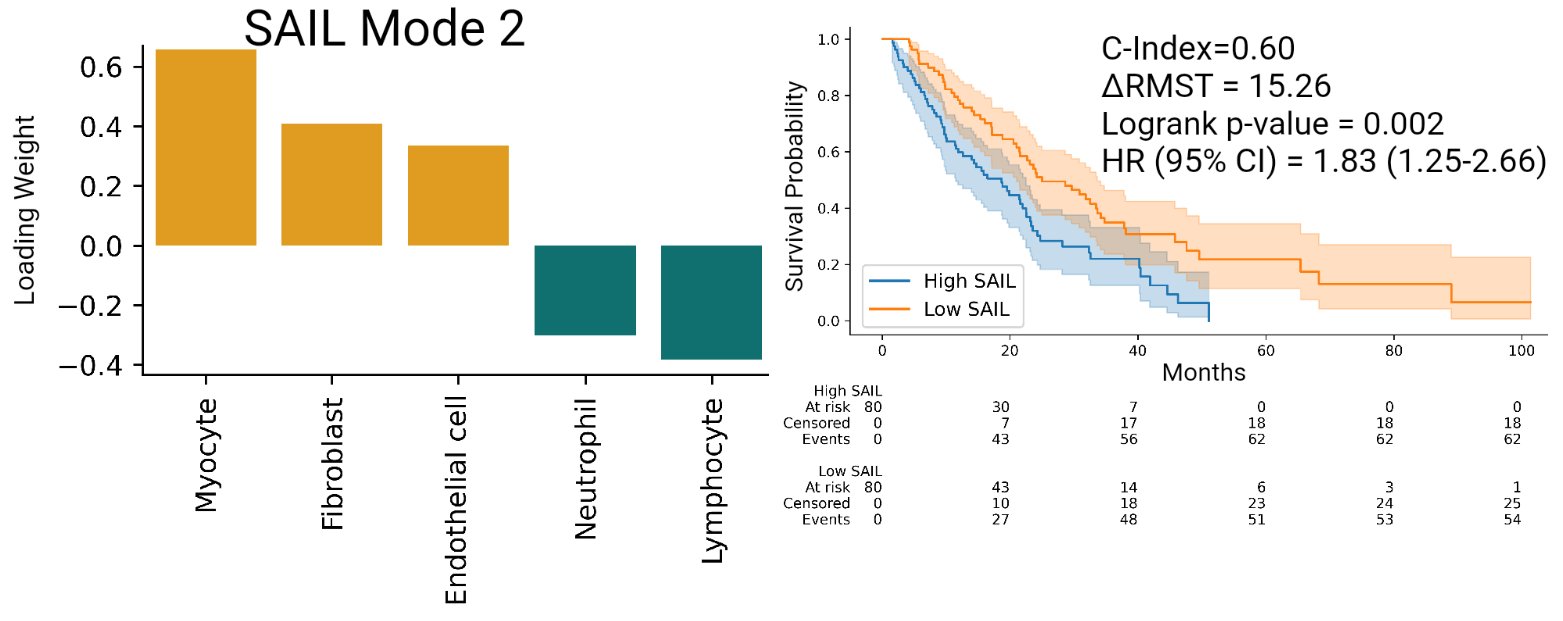}
   \caption{Unsupervised SAIL Modes link spatial coherence to clinical outcome. A SAIL Mode was identified in the BRAF-CRC cohort by minimising Geary's $C$ over one-hot encoded cell types (cancer epithelial cells, normal epithelial cells, fibroblasts, lymphocytes, neutrophils, macrophages, endothelial cells, and myocytes). The bar chart shows the learned $f_{\mathrm{map}}$ weights assigned to each cell type. Global SAIL scores derived from this mode stratify patients into distinct survival groups.}
   \label{fig:braf_analysis}
\end{figure*}

\begin{table*}[!t]
\centering
\small
\caption{Baseline survival associations in the BRAF-CRC cohort using inverse-distance spatial weighting to compute univariate global spatial autocorrelation using Geary's $C$ on each one-hot cell-type channel. The multivariate covariate (Mean aggregate) corresponds to the mean of the univariate cell-type spatial autocorrelations.}
\begin{tabular}{lcccc}
\toprule
Covariate & C-Index & Log-rank $p$-value & HR (high vs low) & $\Delta$RMST (low--high, months) \\
\midrule
Mean aggregate & 0.59 & 0.01 & 1.58 & 10.4 \\
Endothelial cell & 0.58 & 0.02 & 1.57 & 11.8 \\
Lymphocyte & 0.59 & 0.06 & 1.42 & 7.2 \\
Cancer epithelial cell & 0.57 & 0.06 & 1.42 & 9.2 \\
Fibroblast & 0.57 & 0.07 & 1.41 & 7.5 \\
Normal epithelial cell & 0.55 & 0.26 & 1.23 & 0.9 \\
Neutrophil & 0.56 & 0.27 & 1.23 & 3.7 \\
Myocyte & 0.54 & 0.45 & 0.87 & -2.8 \\
Macrophage & 0.53 & 0.80 & 0.95 & -1.1 \\
\bottomrule
\end{tabular}
\label{tab:braf_baselines}
\end{table*}

\begin{figure*}[!t]
\centering
\subfloat[]{\includegraphics[width=2.5in]{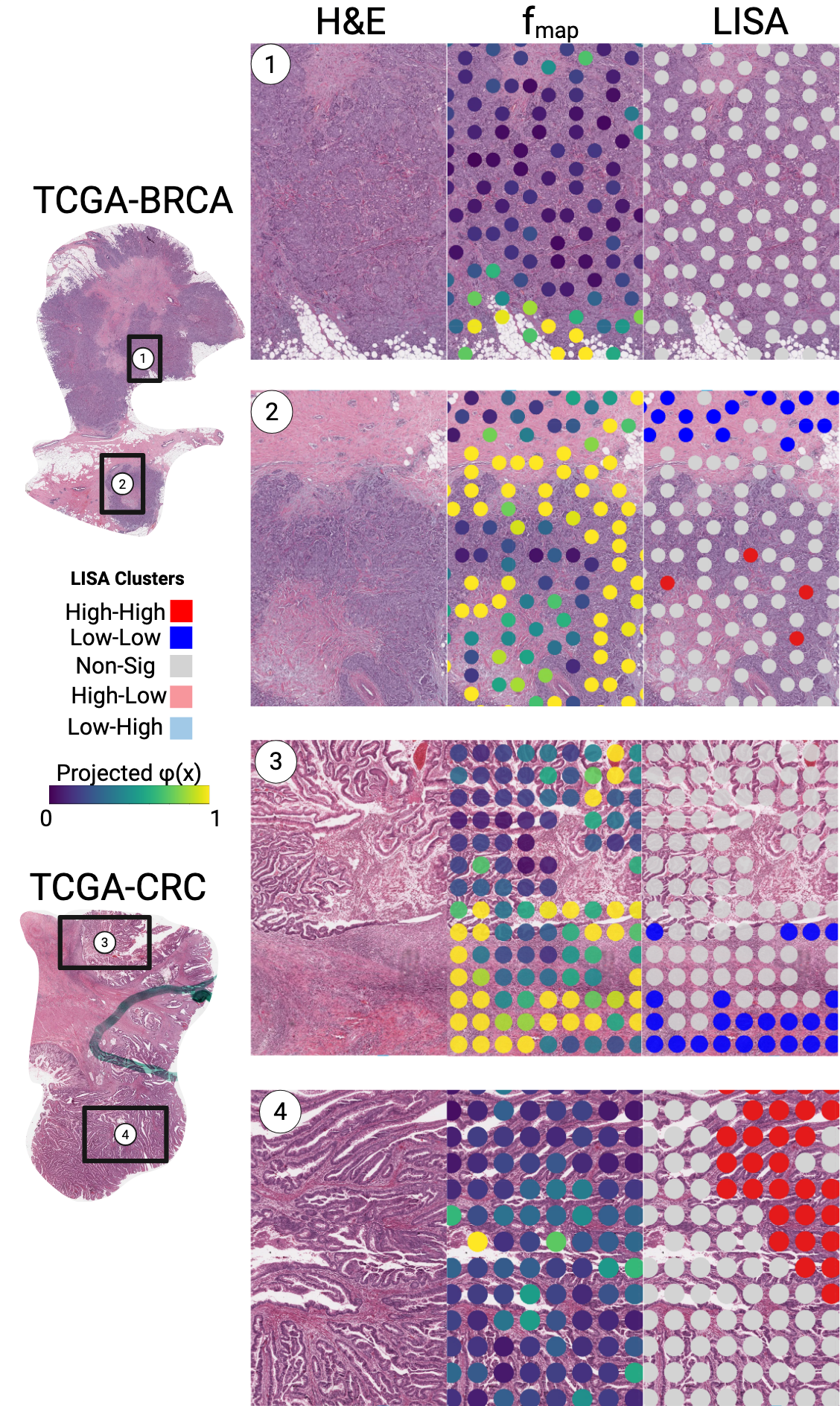}%
\label{TCGA_ZOOMED}}
\hfil
\subfloat[]{\includegraphics[width=2.5in]{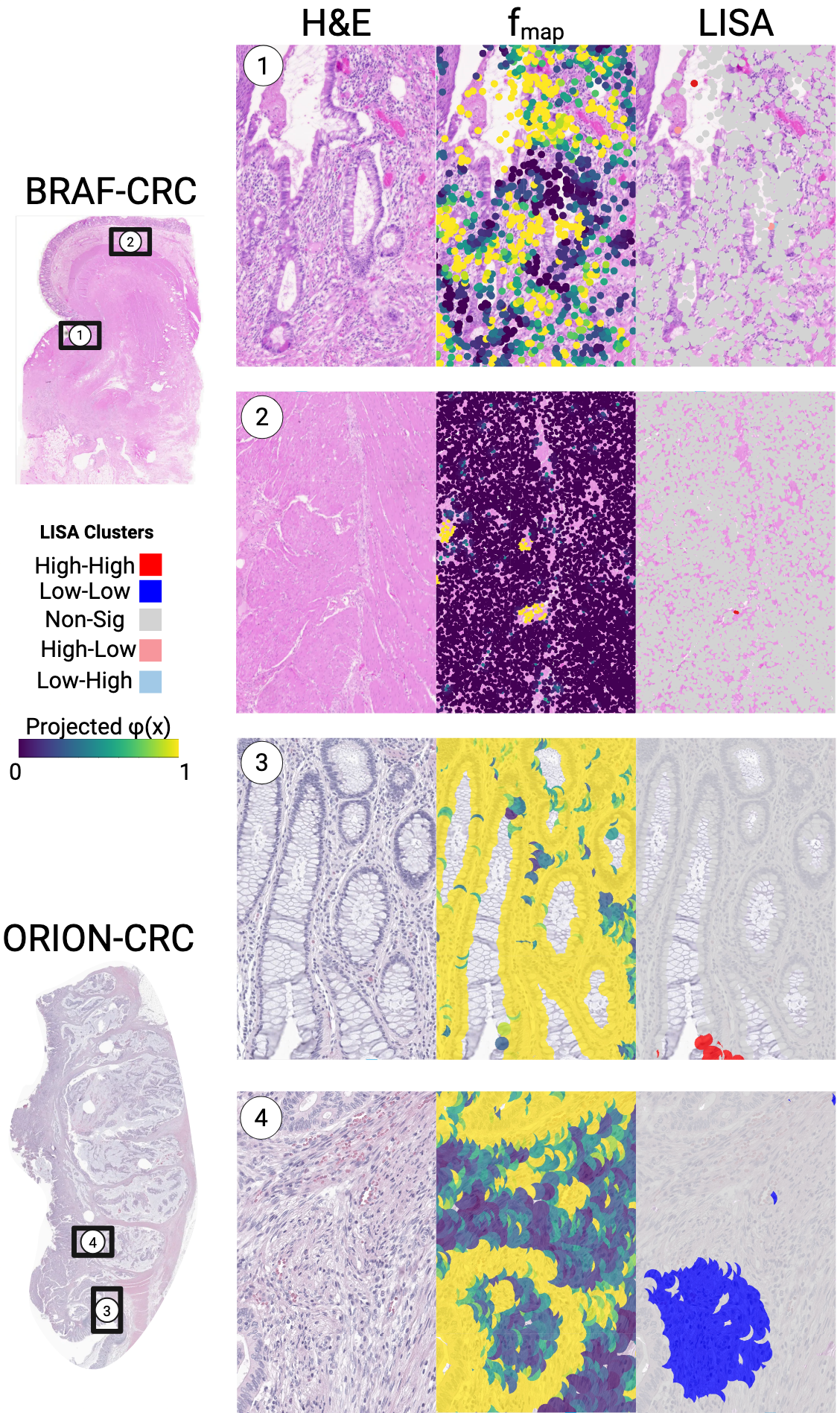}%
\label{CELL_ZOOMED}}
\caption{Zoomed-in regions of SAIL outputs. (a) Patch graphs for TCGA-BRCA and TCGA-CRC. We show two zoomed-in regions per patient, including the $I^{\mathrm{local}}$ score and LISA clusters for each graph node. (b) Cell graphs for BRAF-CRC and ORION-CRC. We show two zoomed-in regions per patient, including the $I^{\mathrm{local}}$ score and LISA clusters for each graph node.}
\label{TCGA_ZOOMED_examples}
\end{figure*}

\section{Additional details on Supervised classification of mesothelioma subtypes with SAIL}
\label{sec:supp_mesograph}
Implementation details for the mesothelioma subtyping analysis of
Section~\ref{sec:mesograph}.

\subsection{Cross-Validation Protocol}
All models are evaluated on a shared, fixed set of splits over the primary
St.~George's Hospital (SGH) cohort. Following the original MesoGraph protocol, the outer split is a leave-one-slide-out
cross-validation: each fold holds out all cores from a single slide as the
test set. Within each fold, the remaining cores are further split
$75\%$/$25\%$ into training and validation sets via a stratified split on
the core-level label.

\subsection{Model Architecture}
SAIL is used in a shared-backbone, dual-head configuration: a single SAIL operator (a linear feature map to $16$ projection dimensions, a Geary-style local index, mean global reduction, and a learnable softmax temperature over log-transformed kNN distances) is evaluated once per core, and its
$16$-dimensional global descriptor is read by two separate linear heads, an epithelioid (E) head and a non-epithelioid (NE) head, mirroring MesoGraph's
two-head design. In addition to the learnable spatial temperature, a query--key content term (linear query and key projections of width $16$ is added to the affinity logits, allowing neighbourhood weighting to depend on cell features as well as spatial distance.

\subsection{Training and Model Selection}
Following MesoGraph, both heads are trained jointly using the same pairwise
margin ranking loss (margin $0.8$): for every pair of cores in a batch with
differing labels, the E head is pushed to rank the epithelioid core higher, 
and the NE head is pushed to rank the non-epithelioid core higher. Models are
trained with Adam (learning rate $5\times10^{-4}$, weight decay $0.02$),
batch size $16$, for up to $300$ epochs per fold. The final core-level decision
score is the difference between the NE-head and E-head core outputs, $\hat y=(w_{\mathrm{NE}}-w_{\mathrm{E}})^\top I^{\mathrm{global}}+(b_{\mathrm{NE}}-b_{\mathrm{E}})$, so the effective contribution of SAIL Head $k$ is $w_{\mathrm{NE},k}-w_{\mathrm{E},k}$ (the per-head weights reported in Fig.~\ref{meso_results_figure}); non-epithelioid cores are treated as the positive class for AUC/AP, matching the evaluation convention of Section~\ref{sec:mesograph}.

\subsection{Statistical Comparison}
To test whether SAIL's improvement in AUC is statistically significant, we draw $10{,}000$ paired bootstrap resamples of the test cores. In each resample, the AUC of SAIL and of each comparison model is computed on the same resampled cores, giving a bootstrap distribution of the difference $\Delta\mathrm{AUC}=\mathrm{AUC}_{\mathrm{SAIL}}-\mathrm{AUC}_{\mathrm{comparison}}$. The $p$-value is the fraction of resamples in which $\Delta\mathrm{AUC}\leq0$.

\begin{figure*}[!t]
  \centering
   \includegraphics[width=\textwidth]{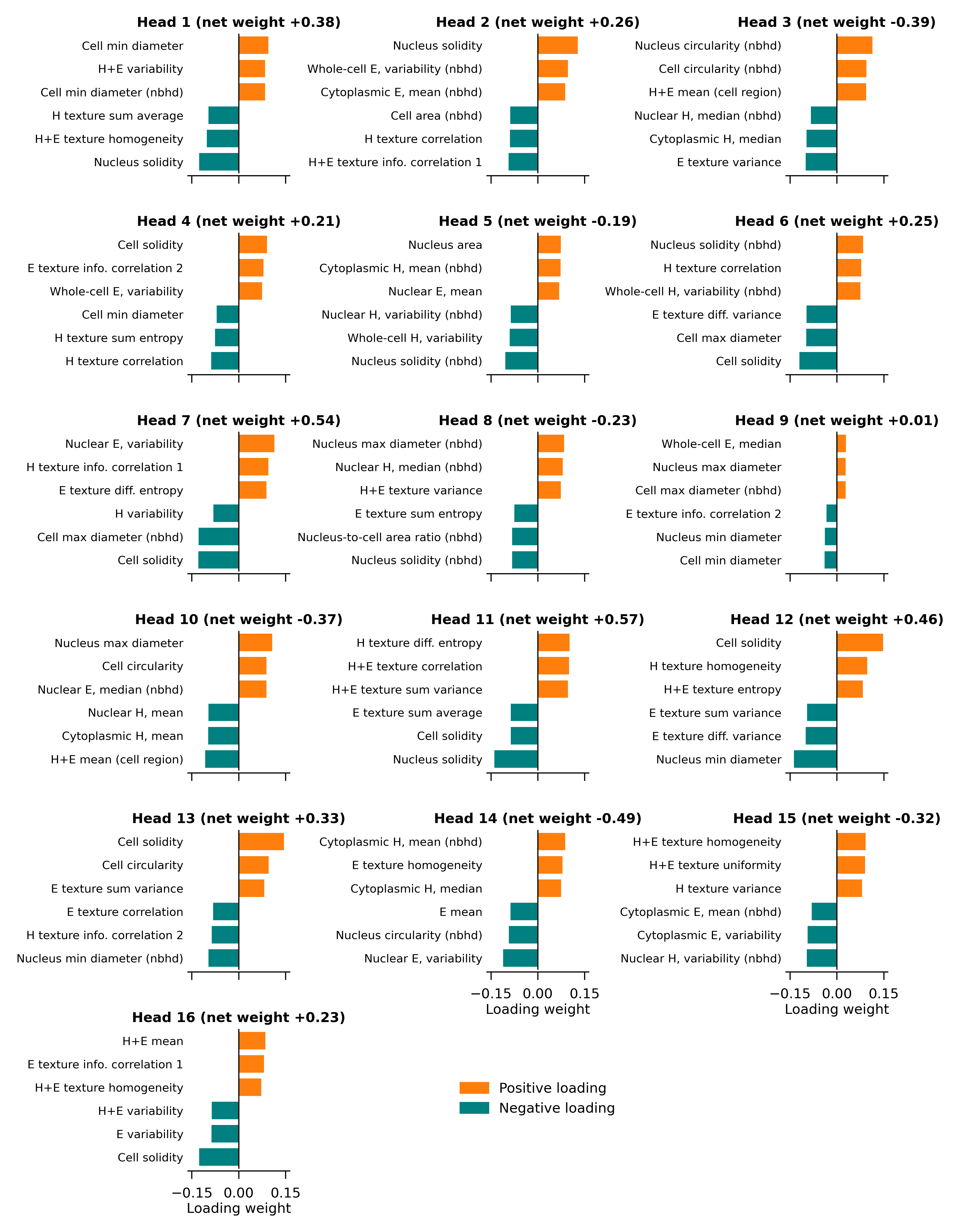}
   \caption{16 learned SAC descriptors (trained using interpretable features) on the mesothelioma dataset used in Section~\ref{sec:supp_mesograph}. For each descriptor we report the three features with the most positive and negative loadings. Here ``nbhd'' denotes a \qty{50}{\um} neighbourhood average.}
   \label{fig:meso_all_modes}
\end{figure*}

\section{Computational Complexity Analysis}
\label{sec:supp_scaling}

\subsection{Empirical Computational Complexity Analysis}
We measure the wall-clock time and peak GPU memory of SAIL as a function of: the number of nodes $n$, the input feature dimension $d$, the number of learned modes (projection dimension) $h$, and the neighbourhood size $k_{\mathrm{NN}}$. We further quantify the cost added by two optional components: the orthogonality (Stiefel) constraint on $f_{\mathrm{map}}$ and a learned feature affinity $f_{\mathrm{affinity}}$.

\subsection{Benchmark Design}
\label{sec:supp_scaling_design}

\paragraph{Synthetic data} Node coordinates are drawn uniformly in a square of side $\sqrt{n}$. Node features are a rank-$10$ mixture of smooth random spatial fields ($64$ random Fourier features with length-scale equal to one eighth of the square side), linearly mixed into $d$ dimensions, perturbed with Gaussian noise (standard deviation $0.5$), and standardised per feature. A $k$-nearest-neighbour graph is built with a $k$-d tree (excluding self-loops), and log-transformed neighbour distances are passed to the affinity module. Data are generated once per $(n,d)$ pair and shared across model variants and repeats.

\paragraph{Model variants} All runs use the unsupervised mode-learning configuration with the Moran-style local index (Eq.~\eqref{eq:moran_local}), mean global reduction, a learnable spatial temperature $\tau$, and decorrelation ($\lambda_{\mathrm{corr}}=8$) and sparsity ($\lambda_{\mathrm{sparse}}=5$) regularisation, optimised with Adam (learning rate $0.05$) on the full graph. To isolate the cost of each component we benchmark all 6 combinations of $f_{\mathrm{map}}$ and $f_{\mathrm{affinity}}$:
\begin{itemize}
    \item \textbf{linear} $f_{\mathrm{map}}$: an unconstrained, column-normalised projection $\mathbf{W}\in\mathbb{R}^{d\times h}$;
    \item \textbf{orthogonal} $f_{\mathrm{map}}$ (default): $\mathbf{W}$ constrained to the Stiefel manifold, parametrised as the first $h$ columns of the matrix exponential of a $d\times d$ skew-symmetric matrix;
\end{itemize}
with the feature affinity,
\begin{itemize}
    \item \textbf{none}: $f_{\mathrm{affinity}}=0$, so that attention depends on spatial distance only (default);
    \item \textbf{shared-map}: $f_{\mathrm{affinity}}(x_j,x_{j'})=\lVert f_{\mathrm{map}}(x_j)-f_{\mathrm{map}}(x_{j'})\rVert_2^2$, reusing the learned projection;
    \item \textbf{query--key}: $f_{\mathrm{affinity}}(x_j,x_{j'})=q(x_j)^\top k(x_{j'})/\sqrt{16}$, with learned $16$-dimensional linear query and key projections.
\end{itemize}

\paragraph{Sweeps} Starting from a reference configuration of $n=10^5$, $d=1{,}000$, $h=25$ and $k_{\mathrm{NN}}=16$, we vary one factor at a time with the others held at their reference values: $n\in\{10^3,\,3\times10^3,\,10^4,\,3\times10^4,\,10^5,\,3\times10^5,\,10^6,\,3\times10^6,\,10^7\}$, $d\in\{50,100,300,1{,}000,3{,}000,10{,}000\}$, $h\in\{1,5,10,25,50,100\}$ and $k_{\mathrm{NN}}\in\{4,8,16,32,64\}$. If a variant exceeds GPU memory, this is recorded and larger values in that sweep are skipped for that variant.

\subsection{Measurement Protocol}
\label{sec:supp_scaling_protocol}

A training step comprises one forward pass over the full graph, evaluation of the objective and both regularisers, a backward pass, and an Adam update; since each benchmark uses a single graph, one step corresponds to one epoch. After $3$ untimed warm-up steps (CUDA initialisation and kernel selection), we record $5$ repeats, each timing $10$ consecutive steps and reporting the mean time per step. The GPU is synchronised before and after every timed interval. Inference is timed as a single forward pass. Peak memory is the maximum memory allocated to tensors during the measurement, via \texttt{torch.cuda.max\_memory\_allocated}. We report means with $\pm1$ standard deviation over repeats. Empirical scaling exponents $b$ (cost $\propto x^{b}$) are least-squares slopes on log--log axes.

All benchmarks were run on a single NVIDIA A100-SXM4-80GB GPU in a DGX A100 node (AMD EPYC 7742 CPU, 2\,TB RAM), using PyTorch 2.5.1 with CUDA 12.4, in FP32 arithmetic (TF32 disabled).

\subsection{Results}
\label{sec:supp_scaling_results}

For the default configuration (Supplementary Figure~\ref{fig:supp_scaling_default}), time per training step is dominated by a fixed, $n$-independent cost of $33$--$38$\,ms up to $n\approx10^5$ and reaches $132$\,ms at $n=3\times10^6$; inference requires $3.9$\,ms at the reference size and $38$\,ms at $n=3\times10^6$. Peak memory grows linearly with $n$ (exponent $0.99$), reaching $27$\,GB at $n=3\times10^6$; $n=10^7$ exceeds the $80$\,GB GPU. The fixed cost is the matrix exponential of the orthogonal map: time per step grows approximately cubically with $d$ (exponent $2.91$; $37$\,ms at $d=1{,}000$ and $30$\,s at $d=10{,}000$), whereas the linear map costs $6$\,ms at $d=1{,}000$ and $13$\,ms at $d=10{,}000$ (Supplementary Figure~\ref{fig:supp_scaling_components}\,a,\,d). The number of modes $h$ and neighbours $k$ have only a minor effect on runtime. Adding a learned affinity increases peak memory by approximately $14\times$ ($13.0$ versus $0.9$\,GB at the reference size; Supplementary Table~\ref{tab:supp_scaling_reference}), growing linearly with both $n$ and $k$, which limits these variants to $n=3\times10^5$ nodes on a single $80$\,GB GPU (Supplementary Figure~\ref{fig:supp_scaling_components}\,b,\,c,\,e,\,f). For reference, constructing the $k$-nearest-neighbour graph on the CPU scaled approximately linearly in $n$, taking $0.36$\,s at $n=10^5$ and $46$\,s at $n=10^7$, and fitting the default model for $100$ epochs at the reference size took $4.4$\,s.

Thus for the training step, the sparse top-$k_{\mathrm{NN}}$ model with a linear $f_{\mathrm{map}}$ requires $\mathcal{O}(ndh)$ time for the projection, $\mathcal{O}(nk_{\mathrm{NN}}h)$ for the affinity-weighted local index and $\mathcal{O}(nh^2)$ for the decorrelation penalty, with $\mathcal{O}(nd+nk_{\mathrm{NN}}h)$ memory, dominated by the input features and the gathered neighbour projections. The orthogonal constraint adds an $n$-independent $\mathcal{O}(d^3)$ term from the $d\times d$ matrix exponential and $\mathcal{O}(d^2)$ memory. Both learned affinities gather the raw features of every neighbour, adding $\mathcal{O}(nk_{\mathrm{NN}}d)$ memory and $\mathcal{O}(nk_{\mathrm{NN}}dr)$ time, where $r=16$ for the query--key affinity and $r=h$ for the shared-map affinity. All variants are therefore linear in $n$ in both time and memory; they differ in their dependence on $d$.

\begin{table}[!t]
\centering
\caption{Computational cost of each SAIL variant at the reference configuration ($n=10^5$, $d=1{,}000$, $h=25$, $k_{\mathrm{NN}}=16$) on one NVIDIA A100-80GB GPU; mean $\pm$ standard deviation over $5$ repeats. \emph{Max.\ $n$} is the largest graph size in the sweep that fitted in GPU memory.}
\label{tab:supp_scaling_reference}
\setlength{\tabcolsep}{3pt}
\begin{tabular}{llcccc}
\toprule
$f_{\mathrm{map}}$ & $f_{\mathrm{affinity}}$ & \shortstack{Train step\\(ms)} & \shortstack{Inference\\(ms)} & \shortstack{Peak mem.\\(GB)} & Max.\ $n$ \\
\midrule
Linear     & none       & $6.1 \pm 0.0$   & $4.0 \pm 0.1$  & $0.91$  & $3\times10^6$ \\
Orthogonal & none       & $37.7 \pm 0.6$  & $3.9 \pm 0.0$  & $0.92$  & $3\times10^6$ \\
Linear     & shared-map & $63.2 \pm 0.5$  & $59.2 \pm 0.8$ & $12.98$ & $3\times10^5$ \\
Orthogonal & shared-map & $129.6 \pm 3.8$ & $59.2 \pm 0.3$ & $13.00$ & $3\times10^5$ \\
Linear     & query--key & $62.7 \pm 0.3$  & $51.0 \pm 1.6$ & $13.02$ & $3\times10^5$ \\
Orthogonal & query--key & $94.8 \pm 1.0$  & $50.6 \pm 1.5$ & $13.03$ & $3\times10^5$ \\
\bottomrule
\end{tabular}
\end{table}

\begin{figure*}[!t]
  \centering
   \includegraphics[width=\textwidth]{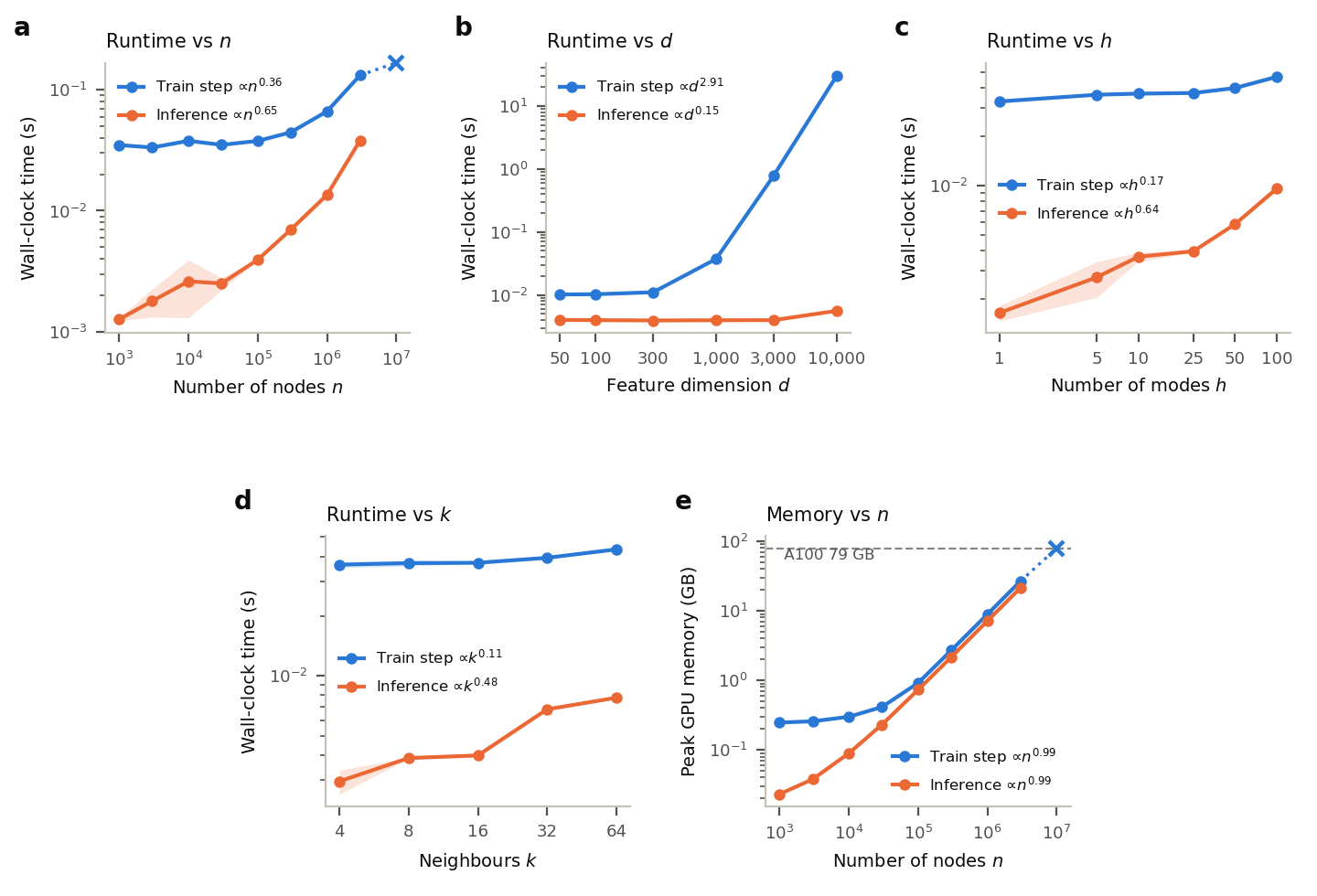}
   \caption{Computational scaling of SAIL in its default configuration (orthogonal $f_{\mathrm{map}}$, no learned feature affinity, Moran-style objective with decorrelation and sparsity regularisation) on synthetic spatial data (Supplementary Section~\ref{sec:supp_scaling_design}). Each panel varies one factor with the others held at $n=10^5$ nodes, $d=1{,}000$ features, $h=25$ modes and $k_{\mathrm{NN}}=16$ neighbours. \emph{Train step}: one full-graph forward pass, backward pass and Adam update (one epoch); \emph{Inference}: one forward pass of the trained model. (a) Runtime versus number of nodes $n$: a fixed, $n$-independent cost dominates up to $n\approx10^5$. (b) Runtime versus feature dimension $d$: training time grows approximately as $d^{3}$, owing to the $d\times d$ matrix exponential of the Stiefel parametrisation, while inference is essentially independent of $d$. (c) Runtime versus number of modes $h$ and (d) versus neighbourhood size $k_{\mathrm{NN}}$: both have only a minor effect. (e) Peak GPU memory versus $n$ grows linearly. The dashed line marks GPU capacity (NVIDIA A100, $80$\,GB); $\times$ marks the first size that did not fit, joined by a dotted line to the largest size that did. Points are means over $5$ repeats with $\pm1$ standard deviation bands; legend exponents are log--log least-squares slopes. All axes are logarithmic.}
   \label{fig:supp_scaling_default}
\end{figure*}

\begin{figure*}[!t]
  \centering
   \includegraphics[width=\textwidth]{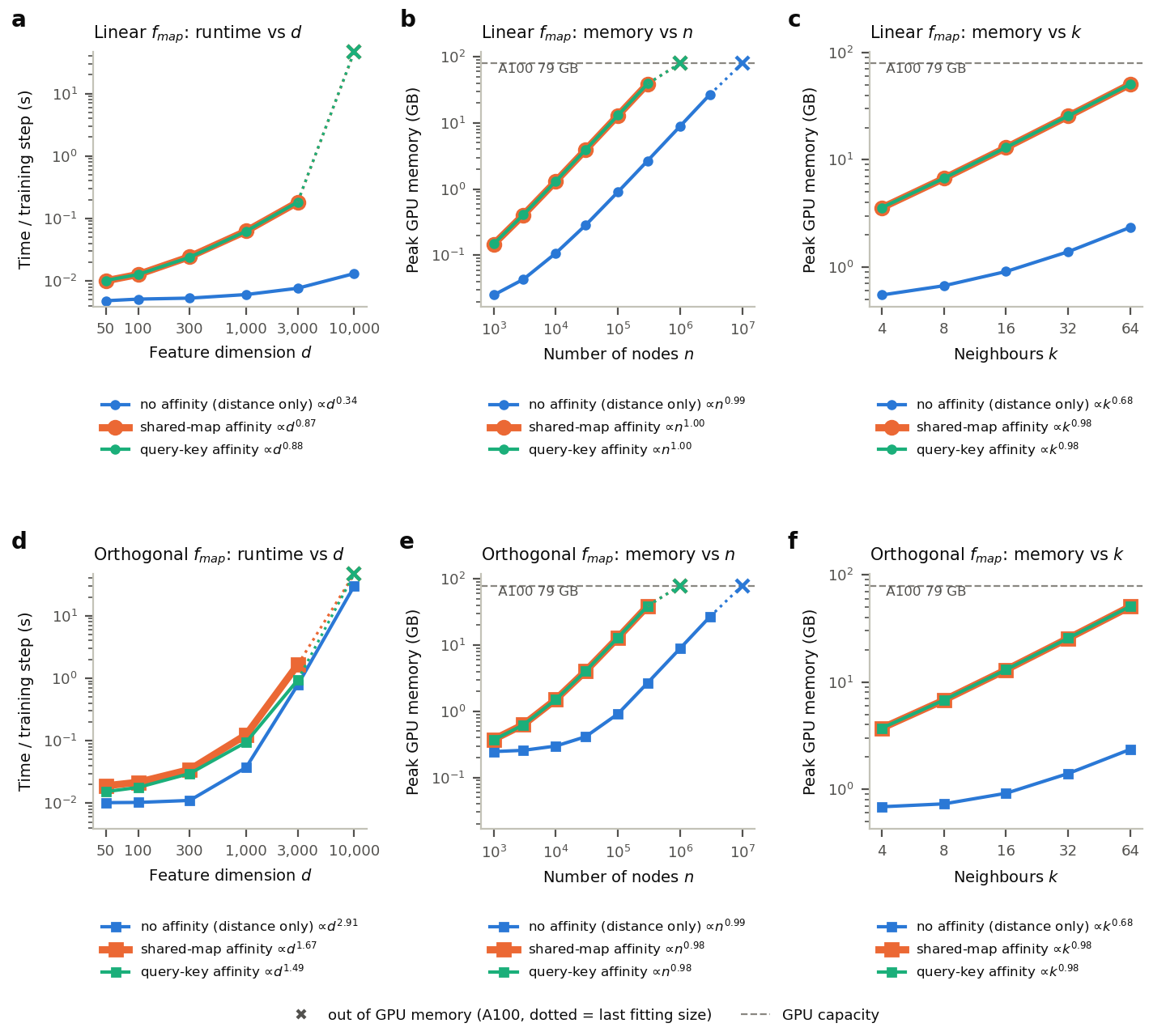}
   \caption{Computational cost of SAIL's optional components. Top row (a--c): linear $f_{\mathrm{map}}$; bottom row (d--f): orthogonal (Stiefel) $f_{\mathrm{map}}$, the default. Colours denote the feature affinity $f_{\mathrm{affinity}}$: none (blue; attention from spatial distance only), shared-map (orange) and query--key (green) (Supplementary Section~\ref{sec:supp_scaling_design}). Other factors are held at $n=10^5$, $d=1{,}000$, $h=25$, $k_{\mathrm{NN}}=16$. (a, d) Time per training step versus feature dimension $d$: nearly flat with the linear map, approximately cubic with the orthogonal map. (b, e) Peak GPU memory versus number of nodes $n$: linear for all variants, but both learned affinities require approximately $14\times$ more memory, because they gather the $d$-dimensional features of every neighbour ($\mathcal{O}(nk_{\mathrm{NN}}d)$), and therefore exceed GPU memory at $n=10^6$. (c, f) Peak GPU memory versus neighbourhood size $k_{\mathrm{NN}}$: linear in $k_{\mathrm{NN}}$ for the learned affinities. The shared-map and query--key affinities have near-identical memory (and near-identical runtime with the linear map), since both are dominated by the same neighbour-feature gather; the shared-map curve is drawn wider underneath so that both remain visible. $\times$ marks the first configuration that did not fit in GPU memory. The dashed line marks GPU capacity. Means over $5$ repeats with $\pm1$ standard deviation bands; all axes are logarithmic.}
   \label{fig:supp_scaling_components}
\end{figure*}

\section{Recovering classical spatial autocorrelation with SAIL in additional synthetic datasets}
\label{sec:supp_toy_construction}
Beyond analysis on the graphs with known spatial structure in Section~\ref{sec:synthetic_experiments}, we provide additional synthetic examples that qualitatively probe SAIL across a wider range of spatial structures.

Each 2D toy graph in Supplementary Figure~\ref{fig:toy_supp} is built on a regular grid, connected via a 4-neighbour (rook-adjacency) edge set, with a single scalar feature per node designed to induce a specific, known spatial-autocorrelation regime. In each case, SAIL is trained on the single
graph to learn $K=1$ mode maximising the mean of the Moran-style local index (Eq.~\eqref{eq:moran_local}), using Adam with learning rate $0.05$ for $50$ epochs, an unconstrained (non-Stiefel) linear projection $f_{\mathrm{map}}$, and no learnable feature affinity ($f_{\mathrm{affinity}}=0$). LISA hotspot maps are computed as in Section~\ref{hotspot_discovery} with a significance threshold of $0.05$. The resulting graphs are shown in Supplementary Figure~\ref{fig:toy_supp}.

\paragraph{Perlin field (top-left)} Node values are drawn from a 2D Perlin
noise field on a $32\times32$ grid. This yields complex but strongly positive local spatial autocorrelation, since neighbouring lattice
points inherit similar noise amplitudes.

\paragraph{Checkerboard (top-right)} On a $16\times16$ grid, each node is assigned a binary value based on the parity of its lattice indices, producing a sharply alternating high-contrast pattern in which every edge connects
nodes of opposite value. This is the canonical negative-autocorrelation construction: every 4-neighbour pair is maximally dissimilar, so the expected Moran's $I$ is strongly negative.

\paragraph{Plateau (bottom-left)} On a $40\times40$ grid over $[0,1]^2$, node
values are a two-level step function, splitting the domain in half along one axis. This produces clustered positive autocorrelation within each half of the domain, together with a sharp discontinuity at the dividing boundary, testing SAIL's ability to localise a hard spatial edge.

\paragraph{Concentric rings (bottom-right)} On a $50\times50$ grid over $[-1,1]^2$, node values follow a radial oscillation, $\sin(8r)$, where $r$ is the distance from the domain centre. This produces a wave pattern that alternates between locally positive and locally negative spatial dependence as $r$ increases, giving concentric High--High and Low--Low rings.

\paragraph{3D point-set terrain} To demonstrate that SAIL applies directly to point sets, we additionally generate a synthetic 3D terrain. A $200\times200$ grid is sampled over $[0,1]^2$ and a height is assigned to each location via 2D Perlin noise, giving a smooth synthetic elevation surface. The height value is used both as the node's third spatial coordinate (each point is embedded in $\mathbb{R}^3$ as $(x,y,z)$, with $z$ the Perlin height) and as the scalar node feature. This yields a point cloud of $40{,}000$ points with no explicit edges: a $k$-nearest-neighbour graph ($k_{\mathrm{NN}}=32$) is built directly from the 3D coordinates. SAIL is trained with a single mode, Moran-style objective, unconstrained linear
projection, Adam, learning rate $0.05$, $50$ epochs. In Supplementary Figure~\ref{fig:POINT_SET} the resulting mode is visualised by jointly colouring points by their LISA cluster assignment and by the projected scalar field $f_{\mathrm{map}}(x)$.

\begin{figure*}[t]
\centering
\subfloat[]{\includegraphics[width=3.4in]{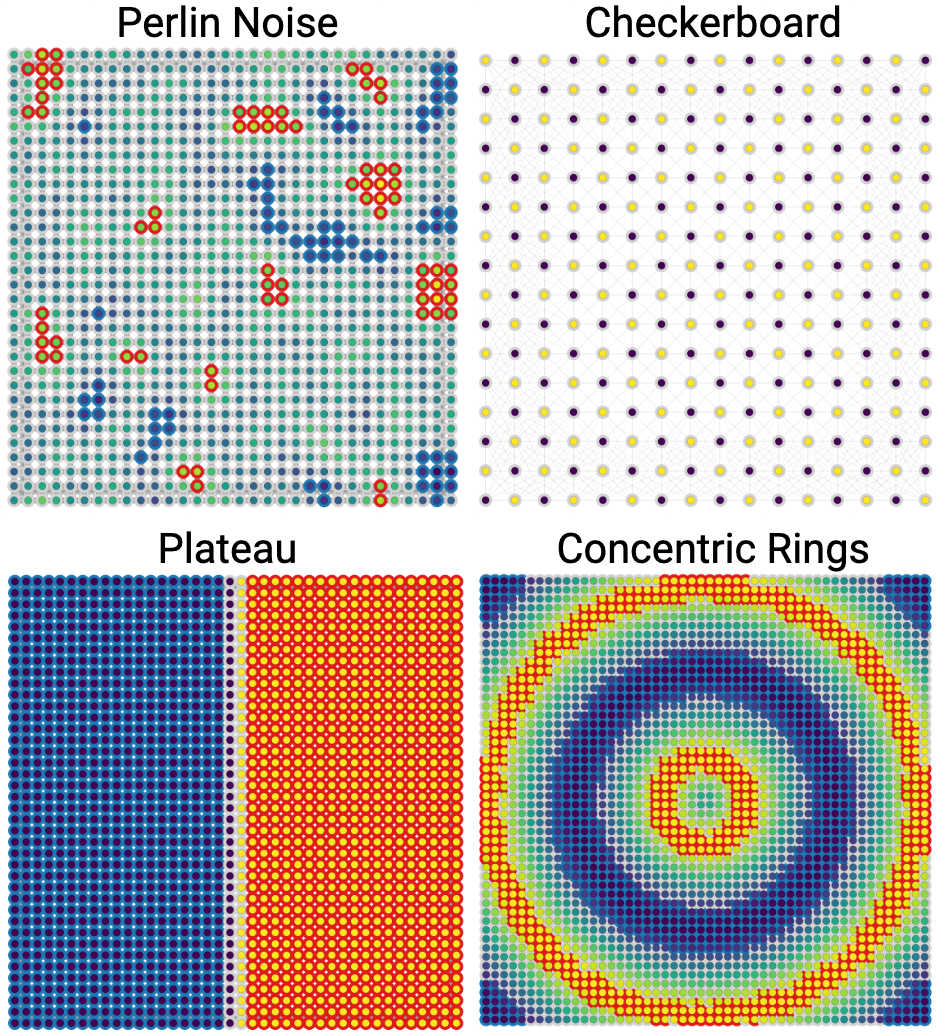}%
\label{fig:toy_supp}}
\hfil
\subfloat[]{\includegraphics[width=3.4in]{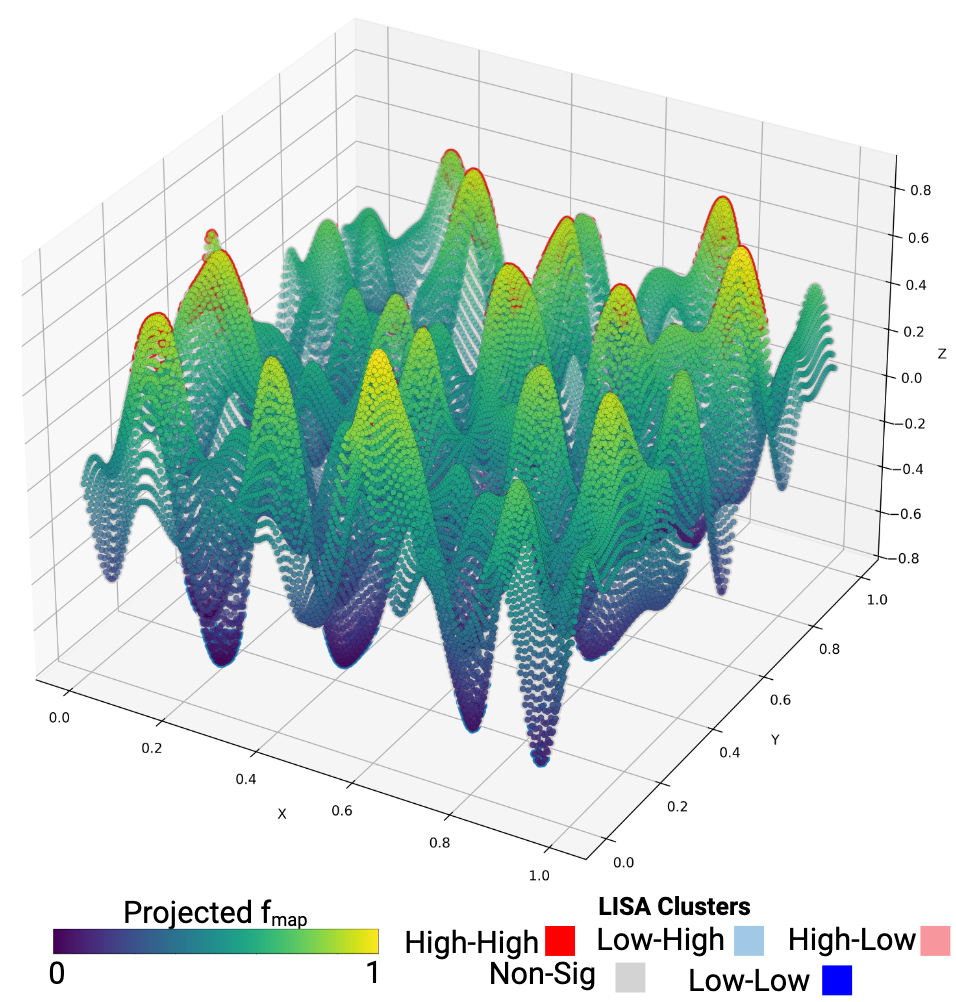}%
\label{fig:POINT_SET}}
\caption{Additional synthetic tests of SAIL. (a) Toy graph examples, coloured by both LISA clusters and the projected scalar field $f_{\mathrm{map}}(x)$. Top-left: a smooth, organically varying spatial field generated with Perlin noise, designed to produce complex but strongly positive local spatial autocorrelation. Top-right: a checkerboard graph with a sharply alternating high-contrast pattern that induces strong negative spatial autocorrelation through repeated adjacent sign flipping. Bottom-left: a plateau graph with a two-level step surface with an abrupt spatial boundary, producing clustered positive autocorrelation within each region and a sharp discontinuity at the divide. Bottom-right: concentric rings with an oscillating radial wave pattern that alternates between positive and negative spatial dependence as distance increases from the centre. (b) A synthetic 3D point set, coloured by both LISA clusters and the projected scalar field $f_{\mathrm{map}}(x)$. It is generated by sampling a grid and assigning heights using 2D Perlin noise, giving a smooth synthetic terrain; the resulting $40{,}000$ points are used to build a $k$-nearest-neighbour graph for SAIL analysis.}
\label{fig:synthetic_supp}
\end{figure*}

\section{Ablation Study}
\label{sec:ablation}
To assess the contribution of each SAIL component, we perform an ablation study on four datasets: the 4i spatial proteomics dataset (Supplementary Table~\ref{tab:hyperparam_results_4i}), the cSCC PD-1/PD-L1 blockade cohort (Supplementary Table~\ref{tab:hyperparam_results_pdl1}), the TCGA-BRCA cohort (Supplementary Table~\ref{tab:hyperparam_results_brca}), and the TCGA-CRC cohort (Supplementary Table~\ref{tab:hyperparam_results_crc}). In every case, $20$ SAIL Modes are learned unsupervised as in Section~\ref{sec:modes} of the main paper, using the same default configuration. Each block of a table varies one component while holding all others at their default (marked $^\dagger$). 

\subsection{4i Spatial Proteomics}
For each of the $10$ annotated subcellular regions we take the best Dice overlap over all modes and both High--High and Low--Low LISA clusters, and report the mean over regions.

\textbf{Diversity regularisation ($\lambda_{\mathrm{corr}}$).} The default $\lambda_{\mathrm{corr}}=1$ gives the highest mean Dice ($0.44$), slightly above no decorrelation ($0.42$). Strong decorrelation degrades performance ($0.36$ at $\lambda_{\mathrm{corr}}=10$ and $0.32$ at $\lambda_{\mathrm{corr}}=100$), consistent with the penalty over-constraining the projections at the expense of spatial coherence.

\textbf{Sparsity regularisation ($\lambda_{\mathrm{sparse}}$).} Mean Dice is unchanged for $\lambda_{\mathrm{sparse}}\in\{0,0.1,1\}$ ($0.44$) and decreases only slightly for stronger sparsity ($0.42$ and $0.41$ at $10$ and $100$). Sparsity therefore improves the interpretability of the learned weights at little cost to region recovery.

\textbf{Local activation ($f_{\mathrm{act}}$).} Identity and ReLU give identical results ($0.44$), since local Geary's $C$ values are non-negative and ReLU leaves them unchanged; a sigmoid gives a marginally lower score ($0.43$). We retain the identity as default to preserve the direct correspondence with classical statistics.

\textbf{Reduction operator ($f_{\mathrm{reduce}}$).} Mean and max reduction perform equally ($0.44$) and sum is marginally lower ($0.43$). Because this analysis uses a single graph, the reduction only rescales the objective and has little effect.

\textbf{Projection parametrisation ($f_{\mathrm{map}}$).} An unconstrained linear projection gives a higher mean Dice than the default Stiefel-constrained projection ($0.47$ vs.\ $0.44$). Because the evaluation takes the best mode per region, it does not penalise redundant modes, and unconstrained projections can place several similar modes near the same compartment. We retain the orthogonality constraint as default because it yields distinct, non-redundant modes, which are easier to interpret.

\textbf{Local statistic.} Geary's $C$ and Moran's $I$ perform equally ($0.44$), whereas Getis--Ord $G^\star$ is substantially lower ($0.32$). This is consistent with $G^\star$ measuring the concentration of high or low values rather than the similarity between neighbouring nodes, which is less suited to delineating compartments.

\textbf{Feature-based affinity ($f_{\mathrm{affinity}}$).} Adding a learnable feature-based affinity reduces mean Dice, both for the query--key variant ($0.33$) and for the shared-$f_{\mathrm{map}}$ variant ($0.36$). In the unsupervised setting, a feature-dependent affinity can increase measured coherence by down-weighting dissimilar neighbours rather than by finding coherent projections.

\textbf{Temperature, neighbourhood size and learning rate.} Learning $\tau$ and fixing $\tau=1$ give the same mean Dice ($0.44$), indicating that on this single, densely sampled graph the $k_{\mathrm{NN}}$ graph already fixes the relevant spatial scale. Results are also robust to the neighbourhood size ($0.43$--$0.44$ for $k_{\mathrm{NN}}\in\{8,16,32,64\}$) and to the learning rate ($0.43$--$0.45$ for $\{0.01,0.05,0.1\}$).

\subsection{cSCC PD-1/PD-L1 Blockade Response}
For each setting we report the mode that best separates responders from non-responders, with a subject-level bootstrap $95\%$ confidence interval and the Mann--Whitney $p$-value.

The default configuration gives an AUC of $0.86$ ($p=0.005$). As on 4i, strong decorrelation ($\lambda_{\mathrm{corr}}\in\{10,100\}$: $0.75$ and $0.78$), a learned feature affinity (query--key $0.80$, shared-$f_{\mathrm{map}}$ $0.75$) and sum reduction ($0.74$) give the largest drops. Sum reduction confounds SAC with the number of cells per patient, which varies across patients in this multi-graph setting, whereas mean and max reduction perform equally ($0.86$). Unlike on 4i, fixing $\tau=1$ lowers performance ($0.79$ vs.\ $0.86$), indicating that learning the spatial scale matters when graphs are analysed jointly. All three local statistics perform similarly ($0.84$--$0.86$), as do the Stiefel and unconstrained projections ($0.86$ vs.\ $0.84$). Larger neighbourhoods improve discrimination, from $0.79$ at $k_{\mathrm{NN}}=8$ to $0.91$ at $k_{\mathrm{NN}}=64$.

\subsection{TCGA-BRCA Survival}
For each setting we report the mode with the highest orientation-free C-Index, $\max(C,1-C)$, together with the log-rank $p$-value and $\Delta$RMST (10-year horizon) of a median split on that mode's score.

The best-mode C-Index is stable across settings ($0.57$--$0.67$), with most configurations between $0.60$ and $0.65$ and the default at $0.64$ ($p=0.005$, $\Delta$RMST $=386$ days). Removing decorrelation ($\lambda_{\mathrm{corr}}=0$) and max reduction give the highest C-Index ($0.67$), whereas strong decorrelation ($\lambda_{\mathrm{corr}}\in\{10,100\}$) loses significant separation ($p=0.46$ and $0.12$). Consistent with the other datasets, Getis--Ord $G^\star$ ($0.57$, $p=0.23$) and the shared-$f_{\mathrm{map}}$ affinity ($0.60$, $p=0.31$) are the weakest variants, Geary's $C$ and Moran's $I$ perform equally ($0.64$), and fixing $\tau=1$ removes significant separation ($p=0.21$). The query--key affinity matches the default C-Index ($0.64$) with a slightly stronger split ($p=0.001$, $\Delta$RMST $=421$ days).

\subsection{TCGA-CRC Survival}
As for TCGA-BRCA, for each setting we report the mode with the highest orientation-free C-Index, together with the log-rank $p$-value and $\Delta$RMST (10-year horizon) of a median split on that mode's score.

The best-mode C-Index varies little across settings ($0.57$--$0.63$), with the default at $0.60$ ($p=0.057$, $\Delta$RMST $=653$ days). Consistent with the other cohorts, strong decorrelation lowers the C-Index ($0.57$ at $\lambda_{\mathrm{corr}}\in\{10,100\}$) and Getis--Ord $G^\star$ is the weakest local statistic ($0.58$, $p=0.20$), while Geary's $C$ and Moran's $I$ perform similarly ($0.60$ and $0.61$). Unlike on 4i and the cSCC cohort, both learned feature affinities slightly improve the C-Index (query--key and shared-$f_{\mathrm{map}}$: $0.63$), as do smaller neighbourhoods ($k_{\mathrm{NN}}\in\{8,16\}$: $0.63$). Fixing $\tau=1$ has little effect on this cohort ($0.59$ vs.\ $0.60$).

\subsection{Summary}
Across the cohorts we see that strong decorrelation ($\lambda_{\mathrm{corr}}\geq10$) degrades performance; Geary's $C$ and Moran's $I$ perform similarly, whereas Getis--Ord $G^\star$ is substantially worse for region recovery and survival. Learning the spatial temperature $\tau$ has little effect on the single 4i graph and on TCGA-CRC, but improves performance on the cSCC and TCGA-BRCA cohorts. The effect of a learned feature affinity is dataset-dependent: it reduces performance on 4i and the cSCC cohort but slightly improves the C-Index on TCGA-CRC.

\begin{table}[!t]
\centering
\caption{Ablation study of SAIL components on the 4i dataset (Supplementary Section~\ref{sec:ablation}). For each setting we report the best Dice per subcellular region (maximum over the 20 modes and High--High/Low--Low LISA clusters), averaged over the 10 regions. Each block varies one component while holding all others at their default setting (indicated via $^\dagger$).}
\begin{tabular}{llc}
\toprule
Category & Setting & Mean Dice\\
\midrule

\multirow{5}{*}{$\lambda_{\mathrm{corr}}$}
 & 0 & 0.42 \\
 & 1$^\dagger$ & \textbf{0.44} \\
 & 5 & 0.42 \\
 & 10 & 0.36 \\
 & 100 & 0.32 \\
\midrule

\multirow{5}{*}{$\lambda_{\mathrm{sparse}}$}
 & 0 & \textbf{0.44} \\
 & 0.1 & \textbf{0.44} \\
 & 1$^\dagger$ & \textbf{0.44} \\
 & 10 & 0.42 \\
 & 100 & 0.41 \\
\midrule

\multirow{3}{*}{$f_{\mathrm{act}}$}
 & Identity$^\dagger$ & \textbf{0.44} \\
 & ReLU & \textbf{0.44} \\
 & Sigmoid & 0.43 \\
\midrule

\multirow{3}{*}{$f_{\mathrm{reduce}}$}
 & Mean$^\dagger$ & \textbf{0.44} \\
 & Max & \textbf{0.44} \\
 & Sum & 0.43 \\
\midrule

\multirow{2}{*}{$f_{\mathrm{map}}$}
 & Stiefel (orthogonal)$^\dagger$ & 0.44 \\
 & Unconstrained & \textbf{0.47} \\
\midrule

\multirow{3}{*}{Statistic}
 & Geary's $C$$^\dagger$ & \textbf{0.44} \\
 & Moran's $I$ & \textbf{0.44} \\
 & Getis--Ord $G^\star$ & 0.32 \\
\midrule

\multirow{3}{*}{$f_{\mathrm{affinity}}$}
 & No affinity$^\dagger$ & \textbf{0.44} \\
 & Query--key & 0.33 \\
 & Shared $f_{\mathrm{map}}$ & 0.36 \\
\midrule

\multirow{2}{*}{$\tau$}
 & Learnable$^\dagger$ & \textbf{0.44} \\
 & Not learnable ($=1$) & \textbf{0.44} \\
\midrule

\multirow{4}{*}{$k_{\mathrm{NN}}$}
 & 8 & 0.43 \\
 & 16 & \textbf{0.44} \\
 & 32$^\dagger$ & \textbf{0.44} \\
 & 64 & \textbf{0.44} \\
\midrule

\multirow{3}{*}{Learning rate}
 & 0.01 & 0.43 \\
 & 0.05$^\dagger$ & 0.44 \\
 & 0.1 & \textbf{0.45} \\
\bottomrule
\end{tabular}
\par\smallskip
\label{tab:hyperparam_results_4i}
\end{table}

\begin{table}[!t]
\centering
\caption{Ablation study of SAIL components on the cSCC PD-1/PD-L1 blockade cohort. For each setting, we report the mode that best separates responders from non-responders (highest AUC over the $20$ modes): its AUC with a $95\%$ bootstrap confidence interval (patients resampled with replacement, $10{,}000$ resamples) and the Mann--Whitney $p$-value. Each block varies one component while holding all others at their default setting (indicated via $^\dagger$).}
\setlength{\tabcolsep}{4pt}
\begin{tabular}{llcc}
\toprule
Category & Setting & AUC (95\% CI) & $p$-value\\
\midrule
\multirow{5}{*}{$\lambda_{\mathrm{corr}}$}
 & 0 & 0.84 (0.62--1.00) & 0.010 \\
 & 1$^\dagger$ & \textbf{0.86} (0.64--1.00) & 0.005 \\
 & 5 & 0.82 (0.60--0.98) & 0.015 \\
 & 10 & 0.75 (0.49--0.95) & 0.062 \\
 & 100 & 0.78 (0.55--0.96) & 0.032 \\
\midrule
\multirow{5}{*}{$\lambda_{\mathrm{sparse}}$}
 & 0 & 0.84 (0.63--0.99) & 0.010 \\
 & 0.1 & 0.82 (0.60--0.98) & 0.015 \\
 & 1$^\dagger$ & \textbf{0.86} (0.64--1.00) & 0.005 \\
 & 10 & 0.82 (0.58--1.00) & 0.015 \\
 & 100 & 0.80 (0.55--1.00) & 0.022 \\
\midrule
\multirow{3}{*}{$f_{\mathrm{act}}$}
 & Identity$^\dagger$ & \textbf{0.86} (0.64--1.00) & 0.005 \\
 & ReLU & 0.83 (0.61--1.00) & 0.012 \\
 & Sigmoid & 0.80 (0.57--0.96) & 0.022 \\
\midrule
\multirow{3}{*}{$f_{\mathrm{reduce}}$}
 & Mean$^\dagger$ & \textbf{0.86} (0.64--1.00) & 0.005 \\
 & Max & \textbf{0.86} (0.66--1.00) & 0.005 \\
 & Sum & 0.74 (0.49--0.94) & 0.073 \\
\midrule
\multirow{2}{*}{$f_{\mathrm{map}}$}
 & Stiefel (orthogonal)$^\dagger$ & \textbf{0.86} (0.64--1.00) & 0.005 \\
 & Unconstrained & 0.84 (0.63--0.98) & 0.010 \\
\midrule
\multirow{3}{*}{Statistic}
 & Geary's $C$$^\dagger$ & \textbf{0.86} (0.64--1.00) & 0.005 \\
 & Moran's $I$ & 0.84 (0.61--1.00) & 0.010 \\
 & Getis--Ord $G^\star$ & 0.84 (0.62--0.99) & 0.010 \\
\midrule
\multirow{3}{*}{$f_{\mathrm{affinity}}$}
 & No affinity$^\dagger$ & \textbf{0.86} (0.64--1.00) & 0.005 \\
 & Query--key & 0.80 (0.57--0.96) & 0.022 \\
 & Shared $f_{\mathrm{map}}$ & 0.75 (0.50--0.95) & 0.062 \\
\midrule
\multirow{2}{*}{$\tau$}
 & Learnable$^\dagger$ & \textbf{0.86} (0.64--1.00) & 0.005 \\
 & Not learnable ($=1$) & 0.79 (0.56--0.97) & 0.027 \\
\midrule
\multirow{4}{*}{$k_{\mathrm{NN}}$}
 & 8 & 0.79 (0.56--0.96) & 0.027 \\
 & 16 & 0.81 (0.56--1.00) & 0.018 \\
 & 32$^\dagger$ & 0.86 (0.64--1.00) & 0.005 \\
 & 64 & \textbf{0.91} (0.75--1.00) & 0.002 \\
\midrule
\multirow{3}{*}{Learning rate}
 & 0.01 & 0.83 (0.62--0.98) & 0.012 \\
 & 0.05$^\dagger$ & \textbf{0.86} (0.64--1.00) & 0.005 \\
 & 0.1 & 0.84 (0.63--0.98) & 0.010 \\
\bottomrule
\end{tabular}
\par\smallskip
\label{tab:hyperparam_results_pdl1}
\end{table}

\begin{table}[!t]
\centering
\caption{Ablation study of SAIL components on the TCGA-BRCA cohort ($498$ patients; Supplementary Section~\ref{sec:ablation}). For each setting we report the SAIL Mode with the highest C-Index; patients are split into high/low groups via a median split on that mode's score, and we report the log-rank $p$-value and $\Delta$RMST (days, 10-year horizon) between groups. Each block varies one component while holding all others at their default setting (indicated via $^\dagger$).}
\begin{tabular}{llccc}
\toprule
Category & Setting & C-Index & $p$-value & $\Delta$RMST\\
\midrule
\multirow{5}{*}{$\lambda_{\mathrm{corr}}$}
 & 0 & \textbf{0.67} & 0.001 & 365 \\
 & 1$^\dagger$ & 0.64 & 0.005 & 386 \\
 & 5 & 0.60 & 0.046 & 243 \\
 & 10 & 0.60 & 0.464 & 96 \\
 & 100 & 0.60 & 0.123 & 188 \\
\midrule
\multirow{5}{*}{$\lambda_{\mathrm{sparse}}$}
 & 0 & 0.62 & 0.020 & 347 \\
 & 0.1 & \textbf{0.65} & 0.032 & 315 \\
 & 1$^\dagger$ & 0.64 & 0.005 & 386 \\
 & 10 & \textbf{0.65} & $<$0.001 & 464 \\
 & 100 & 0.64 & 0.044 & 159 \\
\midrule
\multirow{3}{*}{$f_{\mathrm{act}}$}
 & Identity$^\dagger$ & \textbf{0.64} & 0.005 & 386 \\
 & ReLU & 0.63 & 0.012 & 369 \\
 & Sigmoid & 0.59 & 0.092 & 182 \\
\midrule
\multirow{3}{*}{$f_{\mathrm{reduce}}$}
 & Mean$^\dagger$ & 0.64 & 0.005 & 386 \\
 & Max & \textbf{0.67} & 0.002 & 313 \\
 & Sum & 0.64 & 0.008 & 356 \\
\midrule
\multirow{2}{*}{$f_{\mathrm{map}}$}
 & Stiefel (orthogonal)$^\dagger$ & \textbf{0.64} & 0.005 & 386 \\
 & Unconstrained & 0.62 & 0.017 & 364 \\
\midrule
\multirow{3}{*}{Statistic}
 & Geary's $C$$^\dagger$ & \textbf{0.64} & 0.005 & 386 \\
 & Moran's $I$ & \textbf{0.64} & 0.009 & 379 \\
 & Getis--Ord $G^\star$ & 0.57 & 0.227 & 182 \\
\midrule
\multirow{3}{*}{$f_{\mathrm{affinity}}$}
 & No affinity$^\dagger$ & \textbf{0.64} & 0.005 & 386 \\
 & Query--key & \textbf{0.64} & 0.001 & 421 \\
 & Shared $f_{\mathrm{map}}$ & 0.60 & 0.312 & 174 \\
\midrule
\multirow{2}{*}{$\tau$}
 & Learnable$^\dagger$ & \textbf{0.64} & 0.005 & 386 \\
 & Not learnable ($=1$) & 0.61 & 0.210 & 220 \\
\midrule
\multirow{4}{*}{$k_{\mathrm{NN}}$}
 & 8 & \textbf{0.64} & 0.004 & 394 \\
 & 16 & 0.62 & 0.090 & 293 \\
 & 32$^\dagger$ & \textbf{0.64} & 0.005 & 386 \\
 & 64 & \textbf{0.64} & 0.011 & 362 \\
\midrule
\multirow{3}{*}{Learning rate}
 & 0.01 & 0.63 & 0.111 & 256 \\
 & 0.05$^\dagger$ & \textbf{0.64} & 0.005 & 386 \\
 & 0.1 & 0.61 & $<$0.001 & 490 \\
\bottomrule
\end{tabular}
\par\smallskip
\label{tab:hyperparam_results_brca}
\end{table}

\begin{table}[!t]
\centering
\caption{Ablation study of SAIL components on the TCGA-CRC cohort ($558$ patients; Supplementary Section~\ref{sec:ablation}). For each setting we report the SAIL Mode with the highest C-Index; patients are split into high/low groups via a median split on that mode's score, and we report the log-rank $p$-value and $\Delta$RMST (days, 10-year horizon) between groups. Each block varies one component while holding all others at their default setting (indicated via $^\dagger$).}
\begin{tabular}{llccc}
\toprule
Category & Setting & C-Index & $p$-value & $\Delta$RMST\\
\midrule
\multirow{5}{*}{$\lambda_{\mathrm{corr}}$}
 & 0 & \textbf{0.61} & 0.015 & 474 \\
 & 1$^\dagger$ & 0.60 & 0.057 & 653 \\
 & 5 & 0.60 & 0.091 & 178 \\
 & 10 & 0.57 & 0.320 & 112 \\
 & 100 & 0.57 & 0.021 & 373 \\
\midrule
\multirow{5}{*}{$\lambda_{\mathrm{sparse}}$}
 & 0 & 0.60 & 0.002 & 807 \\
 & 0.1 & 0.60 & 0.032 & 634 \\
 & 1$^\dagger$ & 0.60 & 0.057 & 653 \\
 & 10 & \textbf{0.62} & 0.009 & 389 \\
 & 100 & 0.59 & 0.006 & 554 \\
\midrule
\multirow{3}{*}{$f_{\mathrm{act}}$}
 & Identity$^\dagger$ & 0.60 & 0.057 & 653 \\
 & ReLU & 0.60 & 0.057 & 653 \\
 & Sigmoid & \textbf{0.62} & 0.059 & 184 \\
\midrule
\multirow{3}{*}{$f_{\mathrm{reduce}}$}
 & Mean$^\dagger$ & 0.60 & 0.057 & 653 \\
 & Max & \textbf{0.62} & 0.009 & 313 \\
 & Sum & 0.59 & 0.018 & 381 \\
\midrule
\multirow{2}{*}{$f_{\mathrm{map}}$}
 & Stiefel (orthogonal)$^\dagger$ & 0.60 & 0.057 & 653 \\
 & Unconstrained & \textbf{0.62} & $<$0.001 & 870 \\
\midrule
\multirow{3}{*}{Statistic}
 & Geary's $C$$^\dagger$ & 0.60 & 0.057 & 653 \\
 & Moran's $I$ & \textbf{0.61} & 0.012 & 645 \\
 & Getis--Ord $G^\star$ & 0.58 & 0.203 & 152 \\
\midrule
\multirow{3}{*}{$f_{\mathrm{affinity}}$}
 & No affinity$^\dagger$ & 0.60 & 0.057 & 653 \\
 & Query--key & \textbf{0.63} & 0.029 & 367 \\
 & Shared $f_{\mathrm{map}}$ & \textbf{0.63} & 0.009 & 485 \\
\midrule
\multirow{2}{*}{$\tau$}
 & Learnable$^\dagger$ & \textbf{0.60} & 0.057 & 653 \\
 & Not learnable ($=1$) & 0.59 & 0.034 & 673 \\
\midrule
\multirow{4}{*}{$k_{\mathrm{NN}}$}
 & 8 & \textbf{0.63} & $<$0.001 & 218 \\
 & 16 & \textbf{0.63} & $<$0.001 & 635 \\
 & 32$^\dagger$ & 0.60 & 0.057 & 653 \\
 & 64 & 0.60 & 0.006 & 591 \\
\midrule
\multirow{3}{*}{Learning rate}
 & 0.01 & \textbf{0.61} & 0.001 & 607 \\
 & 0.05$^\dagger$ & 0.60 & 0.057 & 653 \\
 & 0.1 & 0.60 & $<$0.001 & 820 \\
\bottomrule
\end{tabular}
\label{tab:hyperparam_results_crc}
\end{table}

\section{Analysis of spatial coherence of macroeconomic temporal geospatial data}
\label{sec:supp_macro}
In this section we illustrate SAIL's generality on multivariate, spatio-temporal macroeconomic data.

\subsection{Data Source and Cleaning}
We use the Global Macro Database~\cite{S:muller2025global}, an annual panel of macroeconomic indicators across a wide range of countries ($239$). Prior to analysis we exclude: (i) forecast-provenance flags, which mark IMF-style projections
rather than constituting data themselves; (ii) binary crisis dummies (sovereign debt, currency, and banking crises); (iii) indicators denominated in local-currency levels (nominal/real GDP, consumption, government debt, money aggregates M0--M4), which are not comparable across countries and are further subject to abrupt redenomination jumps (e.g.\ Venezuela in 2018, Bolivia in 1985); and (iv) the raw exchange-rate level, which is likewise an arbitrary local-currency-per-USD unit. USD-denominated levels, \%-of-GDP ratios, rates, and index-based indicators are retained. The panel is restricted to 2000--2023 and balanced on a 32-variable subset selected by
greedy forward search to maximise complete country coverage across all years; at this cutoff, 162 countries have complete data for every year in the range.

\subsection{Feature Set}
From the cleaned panel we select 25 cross-country-comparable indicators spanning output and prices (nominal/real GDP, GDP per capita, population, inflation), consumption (total, household, and government, as USD levels and \% of GDP), investment (total and fixed, USD and \% GDP), external accounts
(exports, imports, current account, USD and \% GDP), and government finance (expenditure, revenue, deficit, and debt, all \% GDP). Each feature is z-scored across all country-years in the analysis window prior to modelling.

\subsection{Graph Construction}
One spatial graph is constructed per calendar year, with nodes corresponding to countries. Node coordinates are country centroids obtained from Natural Earth (110\,m boundaries); joining these to the cleaned panel drops a further 17 countries lacking a matching polygon (small island states,
together with Singapore, Malta, Bahrain, and Mauritius), leaving 145 countries in every yearly graph. Edges are defined via great-circle $k$-nearest-neighbour search ($k_{\mathrm{NN}}=6$) using a haversine-distance
ball tree, with raw (non-log-transformed) angular distances passed to SAIL's affinity module.

\subsection{Training Configuration}
A single shared set of $8$ SAIL Modes is learned jointly across all yearly graphs (batch size 1, iterating over all years per epoch), so that a given mode retains a consistent interpretation across the full 2000--2023 window.
Modes are optimised in the unsupervised configuration (Section~\ref{sec:modes}), minimising a Geary-style global spatial autocorrelation objective under a Stiefel-manifold-constrained linear projection, with decorrelation
($\lambda_{\mathrm{corr}}=1.0$) and sparsity ($\lambda_{\mathrm{sparse}}=1.0$) regularisation. Training uses Adam (learning rate $0.05$) for $1000$ epochs. For each mode and year we record the global SAIL score (the spatial
autocorrelation of that mode's projected field on that year's country graph), the per-country projected field ($f_{\mathrm{map}}$ score), and per-country LISA cluster assignments, supporting both the cross-sectional
(mode loadings and country rankings) and longitudinal (spatio-temporal trend) analyses.

\subsection{Results}
\label{sec:supp_macro_results}
SAIL discovers $8$ modes of cross-country spatial coherence, each defined by a sparse combination of a few indicators (Supplementary Figure~\ref{fig:all_macro_weights}). Three broad groups of indicators recur across modes. The first is external trade (exports, imports and the current account; Modes 1, 2, 4 and 8). The second is government finance (expenditure, consumption, revenue and debt as a share of GDP; Modes 3, 5 and 6). The third is overall economic size (real or nominal GDP and GDP per capita; Modes 1, 7 and 8). Several modes contrast two related quantities, for example exports against household consumption (Mode 2), imports against total consumption (Mode 4), government debt against revenue (Mode 5) and government expenditure against government consumption (Mode 6). The learned scalar fields therefore capture relative economic structure rather than any single indicator.

Because one shared set of modes is learned across all years, the global SAIL score of each mode can be tracked from 2000 to 2023 (Supplementary Figure~\ref{fig:all_macro_trends}); lower values indicate that neighbouring countries are more alike along that mode. Some modes are largely stable over time: Modes 1 and 7 remain close to $0.50$ and $0.65$, respectively, throughout the period, indicating a persistent level of regional similarity. Others change gradually. For example, Mode 4 declines steadily until 2019 before rising again, and Mode 8, the least spatially coherent mode overall, falls until around 2012 and then rises. A third group shows sharp, short-lived changes. Mode 2 decreases from $0.65$ to $0.45$ between 2000 and 2008 before rising abruptly in 2009, coinciding with the global financial crisis. Mode 5, which contrasts government debt with revenue, peaks in 2020--2021, coinciding with the COVID-19 pandemic.

\begin{figure*}[!t]
  \centering
   \includegraphics[width=6in]{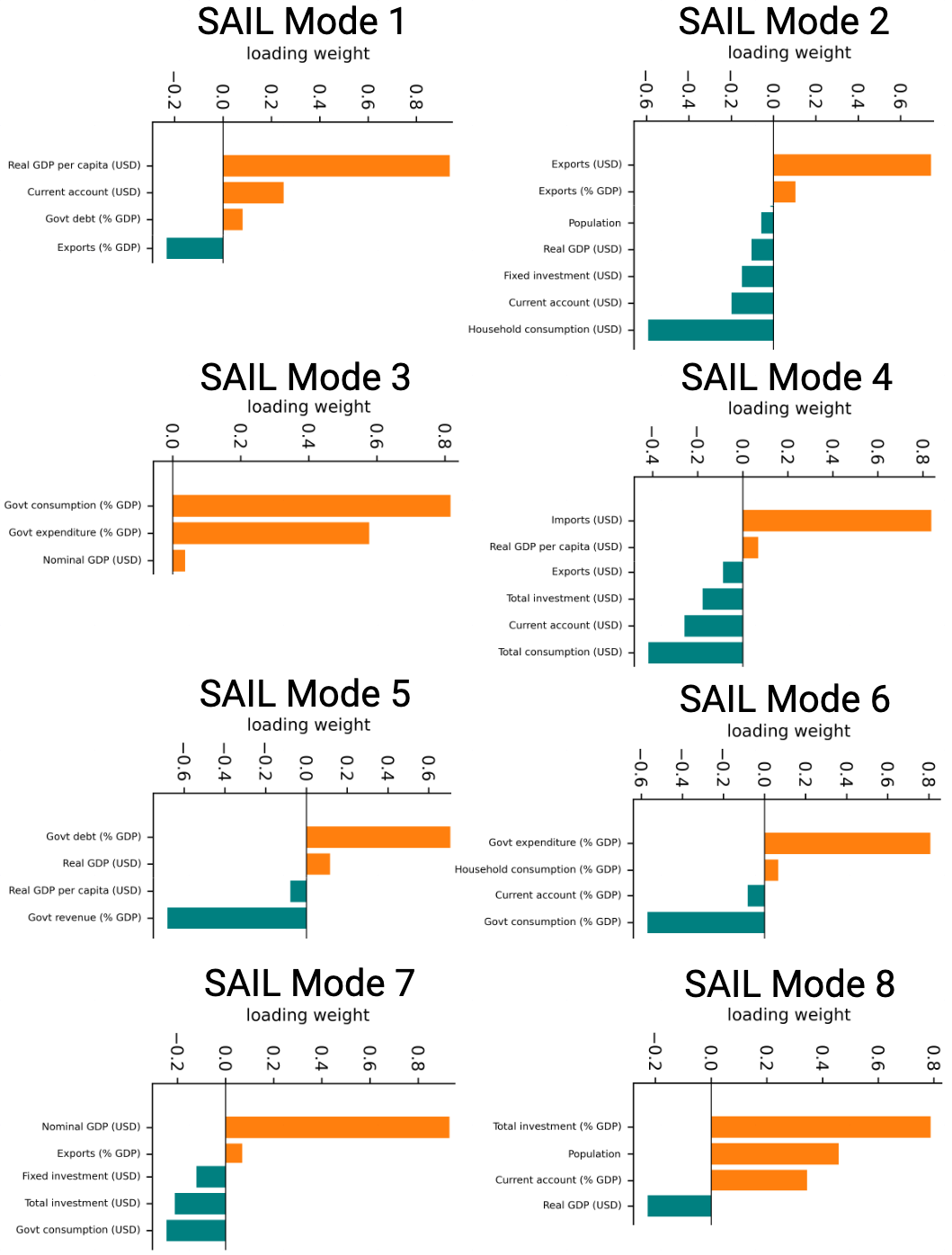}
   \caption{All $8$ SAIL Modes learned on the Global Macro Database. For each mode, bars show the top macroeconomic indicators by absolute loading in the learned $f_{\mathrm{map}}$ projection, with positively and negatively weighted indicators coloured separately; these weights specify the indicator combination whose spatial Geary-style dissimilarity each mode minimises.}
   \label{fig:all_macro_weights}
\end{figure*}

\begin{figure*}[!t]
  \centering
   \includegraphics[width=6in]{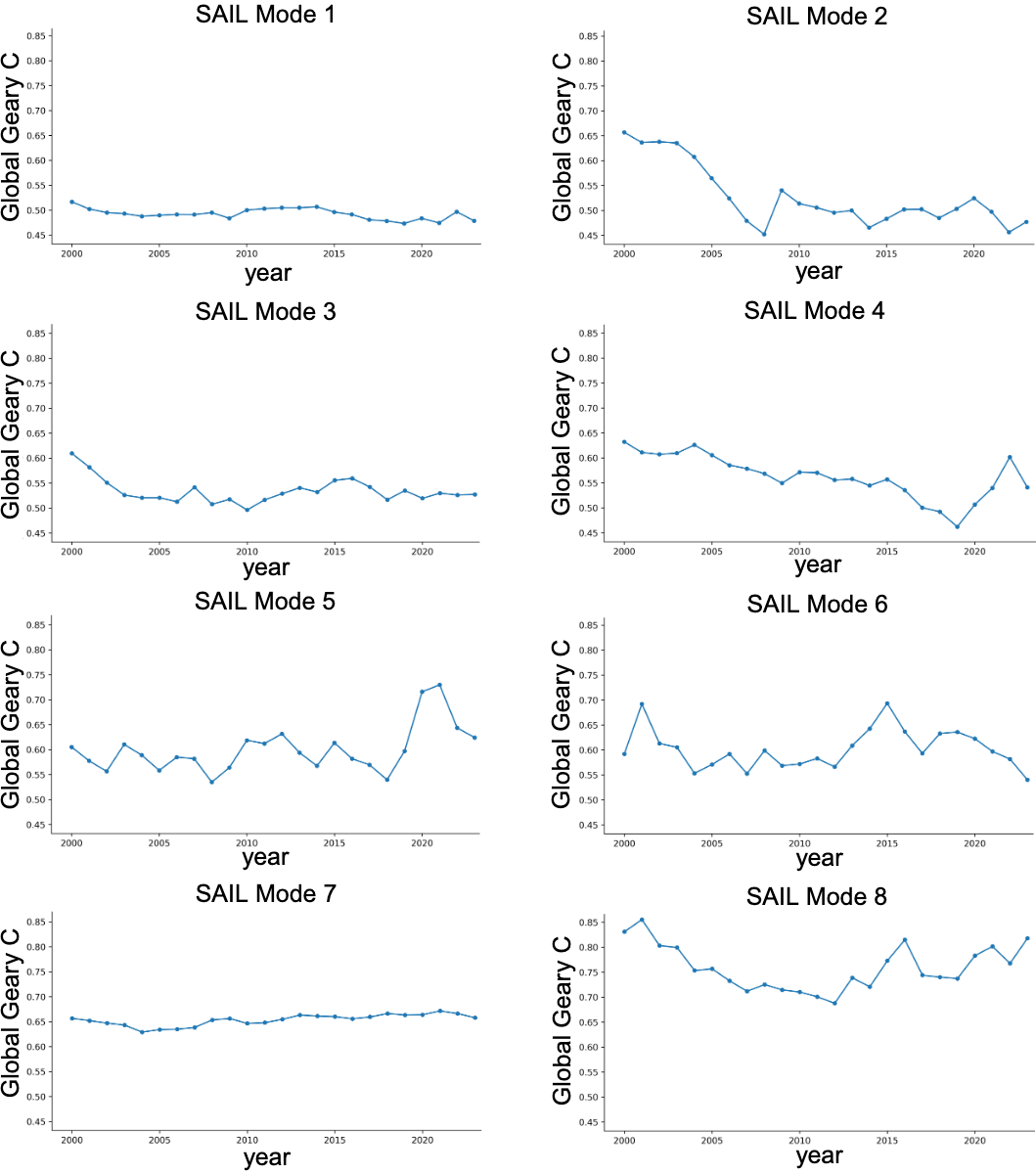}
    \caption{Temporal trends of all $8$ SAIL Modes learned on the Global Macro Database. For each mode, the line shows the global Geary-style SAIL score of the mode's projected field on each yearly country graph from 2000 to 2023; lower values indicate that neighbouring countries are more similar along that mode.}
     \label{fig:all_macro_trends}
\end{figure*}

\section{Geospatial Temporal Analysis: COVID-19 Case Trends}
\label{sec:supp_covid}

\subsection{Data Source and Preprocessing}
We use the Our World in Data (OWID) COVID-19 dataset~\cite{S:dutta2022survey},
taking new confirmed cases per million people as the node feature. Daily
records are aggregated to weekly means per country, yielding one value per
country per calendar week spanning January 2020 to May 2023. Country codes
are matched against Natural Earth (110\,m) administrative boundaries to
obtain a fixed geographic reference frame; any week for which one or more
matched countries have a missing value is dropped.

\subsection{Graph Construction}
Nodes correspond to countries, with coordinates given by country polygon
centroids (latitude/longitude) from the matched shapefile. For each weekly
snapshot, a $k_{\mathrm{NN}}$-nearest-neighbour graph ($k_{\mathrm{NN}}=16$) is built directly from
country coordinates using great-circle (haversine) distance via a ball
tree. The single node feature
per week is the country's weekly per-million case count.

\subsection{Results}
In Supplementary Figure~\ref{fig:COVID_temporal}, tracking the global Moran's $I$ score across all retained weeks reveals
substantial temporal variation, with several distinct peaks corresponding
to periods of strong spatial synchronisation in case growth across
neighbouring countries. Examining the LISA cluster maps over
time, Supplementary Figure~\ref{fig:COVID_lisa}, shows that High--High hotspots
recur disproportionately in Western and Central Europe across multiple
waves, suggesting a persistent geographic pattern of regionally
synchronised outbreaks rather than isolated, one-off events. 

This further illustrates that SAIL's spatial autocorrelation computation extends
directly to real-world, non-biological spatio-temporal data.

\begin{figure}[!t]
  \centering
   \includegraphics[width=3in]{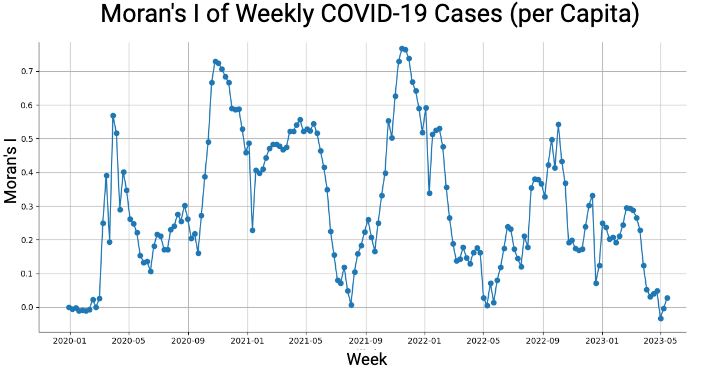}
   \caption{Unsupervised SAIL for COVID-19 trend analysis. Here we show a line graph that reports the global Moran's $I$ of new COVID-19 cases (y-axis) against a given week (x-axis).}
   \label{fig:COVID_temporal}
\end{figure}

\begin{figure}[!t]
  \centering
   \includegraphics[width=3in]{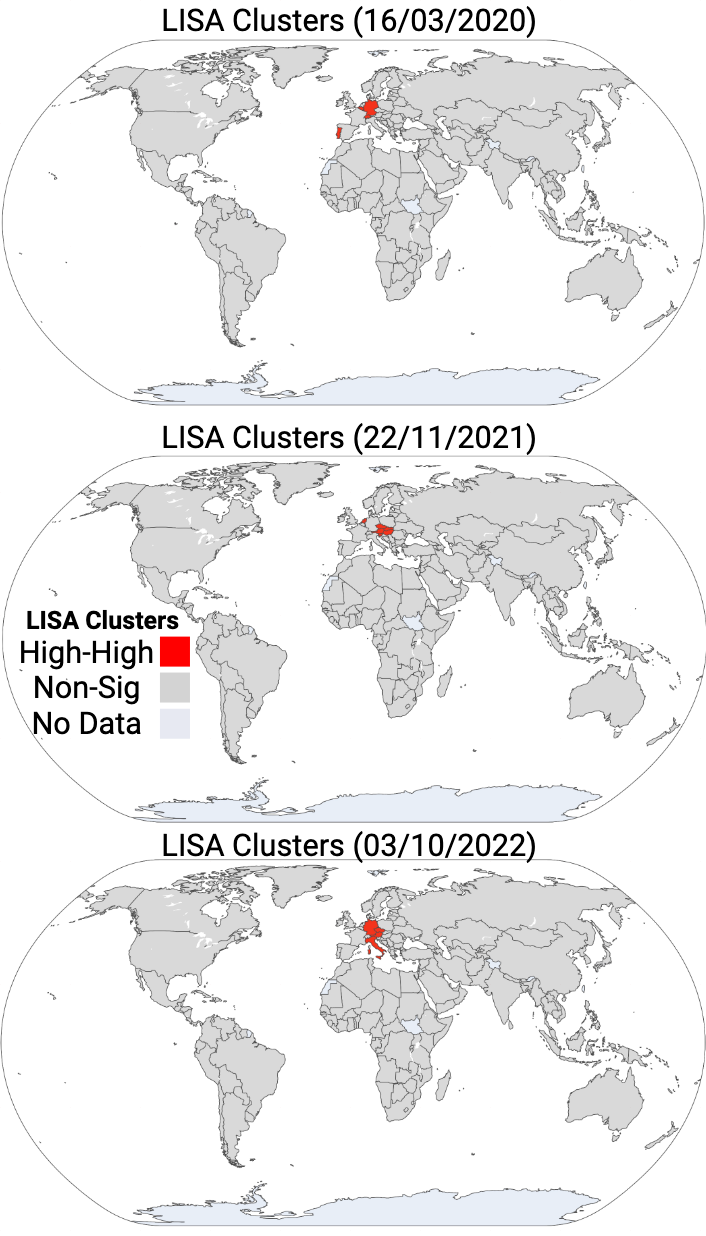}
   \caption{A choropleth map for the world coloured by LISA clustering. Interestingly, we notice that across time, COVID-19 hotspots tend to form in Western and Central European countries, reflecting a consistent pattern of spatially concentrated High--High clusters and regional synchronisation of outbreaks.}
   \label{fig:COVID_lisa}
\end{figure}

\section{Spatial coherence of mixing gases}
\label{sec:supp_diffusion}
To show that SAIL applies to physical systems whose state is defined by spatial coherence, we use it to follow the mixing of two gases by diffusion. Their initial separation is a state of maximal spatial autocorrelation in gas identity, and mixing removes it.

\subsection{Dataset Setup}
We simulate a two-dimensional hard-disc gas of $600$ equal-mass particles (radius $4$) in a $1000\times540$ box, using simulation units for length and seconds for time. At $t=0$, $300$ particles of species A ($s_j=+1$) fill the left half and $300$ particles of species B ($s_j=-1$) fill the right half. The two halves are separated by a vertical partition at the centre of the box, and particles are placed on a jittered grid with a $20$-unit empty margin at the walls and the partition. Initial velocities follow a two-dimensional Maxwell--Boltzmann distribution with a root-mean-square speed of $100$, identical for both species. Particles move ballistically, and all particle--particle, particle--wall and particle--partition collisions are perfectly elastic. The equations of motion are integrated with time step $\Delta t=0.02$\,s, sub-stepped so that no particle moves more than $0.75$ radii per sub-step, for $80$\,s. A snapshot is recorded every $0.24$\,s. We consider three scenarios that differ only in what happens to the partition at $t=0$:
\begin{itemize}
    \item \textbf{Partition closed} (control): the partition remains, so the gases cannot mix;
    \item \textbf{Small hole}: a $120$-unit opening is made at the centre of the partition;
    \item \textbf{Partition removed}: the partition is removed entirely.
\end{itemize}
The partition is purely geometric and never interacts with species identity. Each scenario is simulated with $5$ independent random seeds, and a separate run of the partition-removed scenario, with its own seed, is reserved for training SAIL.

\subsection{Graph Construction and SAIL Model}
Each snapshot is a spatial graph whose nodes are particles, connected by $k_{\mathrm{NN}}=8$ nearest-neighbour edges built from particle positions, with log-transformed neighbour distances passed to the affinity module. The node feature vector is $[s_j,\ \lVert v_j\rVert,\ \epsilon_{j1},\dots,\epsilon_{j4}]$, comprising species identity, particle speed and four nuisance features $\epsilon_{jk}\sim\mathcal{N}(0,1)$, drawn independently per particle and snapshot. Each feature is standardised within each snapshot. Because both species start at the same temperature, speed carries no spatial information, so species identity is the only spatially autocorrelated feature.

SAIL is trained unsupervised (Section~\ref{sec:modes}) to learn a single mode that maximises the Moran-style local index (Eq.~\eqref{eq:moran_local}), using an unconstrained linear $f_{\mathrm{map}}$, mean global reduction, a learnable spatial temperature $\tau$, no learned feature affinity, $\lambda_{\mathrm{corr}}=0$ and $\lambda_{\mathrm{sparse}}=0.01$. The trained mode is then frozen and applied to every snapshot of the $15$ evaluation runs. LISA cluster maps are computed as in Section~\ref{hotspot_discovery}.

\subsection{Results}
SAIL learns a mode that loads almost entirely on species identity (loading $1.00$), with negligible weight on speed ($0.01$) and on all four nuisance features ($|w|\leq0.015$; Supplementary Figure~\ref{fig:diffusion_analysis}\,i), and a learned spatial temperature of $\tau=4.51$. Without supervision, it therefore identifies the one spatially coherent physical quantity in the system and ignores features that vary across snapshots without spatial structure.

At $t=0$ the gases are fully separated and the score is $0.999$ in every run. With the partition closed, it stays close to its initial value throughout ($0.982\pm0.005$ at $t=80$\,s; mean $\pm$ standard deviation over $5$ seeds). When the partition is removed, the score decays rapidly: it reaches $0.53\pm0.04$ at $t=10$\,s, halves after $12.2$\,s and approaches zero by the end of the simulation ($0.04\pm0.03$ at $t=80$\,s). With a small hole, mixing is limited by the flux through the opening and the decay is slower: the score is $0.82\pm0.02$ at $t=10$\,s, halves after $33.4$\,s and is still $0.21\pm0.01$ at $t=80$\,s.

The LISA maps show where the remaining spatial coherence lies (Supplementary Figure~\ref{fig:diffusion_analysis}\,iii). At early times, the two halves of the box form large High--High and Low--Low clusters. As mixing proceeds, these clusters shrink back from the partition towards the outer walls, and the intermixed central region becomes non-significant. With a small hole, mixing starts as a localised plume around the opening while the rest of the box stays clustered; with the partition closed, the two clusters persist throughout.

\begin{figure}[!t]
  \centering
   \includegraphics[width=3in]{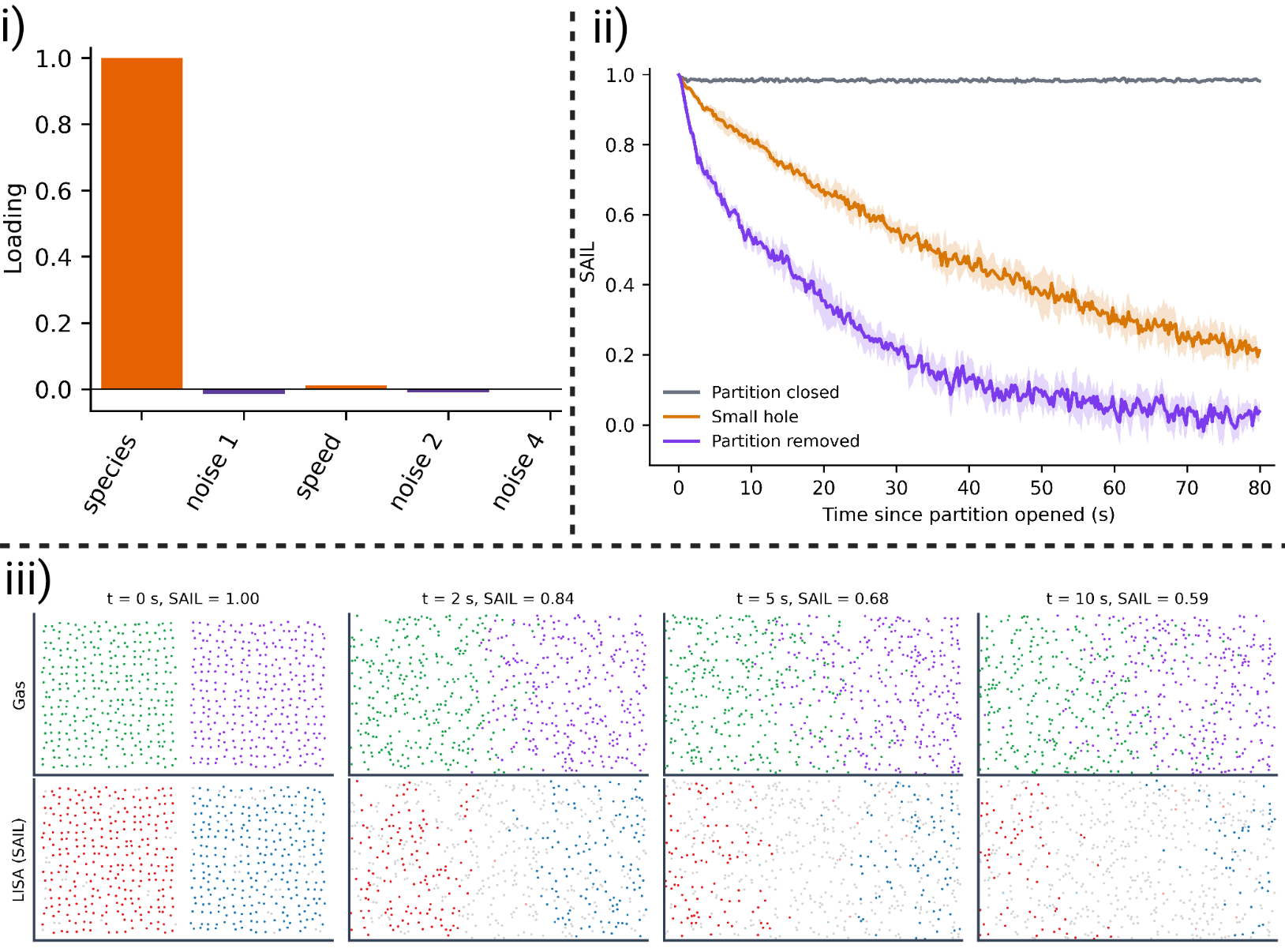}
   \caption{SAIL tracks the mixing of two gases. (i) Loadings of the single SAIL Mode learned on a training run, over species identity, particle speed and four nuisance noise features. (ii) SAIL score over time (seconds) for the three scenarios (partition closed, small hole, partition removed); lines show the mean and shaded bands $\pm1$ standard deviation over $5$ seeds. (iii) Gas states (top) and the corresponding LISA clusters of the learned mode (bottom) at selected time points.}
   \label{fig:diffusion_analysis}
\end{figure}

\section{SAIL limitations}
\label{sec:supp_scope}

SAIL is designed for problems in which the target depends on spatial autocorrelation, and its inductive bias should not be expected to help when it does not. In this section we characterise both regimes: the class of problems on which SAIL is expected to outperform classical message-passing architectures, and the class on which it is not.

\subsection{Problem class where SAIL is optimal}
\label{sec:supp_synth_alpha}

The synthetic benchmark of Supplementary Section~\ref{sec:supp_synth_protocol} is, by construction, aligned with spatial autocorrelation. A natural question is therefore how much of the target must be spatially determined for SAIL to retain its advantage. To answer this, we make the alignment an explicit, controllable parameter. We show that SAIL's advantage is confined to the spatial-autocorrelation regime and disappears when the target depends on feature content alone (Supplementary Figure~\ref{fig:baseline_regression_alpha}).

Specifically, we blend the spatial target with a purely content-based target and sweep the mixing weight $\gamma\in[0,1]$. For each graph, we compute two scalars:
\begin{itemize}
\item \textbf{Spatial component} $y_{\mathrm{spatial}}$: the Moran's $I$ of the hidden random projection $z^\star$ on the hidden $k_{\mathrm{true}}$ neighbourhood, exactly as in Supplementary Section~\ref{sec:supp_synth_dataset}.

\item \textbf{Content component} $y_{\mathrm{feat}}$: a coordinate-free second moment of the features,
\begin{equation}
y_{\mathrm{feat}}
=\frac{1}{n}\sum_{j=1}^{n}\big(x_j^\top v_{\mathrm{feat}}\big)^2
=v_{\mathrm{feat}}^\top \mathbf{C}_{\mathcal{G}}\, v_{\mathrm{feat}},
\end{equation}
where $\mathbf{C}_{\mathcal{G}}$ is the per-graph feature covariance and $v_{\mathrm{feat}}\in\mathbb{R}^d$ is a \emph{fixed} unit direction, held constant across all graphs, splits and seeds. Because $y_{\mathrm{feat}}$ is permutation-invariant and carries no spatial information, it can be recovered by a node-wise MLP with mean pooling, as well as by GCN and GAT.
\end{itemize}
 
Each component is standardised to zero mean and unit variance using statistics from the training split only, so that $\gamma$ mixes two quantities on a comparable scale. The regression target is the convex combination
\begin{equation}
y(\gamma)=(1-\gamma)\,\bar y_{\mathrm{spatial}}+\gamma\,\bar y_{\mathrm{feat}},
\qquad \gamma\in[0,1],
\end{equation}
where $\bar{(\cdot)}$ denotes a train-standardised component. At $\gamma=0$ the target is pure spatial autocorrelation; at $\gamma=1$ it is pure feature content.

We compare SAIL against GCN and GAT on the fully connected graph with node coordinates concatenated to the features, and against a node-wise MLP with coordinates. All models use a hidden width of $32$ (with $2$ heads for attention) and share the optimiser, learning rate, early-stopping criterion and data splits of Supplementary Section~\ref{sec:supp_synth_protocol}. We sweep $\gamma \in \{0,0.15,0.3,0.5,0.7,0.85,1.0\}$ over $10$ seeds and report the mean test Pearson correlation as a function of $\gamma$.

\begin{table}[!t]
\centering
\caption{Per-model results on the synthetic Moran's $I$ regression benchmark.
Each model is evaluated at its Stage-1--selected configuration over $1000$
seeds; we report the mean $\pm$ standard deviation of test Pearson and MSE. Rows are ordered by descending Pearson correlation.}
\label{tab:supp_regression_results}
\begin{tabular}{lccc}
\toprule
Model & \#Params & Pearson & MSE \\
\midrule
SAIL (dense)                        & $82$  & $\mathbf{0.819 \pm 0.015}$ & $\mathbf{0.0114 \pm 0.0016}$ \\
\midrule
GCN (X, FC)         & $81$  & $0.730 \pm 0.022$ & $0.0163 \pm 0.0025$ \\
\midrule
\shortstack[l]{GCN\\(X+coords, FC)}  & $193$ & $0.724 \pm 0.022$ & $0.0169 \pm 0.0028$ \\
\midrule
GAT (X, FC)         & $353$ & $0.627 \pm 0.096$ & $0.0212 \pm 0.0046$ \\
\midrule
\shortstack[l]{GAT\\(X+coords, FC)}  & $785$ & $0.459 \pm 0.203$ & $0.0271 \pm 0.0065$ \\
\midrule
MLP (X)             & $321$ & $0.016 \pm 0.044$ & $0.0335 \pm 0.0031$ \\
\midrule
\shortstack[l]{MLP\\(X+coords)}      & $97$  & $0.010 \pm 0.044$ & $0.0334 \pm 0.0028$ \\
\bottomrule
\end{tabular}
\end{table}


\begin{table}[!t]
\centering
\caption{Spatial-versus-content blend ablation. Mean test Pearson correlation
(subscript: $\pm$ standard deviation over $10$ seeds) as the target shifts from
pure spatial autocorrelation ($\gamma{=}0$) to pure feature content
($\gamma{=}1$). GCN, GAT and MLP use node coordinates (X+coords), and the GNNs
use the fully connected graph (FC). SAIL is strongest at $\gamma{=}0$ and
degrades monotonically as the target loses its spatial component, while the
content-recoverable baselines improve. The best model per row is in bold, with
a clear crossover near $\gamma\approx0.3$--$0.5$.}
\label{tab:supp_alpha_ablation}
\small
\setlength{\tabcolsep}{4pt}
\begin{tabular}{ccccc}
\toprule
$\gamma$ & SAIL & GCN & GAT & MLP \\
\midrule
$0.00$ & $\mathbf{0.831}_{\pm0.009}$ & $0.744_{\pm0.017}$ & $0.450_{\pm0.156}$ & $0.010_{\pm0.049}$ \\
$0.15$ & $\mathbf{0.814}_{\pm0.010}$ & $0.746_{\pm0.018}$ & $0.367_{\pm0.160}$ & $0.118_{\pm0.075}$ \\
$0.30$ & $\mathbf{0.761}_{\pm0.018}$ & $0.755_{\pm0.012}$ & $0.505_{\pm0.125}$ & $0.356_{\pm0.068}$ \\
$0.50$ & $0.584_{\pm0.032}$ & $\mathbf{0.789}_{\pm0.015}$ & $0.736_{\pm0.088}$ & $0.682_{\pm0.038}$ \\
$0.70$ & $0.338_{\pm0.045}$ & $0.815_{\pm0.014}$ & $\mathbf{0.910}_{\pm0.018}$ & $0.905_{\pm0.013}$ \\
$0.85$ & $0.177_{\pm0.053}$ & $0.822_{\pm0.016}$ & $0.975_{\pm0.004}$ & $\mathbf{0.978}_{\pm0.003}$ \\
$1.00$ & $0.100_{\pm0.070}$ & $0.824_{\pm0.014}$ & $0.993_{\pm0.004}$ & $\mathbf{0.998}_{\pm0.001}$ \\
\midrule
\multicolumn{5}{c}{\footnotesize $\gamma$: spatial autocorrelation $\rightarrow$ feature content} \\
\bottomrule
\end{tabular}
\end{table}

\subsection{Results} Supplementary Figure~\ref{fig:baseline_regression_alpha} and Supplementary Table~\ref{tab:supp_alpha_ablation} summarise the results. At $\gamma=0$, SAIL achieves the highest Pearson correlation ($0.831\pm0.009$), ahead of GCN ($0.744\pm0.017$), GAT ($0.450\pm0.156$) and the coordinate-augmented MLP ($0.010\pm0.049$). This is consistent with Theorem~\ref{theorem:inductive_bias}: SAIL explicitly contains the quadratic interactions that define the target. As $\gamma$ increases, SAIL's correlation declines monotonically to $0.100\pm0.070$ at $\gamma=1$. At that point the target is fully determined by feature content and is recovered almost perfectly by the MLP ($0.998\pm0.001$) and GAT ($0.993\pm0.004$). The models cross over between $\gamma\approx0.3$ and $0.5$: GCN leads at $\gamma=0.5$, and the content-based models dominate thereafter.

These results show that SAIL's advantage comes from its match to targets governed by spatial autocorrelation, rather than from a general modelling benefit. They also give an empirical illustration of the problem class $\mathcal{P}_{\mathrm{SAC}}$ of Theorem~\ref{theorem:inductive_bias}. When the target is driven by spatial autocorrelation, SAIL is the appropriate tool. When it is driven mainly by feature content, standard message passing or node-wise models are preferable.

\begin{figure}[t]
\centering
\includegraphics[width=3.4in]{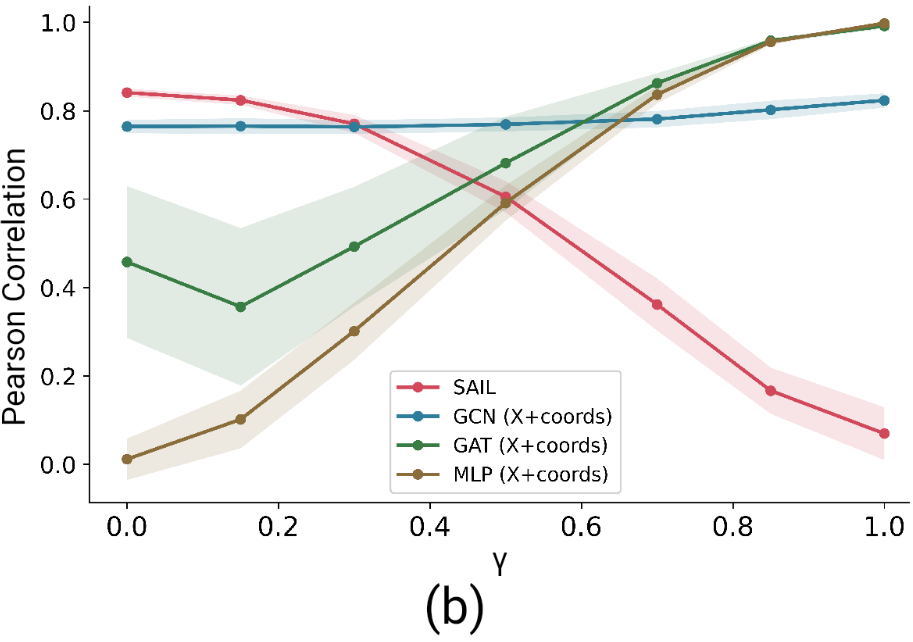}
\caption{Spatial-versus-content blend ablation. Mean Pearson correlation (shaded bands $\pm1$ standard deviation over $10$ seeds) as a function of the blend weight $\gamma$, interpolating the regression target between pure spatial autocorrelation ($\gamma=0$) and pure feature content ($\gamma=1$). SAIL is strongest at $\gamma=0$ and degrades monotonically as the spatial component is removed, while GCN, GAT, and MLP improve, crossing SAIL near $\gamma\approx0.3$--$0.5$.}
\label{fig:baseline_regression_alpha}
\end{figure}


\end{document}